\documentclass{article}

\usepackage{microtype}
\usepackage{graphicx}
\usepackage{subcaption}
\usepackage{booktabs} 
\usepackage{graphicx}
\usepackage{hyperref}

\usepackage[preprint]{icml2026}

\usepackage{amsmath}
\usepackage{amssymb}
\usepackage{mathtools}
\usepackage{amsthm}

\usepackage[capitalize,noabbrev]{cleveref}

\theoremstyle{plain}
\newtheorem{theorem}{Theorem}

\newtheorem{lemma}{Lemma}

\theoremstyle{definition}
\newtheorem{definition}{Definition}
\newtheorem{assumption}{Assumption}
\newtheorem{condition}{Condition}
\theoremstyle{remark}
\newtheorem{remark}{Remark}

\usepackage[textsize=tiny]{todonotes}

\icmltitlerunning{Understanding the Impact of Differentially Private Training on Memorization of Long-Tailed Data}

\begin{document}

\twocolumn[
  \icmltitle{Understanding the Impact of Differentially Private Training on Memorization of Long-Tailed Data}



  \icmlsetsymbol{equal}{*}

  \begin{icmlauthorlist}
    \icmlauthor{Jiaming Zhang}{equal,yyy,yy2}
    \icmlauthor{Huanyi Xie}{equal,yyy}
    \icmlauthor{Meng Ding}{yy3}
    \icmlauthor{Shaopeng Fu}{yyy}
    \icmlauthor{Jinyan Liu}{yy4}
    \icmlauthor{Di Wang}{yyy}
  \end{icmlauthorlist}

  \icmlaffiliation{yyy}{King Abdullah University of Science and Technology}
    \icmlaffiliation{yy2}{Renmin University of China}
  \icmlaffiliation{yy3}{State University of New York at Buffalo}
    \icmlaffiliation{yy4}{Beijing Institute of Technology}

  \icmlcorrespondingauthor{Di Wang}{di.wang@kaust.edu.sa}
  \icmlcorrespondingauthor{Meng Ding}{mengding@buffalo.edu}

  \icmlkeywords{Machine Learning, ICML}

  \vskip 0.3in]



\printAffiliationsAndNotice{}  

\begin{abstract}
Recent research shows that modern deep learning models achieve high predictive accuracy partly by memorizing individual training samples. Such memorization raises serious privacy concerns, motivating the widespread adoption of differentially private training algorithms such as DP-SGD. However, a growing body of empirical work shows that DP-SGD often leads to suboptimal generalization performance, particularly on long-tailed data that contain a large number of rare or atypical samples. Despite these observations, a theoretical understanding of this phenomenon remains largely unexplored, and existing differential privacy analysis are difficult to extend to the nonconvex and nonsmooth neural networks commonly used in practice.
In this work, we develop the first theoretical framework for analyzing DP-SGD on long-tailed data from a feature learning perspective. We show that the test error of DP-SGD-trained models on the long-tailed subpopulation is significantly larger than the overall test error over the entire dataset. Our analysis further characterizes the training dynamics of DP-SGD, demonstrating how gradient clipping and noise injection jointly adversely affect the model’s ability to memorize informative but underrepresented samples. Finally, we validate our theoretical findings through extensive experiments on both synthetic and real-world datasets.

\end{abstract}

\section{Introduction}
Memorization, in the classical statistical regime, is traditionally viewed as a hindrance to generalization, typically mitigated through explicit regularization~\citep{ruppert2004elements}. However, modern deep learning models are often trained in highly overparameterized regimes, where empirical evidence shows that models can perfectly interpolate the training data while still generalizing well~\citep{zhang2019learning,bartlett2020benign,cao2022benign,kou2023benign}. Moreover, incorporating atypical, long-tailed data during training has been shown to be crucial for attaining high predictive accuracy~\citep{carlini2019distribution,feldman2020does}.

However, this propensity for memorization raises significant privacy concerns, especially in sensitive domains such as finance~\citep{lundervold2019overview,chlap2021review}, healthcare~\citep{zhou2024personalized,zhou2024ppml,bi2024advanced}, and user-centric applications~\citep{oroojlooy2023review,zhang2025towards}. Models that memorize specific training samples are vulnerable to privacy attacks, including membership inference and data extraction~\citep{shokri2017membership,yeom2018privacy,carlini2019secret}. To mitigate these risks, Differential Privacy (DP)~\citep{dwork2006calibrating}, particularly through the lens of DP-SGD~\cite{abadi2016deep}, has become the de facto standard for providing formal privacy guarantees by limiting the influence of individual samples. While DP compromises model generalization, this degradation is particularly acute on long-tailed data. Recent empirical findings suggest that DP particularly hinders the recognition of atypical and underrepresented sub-populations at test time~\citep{carlini2019distribution}. This phenomenon manifests as a non-uniform utility drop that exacerbates performance gaps for disadvantaged and data-complex groups~\citep{bagdasaryan2019differential}.

Despite these empirical observations, a rigorous theoretical understanding of how DP-SGD influences memorization remains an open challenge. Prior research on DP has been largely confined to deriving utility bounds, often overlooking the underlying training dynamics~\cite{xue2021differentially,xiang2023practical,wangprivate,huai2020pairwise,xu2025beyond,wang2020empirical,wang2019noninteractive}. Furthermore, these analyses frequently hinge on assumptions of convexity and smoothness, which fail to characterize the non-convex, non-smooth, and high-dimensional nature of modern large-scale neural networks~\citep{dwork2006calibrating,chaudhuri2011differentially,bassily2014private,shen2023differentially}. Recently, \citet{ding2025understanding} leveraged a feature learning perspective to analyze the significance of feature augmentation in the context of DP. \citet{shi2025towards} demonstrated that, provided the signal-to-noise ratio is not excessively low, DP-GD can achieve superior generalization performance compared to standard GD while simultaneously ensuring formal privacy guarantees. Meanwhile, \citet{xuunderstanding} investigated the disparate impact, adversarial robustness, and private fine-tuning under DP-SGD by leveraging a unified framework based on two-layer ReLU CNNs. However, neither study explicitly addresses the impact of Differential Privacy on the memorization of long-tailed data.

To fill this gap, we approach the problem from a feature learning perspective, focusing on how neural networks dynamically extract task-relevant signals versus task-irrelevant components, and how such learning dynamics ultimately govern the model's generalization performance \citep{allen2020towards,allen2022feature,cao2022benign,kou2023benign}. While feature learning has been studied extensively in recent years, it has seldom addressed the unique confluence of gradient clipping, noise injection, and long-tailed data distributions.
To this end, we first formulate a multi-class classification task using a ReLU-activated CNN optimized via standard DP-SGD. Furthermore, motivated by \citet{xu2025rethinking}, we define a class-dependent noise, which allows us to rigorously characterize the memorization of implicit, class-specific information during the training process.

We conclude our contributions as follows:
\begin{itemize}
    \item \textbf{A Novel Theoretical Framework for DP-SGD on Long-tailed Data.} To the best of our knowledge, we are the first to establish a theoretical feature learning framework specifically for analyzing DP-SGD on long-tailed data. Leveraging a multi-patch data structure, we rigorously characterize the update dynamics of two-layer CNNs during the DP-SGD training process.

    \item \textbf{Theoretical Insights into Training Dynamics and Generalization.} We derive formal expressions for the training dynamics and test error under DP-SGD, providing a comparative analysis between full-set and long-tailed data distributions. By deconstructing the specific impacts of gradient clipping and noise injection, we demonstrate that DP-SGD significantly suppresses the {memorization of implicit class-specific features}, which fundamentally leads to degraded generalization performance on long-tailed data.

    \item \textbf{Empirical Validation on Synthetic and Real-world Data.} We validate our theoretical findings using two-layer CNNs and LeNet architectures on synthetic data, MNIST, and CIFAR-10. Our experiments verify the predicted training dynamics of DP-SGD regarding the memorization process and confirm its impact on the test error for both full data and long-tailed data distributions.
\end{itemize}

\section{Related Work}
\textbf{Empirical and Theoretical Studies of Memorization.}
Several empirical research has explored the phenomenon of memorization in neural networks. \citet{feldman2020does} argues that memorizing labels—including those of outliers and noisy samples—is indispensable for attaining near-optimal generalization error. Furthermore, \citet{carlini2019distribution} demonstrates that models must memorize anomalous samples to achieve high confidence during training. Despite these profound empirical insights, a rigorous theoretical foundation for such phenomena remains relatively elusive. 

Recently, several theoretical studies have attempted to model memorization through the lens of feature learning, typically bifurcating overfitting into "benign" signal learning and "harmful" noise memorization \citep{cao2022benign,kou2023benign}. However, these frameworks predominantly characterize memorization as a detrimental process. This perspective stands in stark contrast to the aforementioned empirical findings, which suggest that memorization is not merely a byproduct of training but can actually be beneficial for achieving high-precision generalization performance~\citep{carlini2019distribution}.

\textbf{Theory on Differentially Private Learning.}
Extensive work has focused on differential privacy, including classical results for private empirical risk minimization\citep{chaudhuri2011differentially,bassily2014private,wang2019differentially,wang2019differentially,wang2020empirical,xiao2023theory,feldman2020private} and private stochastic convex optimization \citep{bassily2019private,feldman2020private}. These studies have been further extended to various settings, such as heavy-tailed distributions\citep{wang2020differentially,hu2022high,tao2022private} and non-Euclidean spaces\cite{bassily2021differentially,asi2021private,su2023differentially}. However, the theoretical understanding of private deep learning remains relatively limited, especially from the perspective of feature learning. Recently, while DP for ReLU regression has been well studied~\cite{ding2025nearly,ding2024revisiting}, only a few works have investigated DP in two-layer ReLU neural networks. \citet{ding2025understanding} highlight the importance of feature augmentation in DP, and \citet{shi2025towards} demonstrate that DP-GD can achieve superior generalization compared to GD under certain conditions, whereas \citet{xuunderstanding} theoretically establish that DP can induce disparate impacts across different subpopulations. Nevertheless, these studies are primarily restricted to standard binary classification, where the training noise is sampled from a Gaussian distribution. In contrast, we extend the data distribution to a multi-class setting, where the noise incorporates class-dependent implicit patterns, and investigate the performance disparity of DP-SGD between all data and long-tailed data.

\textbf{The Fairness of Privacy.}
Recently, researchers have recognized that privacy-preserving mechanisms can inadvertently exacerbate algorithmic bias. \citet{bagdasaryan2019differential} first demonstrated that DP models often exhibit a "disparate impact," where the accuracy degradation for minority classes is significantly more pronounced than for majority classes. This phenomenon was further investigated by \citet{pujol2020fair}, who showed that DP noise can lead to inequitable resource allocation in decision-making systems. Beyond empirical observations, \citet{cummings2019compatibility} provided theoretical evidence that even with full distributional access, DP constraints can impede the exact achievement of fairness parity, such as the Equality of False Positives and False Negatives. Despite these insights, a rigorous theoretical understanding of how DP affects fairness specifically within the context of long-tailed data distributions remains largely under-explored.

\section{Problem Setup}
In this section,we introduce the necessary definitions and formally describe the training of two-layer CNNs via DP-SGD under the multi-patch data distribution.

\textbf{Notations.} We use lowercase letters, lowercase boldface letters, and uppercase boldface letters to denote scalars, vectors, and matrices, respectively. For two sequences $\{x_n\}$ and $\{y_n\}$, we write $x_n = O(y_n)$ if there exist absolute constants $C > 0$ and $N > 0$ such that $|x_n| \leq C|y_n|$ for all $n \geq N$. Similarly, we write $x_n = \Omega(y_n)$ if $y_n = O(x_n)$. We say $x_n = \Theta(y_n)$ if both $x_n = O(y_n)$ and $x_n = \Omega(y_n)$. Finally, we use $\tilde{O}(\cdot)$, $\tilde{\Omega}(\cdot)$, and $\tilde{\Theta}(\cdot)$ to denote the corresponding notations with logarithmic factors suppressed.

\begin{definition}[Data Generation Model]\label{def:data_model}
 Let $\mathbf{u}_1,..,\mathbf{u}_k\in\mathbb{R}^d$ be fixed vectors representing the signals contained in data points, where $\langle \mathbf{u}_i, \mathbf{u}_j \rangle = 0$ for all $i, j \in [K]$ and $i \neq j$. Then, each data point $(\mathbf{x}, y)$ with $\mathbf{x}=(\mathbf{x}^{(1)}, \mathbf{x}^{(2)})\in\mathbb{R}^{2d}$ and $y\in[K]$ is generated from the following distribution $\mathcal{D}$:

\begin{enumerate}
    \item Sample the label $y$ following a distribution $\mathcal{K}$, whose support is $[K]$ ($K = \Theta(1)$).
    \item The noise vector $\boldsymbol{\xi}$ is generated as $\mathbf{A}_y \boldsymbol{\zeta}$, where each coordinate of $\boldsymbol{\zeta}$ is i.i.d. drawn from $\mathcal{D}_\zeta$, a symmetric $\sigma_p = \Theta(1)$ sub-Gaussian distribution with variance 1, and $\mathbf{A}_y \in \mathbb{R}^{d \times d}$ satisfies $\mathbf{u}_k^\top \mathbf{A}_y = 0$, for any $k \in [K]$.
    \item One of $\mathbf{x}^{(1)}, \mathbf{x}^{(2)}$ is given as $\mathbf{u}_y$, which represents the signal, the other is given by $\boldsymbol{\xi}$, which represents noises.
    %
    
\end{enumerate}

\end{definition}
Given the inherent heterogeneous class-dependent structures present in real-world datasets, we follow the data generation model proposed by \citet{xu2025rethinking}, where noise is modeled as class-dependent. Since achieving generalization on long-tailed data necessitates the memorization of noise patterns, our data generation model provides a framework to characterize such samples (formally defined in Definition~\ref{def:data_model}). Within this framework, a larger eigenvalue spectrum of $\mathbf{A}_k$ corresponds to a higher degree of noise heterogeneity. This approach extends the feature-noise data distribution, which traditionally assumes homogeneous data noise, a model widely adopted in feature learning literature \citep{cao2022benign, kou2023benign}.

\textbf{Two-layer CNN.} We consider a two-layer CNN with ReLU activation, where the first layer comprises $m$ filters (neurons) for each of the $K$ classes, and the second-layer parameters are fixed at $1/m$ as \citet{cao2022benign,kou2023benign}. 
Given an input $\mathbf{x} = (\mathbf{x}^{(1)}, \mathbf{x}^{(2)})$, the model with weights $\mathbf{W}$ produces a $K$-dimensional output vector $[F_1, \dots, F_K]^\top$, whose $k$-th component is defined as:
\begin{equation}\label{eq:model}
    F_k(\mathbf{W}, \mathbf{x}) = \frac{1}{m} \sum_{r=1}^{m} \sum_{j=1}^{2} \sigma \left( \left\langle \mathbf{w}_{k,r}, \mathbf{x}^{(j)} \right\rangle \right),
\end{equation}
where $\sigma(z) = \max\{0, z\}$ is the ReLU activation function, and $\mathbf{w}_{k,r}$ denotes the weight vector of the $r$-th filter associated with class $k$.

\textbf{Loss function.} Given a training dataset with $n$ samples $\mathcal{S} = \{(\mathbf{x}_i, y_i)\}_{i=1}^n$ drawn from the distribution $\mathcal{D}$, we train the neural network by minimizing the empirical risk with the cross-entropy loss:
\begin{equation}
    \mathcal{L}_{\mathcal{S}}(\mathbf{W}) = \frac{1}{n} \sum_{i=1}^{n} \mathcal{L}(\mathbf{W}, \mathbf{x}_i, y_i),
\end{equation}
where $\mathcal{L}(\mathbf{W}, \mathbf{x}, y) = -\log(\text{logit}_y(\mathbf{W}, \mathbf{x}))$ and $\text{logit}(\cdot)$ represent the output probabilities of the neural network:
\begin{equation}
    \text{logit}_y(\mathbf{W}, \mathbf{x}) = \frac{\exp(F_y(\mathbf{W}, \mathbf{x}))}{\sum_{k=1}^{K} \exp(F_k(\mathbf{W}, \mathbf{x}))}.
\end{equation}

To ensure privacy preservation, the learned weights $\mathbf{W}$ must conform to the formal definition of DP\citep{dwork2006calibrating}:
\begin{definition} [($\epsilon, \delta_{DP}$)-Differential privacy.] 
{A randomized algorithm $\mathcal{M} : \mathcal{Z} \to \mathcal{R}$ is $(\epsilon, \delta)$-DP if, for every pair of neighboring datasets $Z, Z' \in \mathcal{Z}$ that differ by one entry, and for any subset of outputs $S \subseteq \mathcal{R}$, the following holds: $\mathbb{P}[\mathcal{M}(Z) \in S] \le e^{\epsilon} \mathbb{P}[\mathcal{M}(Z') \in S] + \delta_{DP}$.}
    
\end{definition}
\textbf{DP-SGD Training algorithm.} DP-SGD \citep{abadi2016deep}—which consists of gradient clipping and noise injection—is the standard training algorithm for differentially private deep learning. We train the neural network using DP-SGD with a learning rate $\eta$, i.e.,
\begin{equation}\label{eq:updaterule}
\resizebox{0.49\textwidth}{!}{
$\begin{aligned}
    \mathbf{W}^{(t+1)} = \mathbf{W}^{(t)} - \frac{\eta}{B} \sum_{(\mathbf{x}, y) \in \mathcal{S}^{(t)}} \text{clip}_C\left(\nabla \mathcal{L}(\mathbf{W}^{(t)}, \mathbf{x}, y)\right)+\eta\cdot\mathbf n^{(t)}.
\end{aligned}
$}
\end{equation}
$\mathcal{S}^{(t)}$ represents a mini-batch of size $B$ randomly selected at iteration $t$, $\mathbf{n}^{(t)}$ is the noise added for privacy protection, sampled from $\mathcal{N}(0, \sigma_n^2 \mathbf{I})$, and $\text{clip}_C(\mathbf{x})$ is the gradient clipping function with a clipping threshold $C$ on vector $\mathbf{x}$: 
\(
\text{clip}_C(\mathbf{x}) = \frac{\mathbf{x}}{\max\{1, \|\mathbf{x}\|_2 / C\}}.
\)

The initial weights of the neural network's parameters are generated i.i.d. from a Gaussian distribution, i.e., $\mathbf{w}_{j,r}^{(0)} \sim \mathcal{N}(0, \sigma_0^2 \mathbf{I})$, for all $j \in [K], r \in [m]$.

\section{Main Results}
In this section, we present our main theoretical results. We characterize the training dynamics of memorization under DP-SGD, and show that privacy protection induces suboptimal training loss. Furthermore, we analyze the resulting test error on both the entire data distribution and long-tailed data, highlighting the disparate impact of differential privacy on these regimes.
We first introduce the following conditions.

\begin{condition}\label{condition}
Suppose there exists a sufficiently large constant $c_1$ and $0<c_2<1$. For certain probability parameter $\delta \in (0, 1)$, the following conditions hold:

    (a) To ensure that the neurons can learn the data patterns, for any $i, j \in [K]$, the noise patch distributions satisfy:
    \[
    \begin{cases}
    \operatorname{Tr}(\mathbf{A}_i^\top \mathbf{A}_i) \geq c_1 n \max \left\{ \|\mathbf{A}_i^\top \mathbf{A}_j\|_F \log(n^2/\delta), \right. \\
    \left. n^{1/2} \max_{i,j \in [K]} \{ \|\mathbf{A}_i^\top \mathbf{A}_j\|_F^{1/2} \} \log^{1/2}(n^2/\delta) / |\mathcal{S}_i| \right\}, \\
    \|\mathbf{A}_i^\top \mathbf{A}_j\|_F / \|\mathbf{A}_i^\top \mathbf{A}_j\|_{\text{op}} \geq c_1 \sqrt{\log(K/\delta)}, \\
    \|\mathbf{A}_i^\top \mathbf{A}_j\|_F \geq c_1 \max_{k \neq j} \{ \|\mathbf{A}_i^\top \mathbf{A}_k\|_F \}.
    \end{cases}
    \]
    Moreover, there exists a threshold $c' > 0$ such that $\mathbb{P}[\zeta > c'] \geq 0.4$. 

    (b) To ensure that the learning problem is in a sufficiently over-parameterized setting, the training dataset size $n$, network width $m$, and dimension $d$ satisfy:
    \[
    \begin{cases}
    m \geq c_1 \log(n/\delta) \max_i \{ (\lambda_{\max}^+(\mathbf{A}_i))^2 / (\lambda_{\min}^+(\mathbf{A}_i))^2 \}, \\
    n \geq c_1 \log(m/\delta), m \geq \Omega \left( \log(n/\delta) \log(T)^2 / n\sigma_0^2 \right), \\
    \min\{m, d, \operatorname{rank}(\mathbf{A}_j)\} - 0.9m \geq Cn, \\
    d \geq c_1 \log(mn/\delta).
    \end{cases}
    \]
    (c) To ensure that DP-SGD can minimize the training loss, the learning rate $\eta$, the batch size $B$ and initialization $\sigma_0$ satisfy:
    \[
    \begin{cases}
    \scalebox{0.9}{$ \eta \le \left(c_1(C + \sqrt{d}\sigma_n)(\max_{k} \|\mathbf{u}_{k}\|_2 + \sqrt{1.5 \operatorname{Tr} (\mathbf{A}_k^\top \mathbf{A}_k)})\right)^{-1} $},\\
    \eta \leq m n \log(T) / \max_{j \in [K]} \{ \operatorname{Tr}(\mathbf{A}_j^\top \mathbf{A}_j) \}, \\
    B\geq c_2 \cdot n,\\
    \sigma_0 \leq c_1^{-1} n^{-1} \phi \cdot \\
    \scalebox{0.92}{$\left( \max_{k \in [K]} \{ \sqrt{\log(K m / \delta)} \|\mathbf{u}_k\|_2, \log(K m / \delta) \|\mathbf{A}_k\|_F \} \right)^{-1}$},
    \end{cases}
    \]
    where $\phi := \min_{k_1, k_2 \in [K]} \{ \|\mathbf{A}_{k_1}^\top \mathbf{A}_{k_2}\|_F, \|\mathbf{u}_{k_1}\|_2^2 \}$.

\end{condition}
Similar conditions are widely made in the theoretical analysis of feature learning in neural networks \citep{xu2025rethinking,cao2022benign,kou2023benign}. Compared to the conditions in \citet{xu2025rethinking}, we impose a condition on the batch size $B$ and relax the condition on the learning rate $\eta$, as the model cannot converge to an arbitrarily small term. This will be discussed in detail later in Theorem~\ref{thm:trainloss}.

\begin{assumption}[$s$-non-perfect model]\label{ass:non-perfect} 
We assume that the model is almost surely not perfect on any test example, i.e., for some constant $s > 0, \mathcal{L} \left( \mathbf{W}^{(t)}, \mathbf{x}, y \right) \geq s$, for all $(\mathbf{x}, y) \sim \mathcal{D}$.
\end{assumption}
Assumption~\ref{ass:non-perfect} is relatively mild, particularly in light of the inherent stochasticity introduced by DP-SGD during the training process. Consequently, the resulting model is stochastic, making it highly improbable to attain zero cross-entropy loss on any given test sample.

We denote $\Lambda_k=\frac{C}{\|\mathbf{u}_k\|_2+\|\mathbf{A}_k\|_F}$ as the clipping factor for class k, which quantifies the maximum change in gradient magnitude. Then, we first demonstrate the occurrence of memorization during training from the perspective of the training dynamics, as formalized in the theorem below.
\begin{theorem}[noise pattern memorization]\label{thm:mem}
     Under Condition~\ref{condition} and Assumption~\ref{ass:non-perfect}, for any $(\mathbf{x}, y)$ in the training dataset $\mathcal{S}$, after $T \geq \Omega \left( (\eta \Lambda_y \|\mathbf{A}_y\|_F)^{-1} n \sqrt{m} \sigma_0 \right)$ iterations, with probability at least $1 - \delta$, the inner product between the weight vector $\mathbf{w}_{y,r}^{(T)}$ and the noise vector $\boldsymbol{\xi}$ satisfies $\langle \mathbf{w}_{y,r}^{(T)}, \boldsymbol{\xi} \rangle \geq U$ , where:
\begin{equation*}
\resizebox{1\linewidth}{!}{
$\begin{aligned}
    U= \Omega\left(\sum_{t=0}^{T-1}\frac{\eta\Lambda_y}{n\sqrt{m}} \left( 1 - \text{logit}_y(\mathbf{W}^{(t)}, \mathbf{x}) \right) \|\boldsymbol{\xi}\|_2^2\right)-\mathcal{O}\left(\eta \sigma_n\|\mathbf{A_y}\|_F)\right) 
\end{aligned}$}
\end{equation*}

\end{theorem}
\begin{remark}
    The inner product $\langle \mathbf{w}_{y,r}^{(T)}, \boldsymbol{\xi} \rangle$ quantifies the degree of alignment between the neuron weights and the noise, which serves as a proxy for the extent of noise pattern memorization by the model. Theorem~\ref{thm:mem} suggests that this memorization effect accumulates over iterations(as $(1 - \operatorname{logit}_y(\mathbf{W}^{(t)}, \mathbf{x}))\geq0$ ). Notably, DP-SGD mitigates this undesirable process through two mechanisms: first, gradient clipping (manifested as $\Lambda_y < 1$) scales down the magnitude of updates that contribute to memorization; second, the injection of Gaussian noise introduces stochastic perturbations that disrupt the fine-grained alignment. 
\end{remark}
Based on Condition~\ref{condition}, we study the model generalization performance by bounding the test error (accuracy) of the trained model $\mathbf{W}^{(T)}$ on class-conditional distribution $\mathcal{D}_k$ whose probability density function is $\mathbb{P}_{\mathcal{D}_k}[(\mathbf{x},y)] = \mathbb{P}_{\mathcal{D}}[(\mathbf{x},y)|y = k]$, i.e., for all $k \in [K]$,
\begin{equation*}
\begin{aligned}
    & \mathcal{L}_{\mathcal{D}_k}(\mathbf{W}^{(T)}) \\
    = & \mathbb{P}_{(\mathbf{x},y) \sim \mathcal{D}_k} \left[ F_y(\mathbf{W}^{(T)}, \mathbf{x}) \neq \max_{j \in [K]} \{ F_j(\mathbf{W}^{(T)}, \mathbf{x}) \} \right].
\end{aligned}
\end{equation*}
For each class, we define the following long-tailed data. 
\begin{definition}[$L$-Long-tailed data set]\label{def:longtail}
The $L$-long-tailed data distribution $\mathcal{T}_j$ for each $j \in [K]$ with model $\mathbf{W}^{(T)}$ is defined as
\begin{equation*}
\begin{aligned}
&\mathbb{P}_{\mathcal{T}_j}[(\mathbf{x}, y)] =\\& \mathbb{P}_{\mathcal{D}_j} \left[ (\mathbf{x}, y) \middle| \left\langle  \mathbf{w}_{y,r}^{(T)}, \boldsymbol{\xi} \right\rangle \geq L \left\| \mathbf{A}_y^\top  \mathbf{w}_{y,r}^{(T)} \right\|_2 \right],
\end{aligned}
\end{equation*}
\end{definition}
Definition~\ref{def:longtail} identifies data whose equivalent noise $\zeta'$ exceeds a threshold. Specifically, for data $(\mathbf{x}, y) \sim \mathcal{D}$, the inner product term $\langle  \mathbf{w}_{y,r}^{(T)}, \boldsymbol{\xi} \rangle$ satisfies $\langle  \mathbf{w}_{y,r}^{(T)}, \boldsymbol{\xi} \rangle = \Theta \left( \left\| \mathbf{A}_y^\top  \mathbf{w}_{y,r}^{(T)} \right\|_2 \zeta' \right)$, where $\zeta'$ is an equivalent random sub-Gaussian variable with variance 1. Notably, we define long-tailed data using model weights obtained from non-DP (clean) training, and subsequently evaluate models trained with DP-SGD on these data points.

Definition~\ref{def:longtail} selects training data points whose noise patterns align better with the model weights, indicating that these samples are more effectively memorized. Due to the concentration phenomenon in high-dimensional settings, we do not define long-tailed data based on norms. Instead, we identify “long-tailed” samples using the trained model, drawing inspiration from \citet{xu2025rethinking, feldman2020neural}. An illustrative visualization is provided in Appendix~\ref{sec:illu}.

We denote $S_j$ as the set containing training data with label $j$ in the training dataset $S$. Then, we characterize an upper bound on the training loss of models trained with DP-SGD, highlighting the suboptimality induced by privacy preservation.
\begin{theorem}[Training loss]\label{thm:trainloss}
    Under Condition~\ref{condition} and Assumption~\ref{ass:non-perfect}, for any $(\mathbf{x},y)\in \mathcal{S}$, with probability at least $1 - \delta$, the training loss satisfies:
    \begin{equation*}
    \resizebox{1\linewidth}{!}{
    $\begin{aligned}
        \mathcal{L}_S(\mathbf{W}^{(T)},\mathbf x,y) &\leq \underbrace{\exp \left( -\Omega \left( \frac{\eta T\Lambda_{y} |\mathcal{S}_y| }{n\sqrt{m}}\|\mathbf{u}_{y}\|_2^2  \right) \right) \mathcal{L}(\mathbf{W}^{(0)},\mathbf{x},y)}_{\text{Vanishing error}} \\&+ \underbrace{\mathcal{O} \left( \frac{n\sqrt{m}\cdot\sigma_n\sqrt{d}( \|\mathbf{u}_{y}\|_2+\|\mathbf{A}_y\|_F)}{\Lambda_{y} |\mathcal{S}_y|\|\mathbf{u}_{y}\|_2^2} \right)}_{\text{Privacy protection error}}.       
    \end{aligned}
    $}
    \end{equation*}
\end{theorem}

\begin{remark}
    Theorem~\ref{thm:trainloss} characterizes the training loss under DP-SGD, which is composed of a vanishing error and a privacy protection error. While the vanishing error decays exponentially toward zero as $T$ increases, the privacy protection error persists as an irreducible residual. This term prevents the training loss from converging to an arbitrarily small value, representing a significant departure from the "benign overfitting" phenomenon observed in standard overparameterized regimes \citep{cao2022benign,kou2023benign}. 
   Furthermore, our analysis indicates that the privacy protection error increases as the privacy budget $\epsilon$ decreases, highlighting the inherent privacy-utility trade-off.
\end{remark}

We define the signal-to-noise ratio (SNR) as $\frac{|\mathcal{S}_k| \|\mathbf{u}_k\|_2^2}{\sqrt{|\mathcal{S}_j|} \|\mathbf{A}_k^\top \mathbf{A}_j\|_F}$ and the noise correlation ratio (NCR) as $\frac{\sqrt{|\mathcal{S}_k|} \|\mathbf{A}_k^\top \mathbf{A}_k\|_F}{\sqrt{|\mathcal{S}_j|} \|\mathbf{A}_k^\top \mathbf{A}_j\|_F}$. We present our main result in the following theorem.
\begin{theorem}[Test error]\label{thm:testerror}
For any $k \in [K]$, under Condition~\ref{condition} and Assumption~\ref{ass:non-perfect}, there exists \\$T = \tilde{\mathcal{O}}(\eta^{-1}C^{-1}n\sqrt{m} )$ with probability at least $1 - \delta$, the following holds:

        1. (For all data) When the signal-to-noise ratio is large, i.e., $|\mathcal{S}_k|\Lambda_k \|\mathbf{u}_k\|_2^2 \geq C_1 \cdot (\max_{j \neq k} \{\sqrt{|\mathcal{S}_j|} \|\mathbf{A}_k^\top \mathbf{A}_j\|_F\}+n\sigma_n(\|\mathbf{u}_k\|_2+\|\mathbf{A}_k\|_F))$, the test error satisfies:
        \begin{equation*}
        \small
        \begin{aligned}
            &\mathcal{L}_{\mathcal{D}_k}(\mathbf{W}^{(T)})\leq  \\
            &\sum_{j \neq k} \exp \left[ -c_1 \cdot\left( \frac{|\mathcal{S}_k|\Lambda_k \|\mathbf{u}_k\|_2^2-n\sigma_n\|\mathbf{u}_k\|_2\sqrt{2\log(2/\delta)}}{\sqrt{|\mathcal{S}_j| }\|\mathbf{A}_k^\top \mathbf{A}_j\|_F+n\sigma_n\|\mathbf{A}_k\|_F} \right)^2\right]
        \end{aligned}
        \end{equation*}
    2. (Only for long-tailed data) When the noise correlation ratio is large, i.e., $\sqrt{|\mathcal{S}_k|}\Lambda_k \|\mathbf{A}_k^\top \mathbf{A}_k\|_F \geq C_2 \cdot (\max_{j \neq k} \{\sqrt{|\mathcal{S}_j| }\|\mathbf{A}_k^\top \mathbf{A}_j\|_F\}+n\sigma_n\sqrt{d+L^2}\|\mathbf{A}_k\|_{op}+n\sigma_n\|\mathbf{A}_k\|_F)$, the test error satisfies:
    \begin{equation*}
    \resizebox{1\linewidth}{!}{
        $\begin{aligned}
            &\mathcal{L}_{\mathcal{T}_k}(\mathbf{W}^{(T)}) \leq
            \\& \sum_{j \neq k} \exp \left[ -c_2  \cdot  \left(\frac{L\sqrt{|\mathcal{S}_k|}\Lambda_k \|\mathbf{A}_k^\top \mathbf{A}_k\|_F-n\sigma_n\sqrt{d+L^2}\|\mathbf{A}_k\|_{op}}{\sqrt{|\mathcal{S}_j|} \|\mathbf{A}_k^\top \mathbf{A}_j\|_F+n\sigma_n\|\mathbf{A}_k\|_F} \right)^2\right]
        \end{aligned}
        $}
    \end{equation*}

Here, $\Lambda_k=\frac{C}{\|\mathbf{u}_k\|_2+\|\mathbf{A}_k\|_F}$, $C_1, C_2, C_3, c_1, c_2, c_3$ are some absolute constants. 
\end{theorem}

Theorem~\ref{thm:testerror} characterizes the test performance under DP-SGD for both the general distribution and the long-tailed sub-distribution. Our results reveal two distinct generalization pathways: (i) In the high signal-to-noise ratio regime, the model achieves superior generalization by successfully learning robust signal features (Statement 1). (ii) Conversely, when signal strength is insufficient, the model can still attain satisfactory performance on long-tailed data by memorizing class-specific noise patterns, provided that the noise correlation ratio is high (i.e., the inter-class noise heterogeneity is strong) (Statement 2). 

Furthermore, the theorem quantifies how DP-SGD fundamentally hinders generalization on both data categories. Specifically, the test error upper bounds scale positively with the noise injection variance $\sigma_n$. Compared to standard (clean) training ($\sigma_n=0$), DP-SGD necessitates significantly higher thresholds for either the signal-to-noise ratio or the noise correlation ratio to achieve the same level of generalization, highlighting the inherent utility cost of privacy-preserving optimization.

To further elucidate the implications of Theorem~\ref{thm:testerror}, we now conduct a comparative analysis to highlight the disproportionate degradation of test accuracy on long-tailed data under privacy constraints.
\begin{remark}[Disproportionate Impact of DP-SGD on Long-Tailed Data]\label{rm:11}
    Our analysis reveals several critical insights regarding the performance degradation of private models on long-tailed distributions:

    \begin{enumerate}
        \item \textbf{Disproportionate Sensitivity to Privacy Noise:}  DP-SGD exerts a more pronounced negative impact on long-tailed data compared to the general distribution, that is $\mathcal{L}_{\mathcal{D}_k} \leq \mathcal{L}_{\mathcal{T}_k}$. By comparing the two statements in Theorem~\ref{thm:testerror}, it is evident that long-tailed generalization is more vulnerable to $\sigma_n$. Consider an initial equilibrium in clean training ($\sigma_n=0$) where $\mathcal{L}_{\mathcal{D}_k} = \mathcal{L}_{\mathcal{T}_k}$, implying a balance between SNR and NCR ($|\mathcal{S}_k|^2 \|\mathbf{u}_k\|_2^4 \approx L^2 |\mathcal{S}_k| \|\mathbf{A}_k^\top \mathbf{A}_k\|_F^2$). Upon injecting privacy noise $\sigma_n$, the error bound for long-tailed data (Statement 2) increases much faster because its corresponding noise term, $n\sigma_n\sqrt{d+L^2}\|\mathbf{A}_k\|_{\text{op}}$, typically dominates the signal-side noise term $n\sigma_n\|\mathbf{u}_k\|_2\sqrt{2\log(2/\delta)}$. This follows from the fact that $\sqrt{d}\|\mathbf{A}_k\|_{\text{op}} \geq \|\mathbf{A}_k\|_F$, making the noise-memorization pathway significantly more fragile under differential privacy.

        \item \textbf{Consistency with Empirical Observations:} These theoretical findings echo existing empirical evidence. Prior studies have shown that rare subsets in training data are notably difficult for private models to identify during inference \citep{carlini2019distribution}. Furthermore, DP-SGD has been observed to cause a disproportionate drop in accuracy for minority or vulnerable subgroups \citep{bagdasaryan2019differential}. This phenomenon is fundamentally rooted in the fact that differential privacy inherently impedes the model's ability to memorize the fine-grained, atypical patterns required for classifying rare samples \citep{feldman2020does}.

        \item \textbf{Vanished Benefits of Atypical Data:} Under DP-SGD, the strategy of incorporating "longer-tailed" or more atypical data (i.e., increasing $L$) fails to enhance test accuracy. In standard training ($\sigma_n=0$), atypical samples (with larger $L$) can actually reduce test error by providing distinct class-dependent noise patches that the model can leverage. However, our results show that in the presence of DP-SGD, particularly when $\sqrt{|\mathcal{S}_k|}\Lambda_k \|\mathbf{A}_k^\top \mathbf{A}_k\|_F \leq n\sigma_n\sqrt{d}\|\mathbf{A}_k\|_{\text{op}}$, increasing $L$ no longer reduces the upper bound of the test error. Consequently, the model's performance remains poor even on highly atypical data, negating the potential accuracy gains from introducing such samples into the training set.
    \end{enumerate}
\end{remark}

\begin{remark}[Privacy-Utility Trade-off]
    To satisfy the $(\epsilon, \delta_{DP})$-DP requirement, the noise variance $\sigma_n$ is set as $\mathcal{O}\left( \frac{C\sqrt{T \ln(1/\delta_{DP})}}{n\epsilon} \right)$ according to the advanced composition theorem \citep{dwork2010boosting}. Substituting this into Theorem~\ref{thm:testerror}, we obtain the following explicit test error bound:
    \begin{equation*}
        \small
        \begin{aligned}
            &\mathcal{L}_{\mathcal{D}_k}(\mathbf{W}^{(T)})\leq  \\
            &\sum_{j \neq k} \exp \left[ -c_1 \cdot\left( \frac{\epsilon|\mathcal{S}_k|\Lambda_k \|\mathbf{u}_k\|_2^2-C\sqrt{T}\|\mathbf{u}_k\|_2}{\epsilon\sqrt{|\mathcal{S}_j| }\|\mathbf{A}_k^\top \mathbf{A}_j\|_F+C\sqrt{T}\|\mathbf{A}_k\|_F} \right)^2\right]
        \end{aligned}
        \end{equation*}
    This characterization yields two key insights regarding the impact of differential privacy: (i) a larger privacy budget $\epsilon$ (representing a more relaxed privacy guarantee) results in a smaller utility loss; (ii) the privacy protection error scales with the number of iterations $T$, suggesting that prolonged training under DP-SGD inevitably incurs a higher cost in terms of generalization performance.
\end{remark}
\begin{remark}[Compared with prior work]
    In the absence of privacy constraints—specifically, under the clean training setting where $\Lambda_k = 1$ and $\sigma_n = 0$—our results are consistent with the findings in \citet{xu2025rethinking}. Furthermore, by setting the noise structure as $\mathbf{A}_j \mathbf{A}_j^\top = \mathbf{I} - \sum_{k=1}^K \mathbf{u}_k \mathbf{u}_k^\top / \|\mathbf{u}_k\|_2^2$ for all $j \in [K]$, our Statement 1 recovers the same convergence orders reported in the standard benign overfitting literature \citep{kou2023benign}. 
\end{remark}


\section{Experiments}
\subsection{Experimental Setup}
\textbf{Datasets.}  For the experiments on real-world data, we utilize the MNIST and CIFAR-10 datasets, which have been widely demonstrated to have atypical long-tailed data~\citep{carlini2019distribution,feldman2020does} and have been studied for analyzing DP-SGD (see Figure~\ref{fig:long-tail} in Appendix~\ref{sec:illu} for an illustration). 

For the synthetic data experiments, each sample is generated according to the distribution specified in Definition~\ref{def:data_model}. Specifically, we parameterize the signal vectors as $u_k = U \cdot \|u_k\|_2 \cdot z_k$ for $k \in [K]$, where $z_k$ is the $k$-th standard basis vector (with the $k$-th entry being one and the others zero), and $U$ is a randomly generated unitary matrix. We set the number of classes to $K = 5$ and the signal dimension to $d = 1000$. For simplicity, we assume all signals have the same norm, i.e., $\|u_k\|_2$ is the same for all $k$. The coordinates $\zeta_i$ of the noise vector are drawn independently from either a Gaussian distribution $\mathcal{N}(0, I_d)$ or a Uniform distribution $\mathcal{U}(-\sqrt{3}, \sqrt{3})$. Regarding the matrices $\mathbf{A}_k$, we set all but one of their non-zero eigenvalues to $0.5$. The remaining non-fixed eigenvalue is then tuned to vary the noise correlation ratio accordingly. Finally, we generate $500$ samples for training and $500$ samples for evaluation, setting $|\mathcal{S}_k| = 100$ and ensuring the noise correlation ratio $\frac{\sqrt{|\mathcal{S}_k|} \|\mathbf{A}_k^\top \mathbf{A}_k\|_F}{\sqrt{|\mathcal{S}_j|} \|\mathbf{A}_k^\top \mathbf{A}_j\|_F}$ is identical for all pairs $k, j \in [K], k \neq j$.
\begin{figure*}[t] 
    \centering 
    \includegraphics[width=0.7\textwidth]{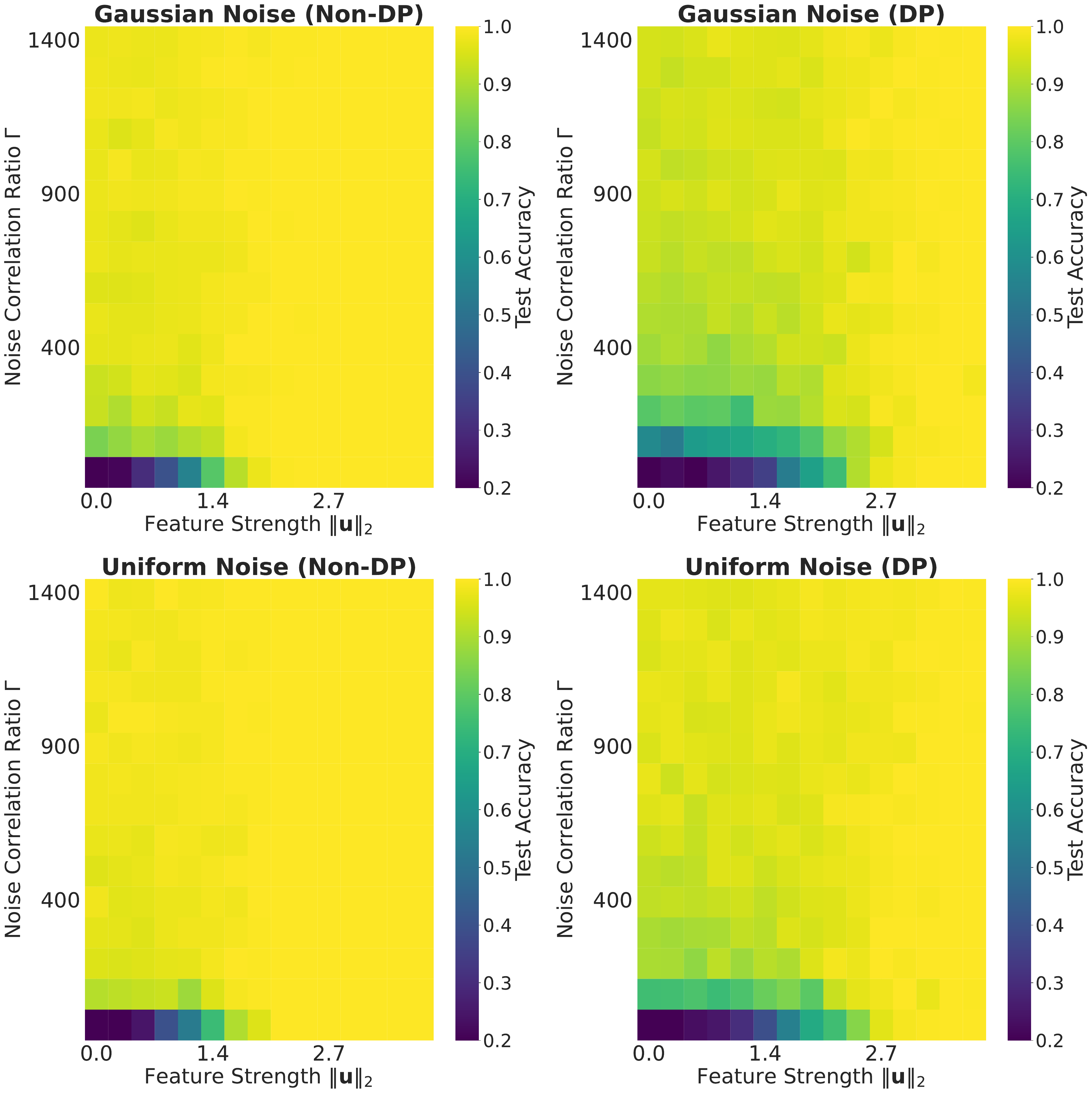} 
    \caption{Heatmap of test accuracy on synthetic data across various feature strength and noise correlation ratio}
    \label{fig:heatmap}
\end{figure*}

\textbf{Model architectures.} For synthetic data experiments, we adopt the two-layer neural network defined in \eqref{eq:model}, where $m$ is set to 100 neurons. For real-world data experiments, We implement LeNet\citep{lecun2002gradient} on MNIST and SmoothNets\citep{remerscheid2022smoothnetsoptimizingcnnarchitecture} on CIFAR-10. Model parameters in all experiments are initialized following the default initialization method of PyTorch.

\textbf{Model training \& evaluation.}
All models in our experiments are trained using DP-SGD. We set the privacy budget to $(\epsilon=8, \delta=1e-5)$, the gradient clipping threshold to $C=1.0$, and the learning rate to $\eta=0.002$. Each model is trained for $20$ epochs with a batch size of $B=256$. Finally, we report the test accuracy on both the entire dataset and the long-tailed subset.
\subsection{Results Analysis}
\begin{figure*}[t] 
    \centering
    \begin{subfigure}{0.42\textwidth} 
        \centering
        \includegraphics[width=\textwidth]{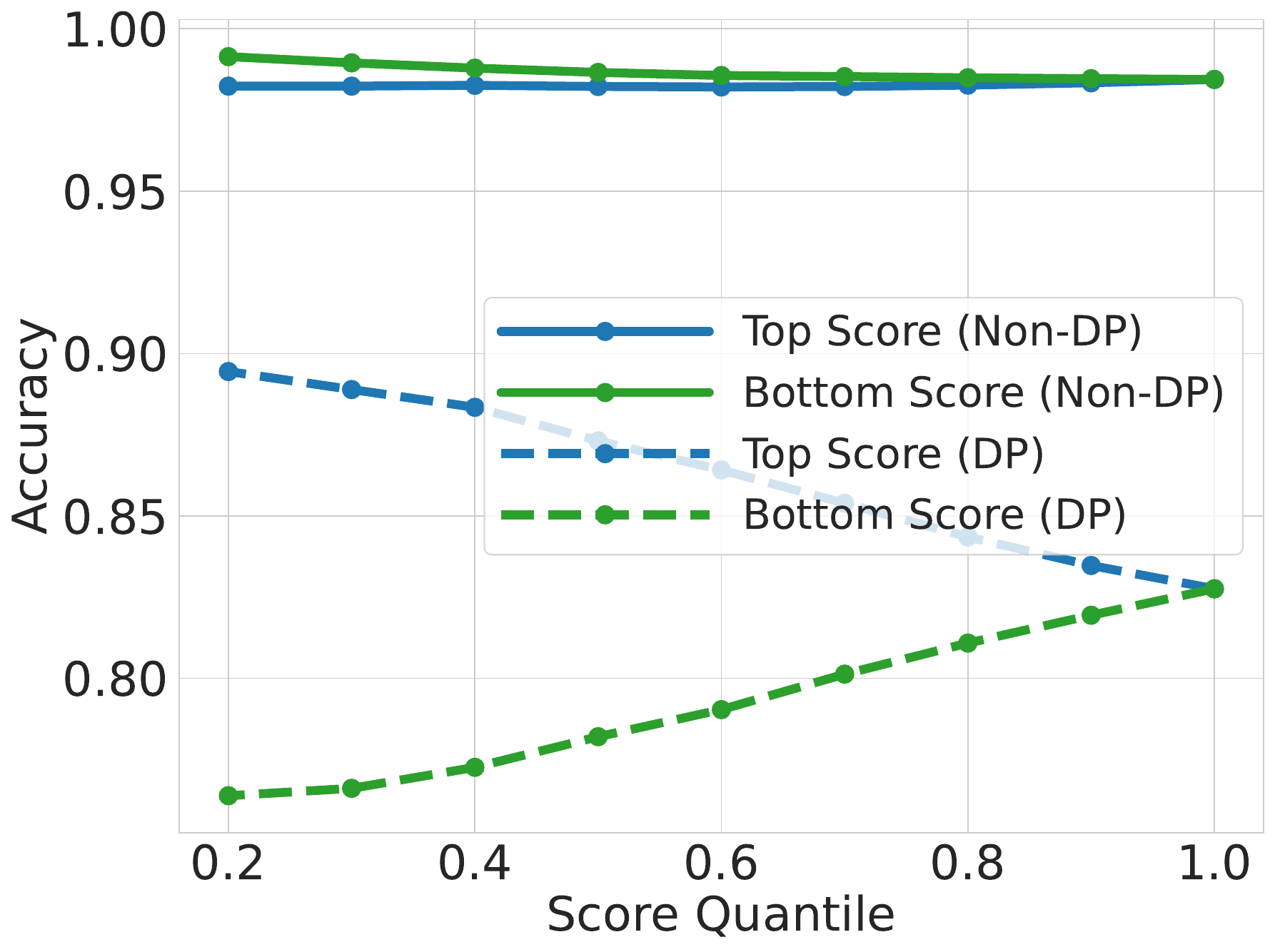}
        \caption{MNIST}
        \label{fig:accuracy_mnist}
    \end{subfigure}
    \hspace{1.6cm} 
    \begin{subfigure}{0.42\textwidth} 
        \centering
        \includegraphics[width=\textwidth]{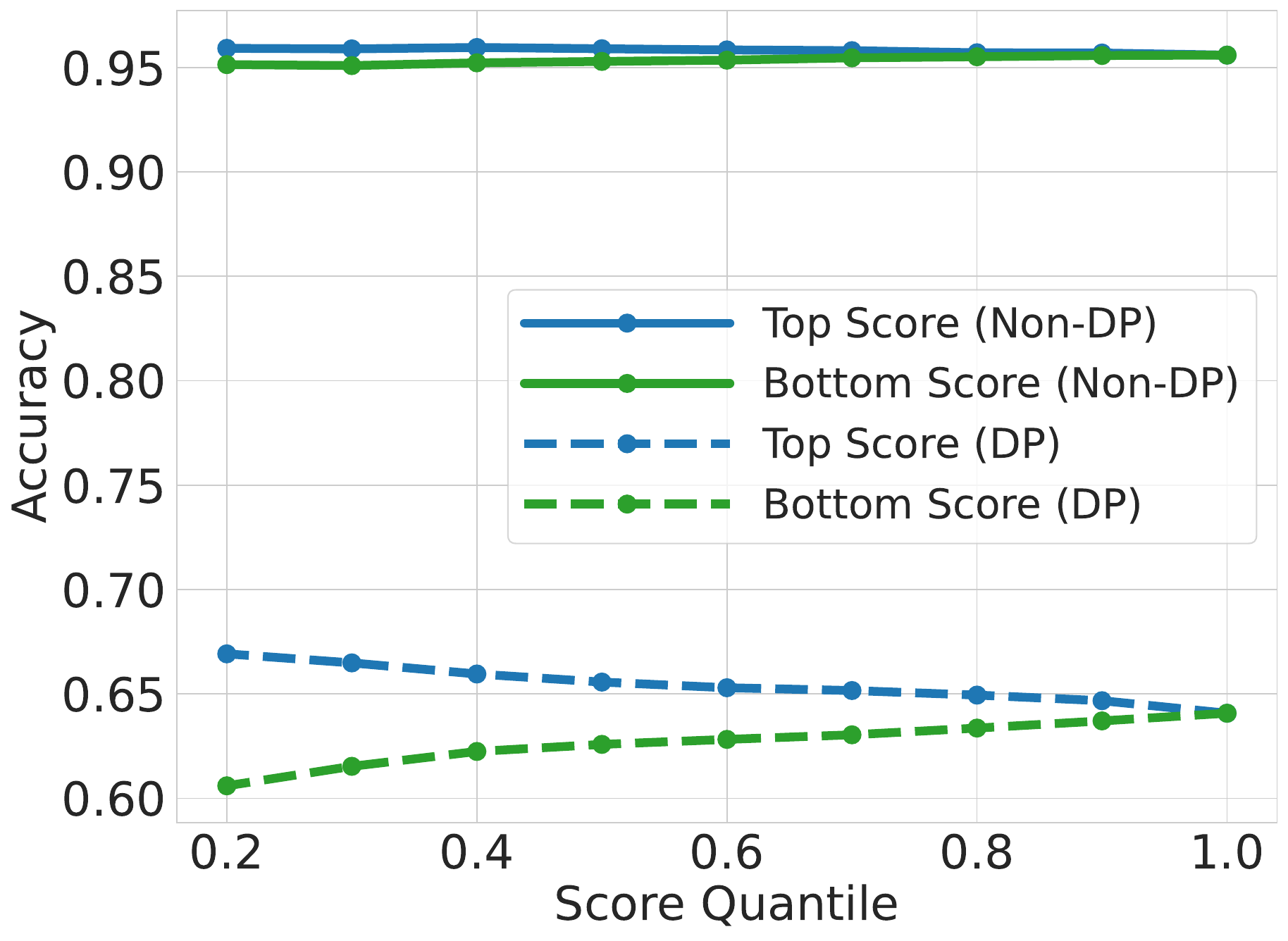}
        \caption{CIFAR-10}
        \label{fig:accuracy_cifar10}
    \end{subfigure}
    
    \caption{Test Accuracy across Top and Bottom Influence Score Quantiles under DP and Non-DP} 
    \label{fig:total_results}
\end{figure*}
\textbf{Impact of DP-SGD on Training Dynamics.}
We fix the noise correlation ratio $NCR =1400$ and signal strength $\|\mathbf{u}\|_2=0.5$ and conduct experiments on synthetic data under both non-DP and DP settings. And we evaluate the training dynamics of $\langle \mathbf{w}, \mathbf{u} \rangle$ (signal learning) and $\langle \mathbf{w}, \boldsymbol{\xi} \rangle$ (noise pattern memorization)\footnote{We denote $\langle \mathbf{w}, \mathbf{u} \rangle$ as $1/N\max_{r\in[m]}\sum_{i=1}^N\langle \mathbf{w}_{i,r}, \mathbf{u} \rangle$ and $\langle \mathbf{w}, \boldsymbol{\xi} \rangle$ as $1/N\max_{r\in[m]}\sum_{i=1}^N\langle \mathbf{w}_{i,r}, \boldsymbol{\xi}_i \rangle$} across DP and non-DP regimes. As shown in Figure~\ref{fig:training_dynamic}, DP-SGD markedly suppresses noise memorization, which is consistent with our theoretical analysis in Theorem~\ref{thm:mem}. Notably, while signal learning and noise memorization exhibit comparable magnitudes in the non-DP setting, the introduction of DP causes $\langle \mathbf{w}, \boldsymbol{\xi} \rangle$ significantly lower---and even negative---relative to $\langle \mathbf{w}, \mathbf{u} \rangle$. This disparity confirms that DP disproportionately affects the memorization mechanism, thereby accounting for the utility loss observed in long-tailed data distributions where memorization is indispensable.

\begin{figure}[h] 
    \centering 
    \includegraphics[width=0.42\textwidth]{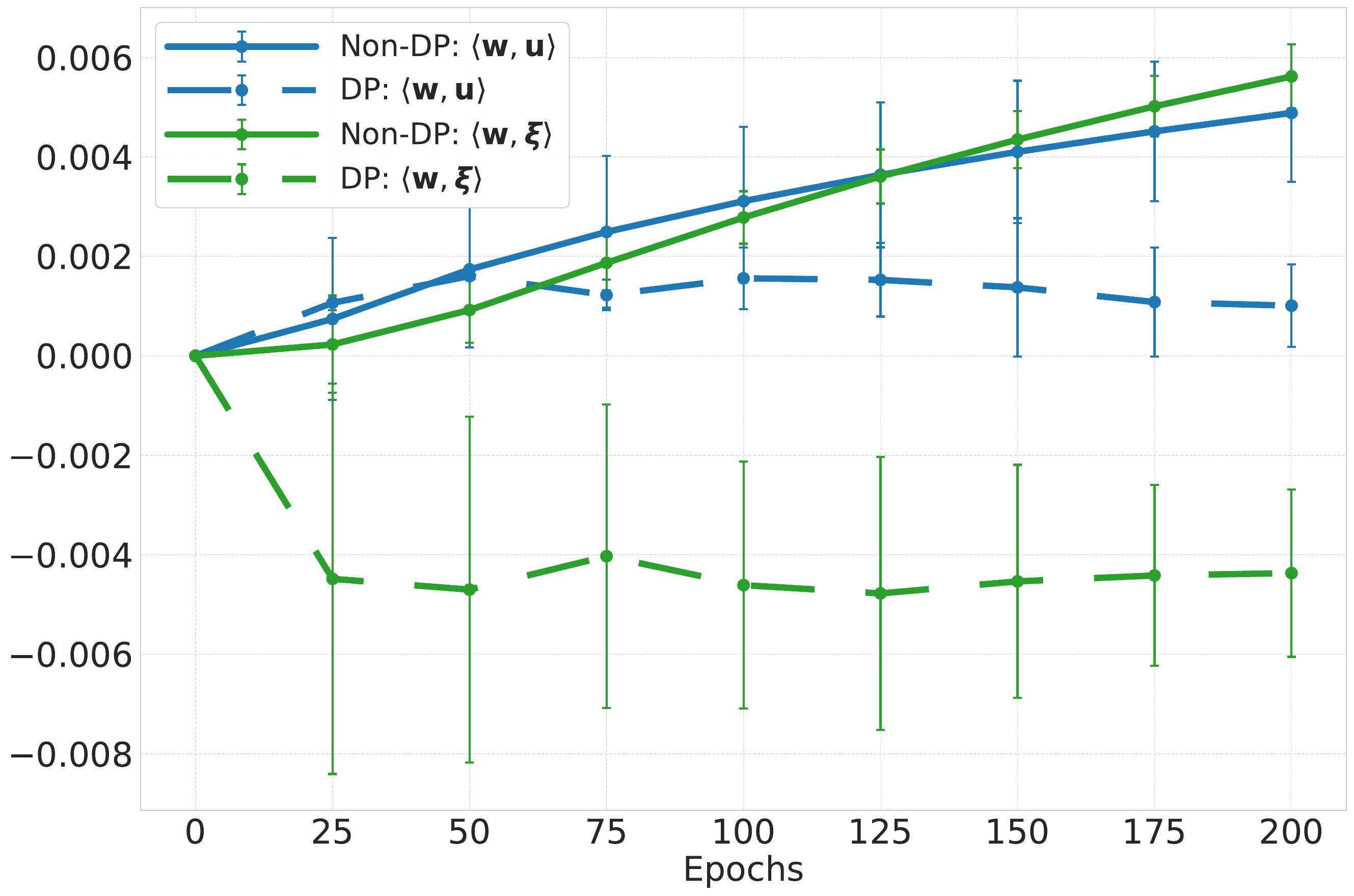} 
    \caption{Training Dynamics Under DP and Non-DP}
    \label{fig:training_dynamic}
\end{figure}
\textbf{Impact of DP-SGD on test accuracy(synthetic data).}
We vary the feature strength $\|\mathbf{u}\|_2$ from $0$ to $3.8$ and the noise correlation ratio $NCR$ from $0$ to $1400$, conducting a comparative study between DP and non-DP settings on synthetic data. Figure~\ref{fig:heatmap} presents the test accuracy heatmaps under different noise distributions (Gaussian and Uniform) and privacy regimes. 

Our empirical results yield several key observations. First, the test accuracy increases as both the feature strength and the noise correlation ratio grow, which aligns with the test error upper bounds established in Theorem~\ref{thm:testerror}. Second, it is noteworthy that even in the regime of minimal feature strength (e.g., $\|\mathbf{u}\|_2 \approx 0$), a high noise correlation ratio still results in a relatively low test error; this suggests that noise pattern memorization is the key to achieving high test accuracy in the absence of strong signals. 

Finally, a comparison between the DP and non-DP cases reveals that while models trained with DP can still achieve nearly optimal test accuracy when the feature strength is sufficiently large, their performance degrades significantly in the low-feature-strength regime, regardless of the noise correlation ratio. This indicates that DP primarily disrupts the memorization of noise patterns, which is indispensable for maintaining utility on rare, long-tailed samples that lack prominent features, corresponding to Remark~\ref{rm:11}.

\textbf{Impact of DP-SGD on test accuracy(real-world data).}
In real-world datasets, parameters such as signal strength and noise correlation ratios are inherent and immutable. Furthermore, as the data-generating distributions of real-world datasets do not strictly align with the data structure in Definition~\ref{def:data_model}, it is unfeasible to directly identify long-tailed samples using the criteria in Definition~\ref{def:longtail}. To address this, we adopt a surrogate approach by identifying long-tailed samples as those that significantly enhance the heterogeneity of noise patterns, specifically through their impact on the squared Frobenius norm $\|\mathbf{A}_y^{\top}\mathbf{A}_y\|_F^2$. 

Inspired by \citet{xu2025rethinking}, and observing that this norm is equivalent to the Frobenius norm of the covariance matrix for class $y$, we formalize the influence score of an image $(x, y) \in \mathcal{S}_{test,y}$ as follows:
\begin{equation}
    \mathcal{I}(x, y) = \|\hat{\Sigma}(\mathcal{S}_{test,y})\|_F^2 - \|\hat{\Sigma}(\mathcal{S}_{test,y} \setminus \{(x, y)\})\|_F^2,
\end{equation}
where $\hat{\Sigma}(\mathcal{S}_{test,y})$ denotes the estimated covariance of the underlying data distribution within the subset $\mathcal{S}_{test,y}$. Samples with higher influence scores are thus treated as long-tailed data, as they represent the primary sources of distributional variance and noise heterogeneity.

Figure~\ref{fig:total_results} illustrates our experimental results on MNIST and CIFAR-10 under both non-DP and DP settings, where we evaluate performance on subsets partitioned by the top and bottom $X\%$ influence scores. In the non-DP regime, the test accuracy remains nearly uniform and optimal across all subsets, suggesting that the noise patterns of long-tailed samples are effectively memorized to maintain high utility. Conversely, the introduction of DP leads to a marked divergence in accuracy, where the performance on the top influence score subset (representing long-tailed data) is significantly lower than that on the bottom subset; this disparity demonstrates that DP disproportionately impairs the utility of long-tailed samples that rely on noise pattern memorization, thereby empirically validating our theoretical results in Theorem~\ref{thm:testerror} and Remark~\ref{rm:11}.

\vspace{-0.1in}
\section{Conclusion}
This paper presents the first analysis of the training dynamics and generalization of two-layer neural networks trained via DP-SGD on long-tailed data. Our theoretical results demonstrate that DP suppresses the memorization of noise patterns, leading to sub-optimal training loss and a disproportionately more severe impact on the test error of long-tailed samples compared to the overall population. Empirical evaluations on both synthetic and real-world datasets validate our theoretical findings, exhibiting a high degree of consistency between the observed training dynamics, test accuracies, and our theoretical predictions.

\section*{Impact Statement}
This paper presents work whose goal is to advance the field of Machine
Learning. There are many potential societal consequences of our work, none
which we feel must be specifically highlighted here.

\nocite{langley00}

\bibliography{example_paper}

@article{wang2020empirical,
  title={Empirical risk minimization in the non-interactive local model of differential privacy},
  author={Wang, Di and Gaboardi, Marco and Smith, Adam and Xu, Jinhui},
  journal={Journal of machine learning research},
  volume={21},
  number={200},
  pages={1--39},
  year={2020}
}

@article{ding2024revisiting,
  title={Revisiting differentially private relu regression},
  author={Ding, Meng and Lei, Mingxi and Zhu, Liyang and Wang, Shaowei and Wang, Di and Xu, Jinhui},
  journal={Advances in Neural Information Processing Systems},
  volume={37},
  pages={55470--55506},
  year={2024}
}

@article{ding2025nearly,
  title={Nearly optimal differentially private relu regression},
  author={Ding, Meng and Lei, Mingxi and Wang, Shaowei and Zheng, Tianhang and Wang, Di and Xu, Jinhui},
  journal={arXiv preprint arXiv:2503.06009},
  year={2025}
}

@inproceedings{su2023differentially,
  title={Differentially private stochastic convex optimization in (non)-Euclidean space revisited},
  author={Su, Jinyan and Zhao, Changhong and Wang, Di},
  booktitle={Uncertainty in Artificial Intelligence},
  pages={2026--2035},
  year={2023},
  organization={PMLR}
}

@inproceedings{tao2022private,
  title={Private Stochastic Convex Optimization and Sparse Learning with Heavy-tailed Data Revisited},
  author={Tao, Youming and Wu, Yulian and Cheng, Xiuzhen and Wang, Di},
  booktitle={31st International Joint Conference on Artificial Intelligence, IJCAI 2022},
  pages={3947--3953},
  year={2022},
  organization={International Joint Conferences on Artificial Intelligence Organization}
}

@article{shen2023differentially,
  title={Differentially private non-convex learning for multi-layer neural networks},
  author={Shen, Hanpu and Wang, Cheng-Long and Xiang, Zihang and Ying, Yiming and Wang, Di},
  journal={arXiv preprint arXiv:2310.08425},
  year={2023}
}

@inproceedings{xiao2023theory,
  title={A theory to instruct differentially-private learning via clipping bias reduction},
  author={Xiao, Hanshen and Xiang, Zihang and Wang, Di and Devadas, Srinivas},
  booktitle={2023 IEEE Symposium on Security and Privacy (SP)},
  pages={2170--2189},
  year={2023},
  organization={IEEE}
}

@article{xu2025beyond,
  title={Beyond Ordinary Lipschitz Constraints: Differentially Private Stochastic Optimization with Tsybakov Noise Condition},
  author={Xu, Difei and Ding, Meng and Xiang, Zihang and Xu, Jinhui and Wang, Di},
  journal={arXiv preprint arXiv:2509.04668},
  year={2025}
}

@inproceedings{wang2019noninteractive,
  title={Noninteractive locally private learning of linear models via polynomial approximations},
  author={Wang, Di and Smith, Adam and Xu, Jinhui},
  booktitle={Algorithmic Learning Theory},
  pages={898--903},
  year={2019},
  organization={PMLR}
}

@inproceedings{huai2020pairwise,
  title={Pairwise learning with differential privacy guarantees},
  author={Huai, Mengdi and Wang, Di and Miao, Chenglin and Xu, Jinhui and Zhang, Aidong},
  booktitle={Proceedings of the AAAI Conference on Artificial Intelligence},
  volume={34},
  number={01},
  pages={694--701},
  year={2020}
}

@article{xiang2023practical,
  title={Practical differentially private and byzantine-resilient federated learning},
  author={Xiang, Zihang and Wang, Tianhao and Lin, Wanyu and Wang, Di},
  journal={Proceedings of the ACM on Management of Data},
  volume={1},
  number={2},
  pages={1--26},
  year={2023},
  publisher={ACM New York, NY, USA}
}

@inproceedings{wangprivate,
  title={Private Training Large-scale Models with Efficient DP-SGD},
  author={Wang, Liangyu and Wang, Junxiao and Ren, Jie and Xiang, Zihang and Keyes, David E and Wang, Di},
  booktitle={The Thirty-ninth Annual Conference on Neural Information Processing Systems},
  year={2025} 
}

@inproceedings{xue2021differentially,
  title={Differentially Private Pairwise Learning Revisited},
  author={Xue, Zhiyu and Yang, Shaoyang and Huai, Mengdi and Wang, Di},
  booktitle={30th International Joint Conference on Artificial Intelligence, IJCAI 2021},
  pages={3242--3248},
  year={2021},
  organization={International Joint Conferences on Artificial Intelligence Organization}
}

@article{zhang2025towards,
  title={Towards user-level private reinforcement learning with human feedback},
  author={Zhang, Jiaming and Lei, Mingxi and Ding, Meng and Li, Mengdi and Xiang, Zihang and Xu, Difei and Xu, Jinhui and Wang, Di},
  journal={arXiv preprint arXiv:2502.17515},
  year={2025}
}

@article{zhou2024ppml,
  title={PPML-Omics: A privacy-preserving federated machine learning method protects patients’ privacy in omic data},
  author={Zhou, Juexiao and Chen, Siyuan and Wu, Yulian and Li, Haoyang and Zhang, Bin and Zhou, Longxi and Hu, Yan and Xiang, Zihang and Li, Zhongxiao and Chen, Ningning and others},
  journal={Science Advances},
  volume={10},
  number={5},
  pages={eadh8601},
  year={2024},
  publisher={American Association for the Advancement of Science}
}

@article{zhou2024personalized,
  title={Personalized and privacy-preserving federated heterogeneous medical image analysis with PPPML-HMI},
  author={Zhou, Juexiao and Zhou, Longxi and Wang, Di and Xu, Xiaopeng and Li, Haoyang and Chu, Yuetan and Han, Wenkai and Gao, Xin},
  journal={Computers in Biology and Medicine},
  volume={169},
  pages={107861},
  year={2024},
  publisher={Elsevier}
}

@article{xuunderstanding,
  title={Understanding Impacts of Differential Privacy: A Unified Framework with Two-Layer Neural Networks},
  author={Xu, Ruichen and Chen, Kexin}
}

@article{ding2025understanding,
  title={Understanding Private Learning From Feature Perspective},
  author={Ding, Meng and Lei, Mingxi and Fu, Shaopeng and Wang, Shaowei and Wang, Di and Xu, Jinhui},
  journal={arXiv preprint arXiv:2511.18006},
  year={2025}
}

@article{xu2025rethinking,
  title={Rethinking benign overfitting in two-layer neural networks},
  author={Xu, Ruichen and Chen, Kexin},
  journal={arXiv preprint arXiv:2502.11893},
  year={2025}
}

@article{carlini2019distribution,
  title={Distribution density, tails, and outliers in machine learning: Metrics and applications},
  author={Carlini, Nicholas and Erlingsson, Ulfar and Papernot, Nicolas},
  journal={arXiv preprint arXiv:1910.13427},
  year={2019}
}

@inproceedings{feldman2020does,
  title={Does learning require memorization? a short tale about a long tail},
  author={Feldman, Vitaly},
  booktitle={Proceedings of the 52nd annual ACM SIGACT symposium on theory of computing},
  pages={954--959},
  year={2020}
}

@article{bagdasaryan2019differential,
  title={Differential privacy has disparate impact on model accuracy},
  author={Bagdasaryan, Eugene and Poursaeed, Omid and Shmatikov, Vitaly},
  journal={Advances in neural information processing systems},
  volume={32},
  year={2019}
}

@article{feldman2020neural,
  title={What neural networks memorize and why: Discovering the long tail via influence estimation},
  author={Feldman, Vitaly and Zhang, Chiyuan},
  journal={Advances in Neural Information Processing Systems},
  volume={33},
  pages={2881--2891},
  year={2020}
}

@inproceedings{pujol2020fair,
  title={Fair decision making using privacy-protected data},
  author={Pujol, David and McKenna, Ryan and Kuppam, Satya and Hay, Michael and Machanavajjhala, Ashwin and Miklau, Gerome},
  booktitle={Proceedings of the 2020 Conference on Fairness, Accountability, and Transparency},
  pages={189--199},
  year={2020}
}

@inproceedings{cummings2019compatibility,
  title={On the compatibility of privacy and fairness},
  author={Cummings, Rachel and Gupta, Varun and Kimpara, Dhamma and Morgenstern, Jamie},
  booktitle={Adjunct publication of the 27th conference on user modeling, adaptation and personalization},
  pages={309--315},
  year={2019}
}

@inproceedings{kou2023benign,
  title={Benign overfitting in two-layer relu convolutional neural networks},
  author={Kou, Yiwen and Chen, Zixiang and Chen, Yuanzhou and Gu, Quanquan},
  booktitle={International conference on machine learning},
  pages={17615--17659},
  year={2023},
  organization={PMLR}
}

@article{cao2022benign,
  title={Benign overfitting in two-layer convolutional neural networks},
  author={Cao, Yuan and Chen, Zixiang and Belkin, Misha and Gu, Quanquan},
  journal={Advances in neural information processing systems},
  volume={35},
  pages={25237--25250},
  year={2022}
}

@inproceedings{dwork2006calibrating,
  title={Calibrating noise to sensitivity in private data analysis},
  author={Dwork, Cynthia and McSherry, Frank and Nissim, Kobbi and Smith, Adam},
  booktitle={Theory of cryptography conference},
  pages={265--284},
  year={2006},
  organization={Springer}
}

@inproceedings{abadi2016deep,
  title={Deep learning with differential privacy},
  author={Abadi, Martin and Chu, Andy and Goodfellow, Ian and McMahan, H Brendan and Mironov, Ilya and Talwar, Kunal and Zhang, Li},
  booktitle={Proceedings of the 2016 ACM SIGSAC conference on computer and communications security},
  pages={308--318},
  year={2016}
}

@article{chaudhuri2011differentially,
  title={Differentially private empirical risk minimization.},
  author={Chaudhuri, Kamalika and Monteleoni, Claire and Sarwate, Anand D},
  journal={Journal of Machine Learning Research},
  volume={12},
  number={3},
  year={2011}
}

@inproceedings{bassily2014private,
  title={Private empirical risk minimization: Efficient algorithms and tight error bounds},
  author={Bassily, Raef and Smith, Adam and Thakurta, Abhradeep},
  booktitle={2014 IEEE 55th annual symposium on foundations of computer science},
  pages={464--473},
  year={2014},
  organization={IEEE}
}

@inproceedings{wang2019differentially,
  title={Differentially private empirical risk minimization with smooth non-convex loss functions: A non-stationary view},
  author={Wang, Di and Xu, Jinhui},
  booktitle={Proceedings of the AAAI Conference on Artificial Intelligence},
  volume={33},
  number={01},
  pages={1182--1189},
  year={2019}
}

@article{bassily2019private,
  title={Private stochastic convex optimization with optimal rates},
  author={Bassily, Raef and Feldman, Vitaly and Talwar, Kunal and Guha Thakurta, Abhradeep},
  journal={Advances in neural information processing systems},
  volume={32},
  year={2019}
}

@inproceedings{feldman2020private,
  title={Private stochastic convex optimization: optimal rates in linear time},
  author={Feldman, Vitaly and Koren, Tomer and Talwar, Kunal},
  booktitle={Proceedings of the 52nd Annual ACM SIGACT Symposium on Theory of Computing},
  pages={439--449},
  year={2020}
}

@inproceedings{wang2020differentially,
  title={On differentially private stochastic convex optimization with heavy-tailed data},
  author={Wang, Di and Xiao, Hanshen and Devadas, Srinivas and Xu, Jinhui},
  booktitle={International Conference on Machine Learning},
  pages={10081--10091},
  year={2020},
  organization={PMLR}
}

@inproceedings{hu2022high,
  title={High dimensional differentially private stochastic optimization with heavy-tailed data},
  author={Hu, Lijie and Ni, Shuo and Xiao, Hanshen and Wang, Di},
  booktitle={Proceedings of the 41st ACM SIGMOD-SIGACT-SIGAI Symposium on Principles of Database Systems},
  pages={227--236},
  year={2022}
}

@article{bassily2021differentially,
  title={Differentially private stochastic optimization: New results in convex and non-convex settings},
  author={Bassily, Raef and Guzm{\'a}n, Crist{\'o}bal and Menart, Michael},
  journal={Advances in Neural Information Processing Systems},
  volume={34},
  pages={9317--9329},
  year={2021}
}

@inproceedings{asi2021private,
  title={Private adaptive gradient methods for convex optimization},
  author={Asi, Hilal and Duchi, John and Fallah, Alireza and Javidbakht, Omid and Talwar, Kunal},
  booktitle={International Conference on Machine Learning},
  pages={383--392},
  year={2021},
  organization={PMLR}
}

@inproceedings{zhang2019learning,
  title={Learning one-hidden-layer relu networks via gradient descent},
  author={Zhang, Xiao and Yu, Yaodong and Wang, Lingxiao and Gu, Quanquan},
  booktitle={The 22nd international conference on artificial intelligence and statistics},
  pages={1524--1534},
  year={2019},
  organization={PMLR}
}

@article{bartlett2020benign,
  title={Benign overfitting in linear regression},
  author={Bartlett, Peter L and Long, Philip M and Lugosi, G{\'a}bor and Tsigler, Alexander},
  journal={Proceedings of the National Academy of Sciences},
  volume={117},
  number={48},
  pages={30063--30070},
  year={2020},
  publisher={National Academy of Sciences}
}

@misc{ruppert2004elements,
  title={The elements of statistical learning: data mining, inference, and prediction},
  author={Ruppert, David},
  year={2004},
  publisher={Taylor \& Francis}
}

@article{lundervold2019overview,
  title={An overview of deep learning in medical imaging focusing on MRI},
  author={Lundervold, Alexander Selvikv{\aa}g and Lundervold, Arvid},
  journal={Zeitschrift fuer medizinische Physik},
  volume={29},
  number={2},
  pages={102--127},
  year={2019},
  publisher={Elsevier}
}

@article{chlap2021review,
  title={A review of medical image data augmentation techniques for deep learning applications},
  author={Chlap, Phillip and Min, Hang and Vandenberg, Nym and Dowling, Jason and Holloway, Lois and Haworth, Annette},
  journal={Journal of medical imaging and radiation oncology},
  volume={65},
  number={5},
  pages={545--563},
  year={2021},
  publisher={Wiley Online Library}
}

@article{bi2024advanced,
  title={Advanced portfolio management in finance using deep learning and artificial intelligence techniques: Enhancing investment strategies through machine learning models},
  author={Bi, Shuochen and Lian, Yufan},
  journal={Journal of Artificial Intelligence Research},
  volume={4},
  number={1},
  pages={233--298},
  year={2024}
}

@article{oroojlooy2023review,
  title={A review of cooperative multi-agent deep reinforcement learning},
  author={Oroojlooy, Afshin and Hajinezhad, Davood},
  journal={Applied Intelligence},
  volume={53},
  number={11},
  pages={13677--13722},
  year={2023},
  publisher={Springer}
}

@inproceedings{shokri2017membership,
  title={Membership inference attacks against machine learning models},
  author={Shokri, Reza and Stronati, Marco and Song, Congzheng and Shmatikov, Vitaly},
  booktitle={2017 IEEE symposium on security and privacy (SP)},
  pages={3--18},
  year={2017},
  organization={IEEE}
}

@inproceedings{yeom2018privacy,
  title={Privacy risk in machine learning: Analyzing the connection to overfitting},
  author={Yeom, Samuel and Giacomelli, Irene and Fredrikson, Matt and Jha, Somesh},
  booktitle={2018 IEEE 31st computer security foundations symposium (CSF)},
  pages={268--282},
  year={2018},
  organization={IEEE}
}

@inproceedings{carlini2019secret,
  title={The secret sharer: Evaluating and testing unintended memorization in neural networks},
  author={Carlini, Nicholas and Liu, Chang and Erlingsson, {\'U}lfar and Kos, Jernej and Song, Dawn},
  booktitle={28th USENIX security symposium (USENIX security 19)},
  pages={267--284},
  year={2019}
}

@article{lecun2002gradient,
  title={Gradient-based learning applied to document recognition},
  author={LeCun, Yann and Bottou, L{\'e}on and Bengio, Yoshua and Haffner, Patrick},
  journal={Proceedings of the IEEE},
  volume={86},
  number={11},
  pages={2278--2324},
  year={2002},
  publisher={Ieee}
}

@inproceedings{dwork2010boosting,
  title={Boosting and differential privacy},
  author={Dwork, Cynthia and Rothblum, Guy N and Vadhan, Salil},
  booktitle={2010 IEEE 51st annual symposium on foundations of computer science},
  pages={51--60},
  year={2010},
  organization={IEEE}
}

@article{allen2020towards,
  title={Towards understanding ensemble, knowledge distillation and self-distillation in deep learning},
  author={Allen-Zhu, Zeyuan and Li, Yuanzhi},
  journal={arXiv preprint arXiv:2012.09816},
  year={2020}
}

@inproceedings{allen2022feature,
  title={Feature purification: How adversarial training performs robust deep learning},
  author={Allen-Zhu, Zeyuan and Li, Yuanzhi},
  booktitle={2021 IEEE 62nd annual symposium on foundations of computer science (FOCS)},
  pages={977--988},
  year={2022},
  organization={IEEE}
}

@book{vershynin2018high,
  title={High-dimensional probability: An introduction with applications in data science},
  author={Vershynin, Roman},
  volume={47},
  year={2018},
  publisher={Cambridge university press}
}

@misc{remerscheid2022smoothnetsoptimizingcnnarchitecture,
      title={SmoothNets: Optimizing CNN architecture design for differentially private deep learning}, 
      author={Nicolas W. Remerscheid and Alexander Ziller and Daniel Rueckert and Georgios Kaissis},
      year={2022},
      eprint={2205.04095},
      archivePrefix={arXiv},
      primaryClass={cs.CV},
      url={https://arxiv.org/abs/2205.04095}, 
}

@article{shi2025towards,
  title={Towards Understanding Generalization in DP-GD: A Case Study in Training Two-Layer CNNs},
  author={Shi, Zhongjie and Wang, Puyu and Zhang, Chenyang and Cao, Yuan},
  journal={arXiv preprint arXiv:2511.22270},
  year={2025}
}
\bibliographystyle{icml2026}

\newpage
\appendix
\onecolumn

\section{Illustration of long-tailed data}\label{sec:illu}

\begin{figure}[h] 
    \centering 
    \includegraphics[width=0.7\textwidth]{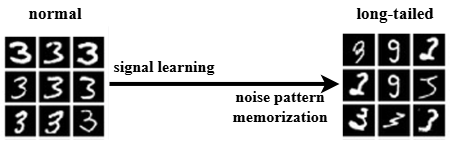} 
    \caption{Illustration of long-tailed data on MNIST}
    \label{fig:long-tail}
\end{figure}
Figure~\ref{fig:long-tail} illustrates both normal and long-tailed data from the MNIST dataset (digit 3). The samples on the left represent normal data that rely on signal learning. Conversely, the samples on the right represent atypical long-tailed data that rely on noise pattern memorization.

\begin{remark}[Justification for Definition~\ref{def:longtail}]
    We think that the essence of long-tailed data is atypicality rather than mere frequency. Tail samples possess unique noise patterns $\boldsymbol{\xi}$ that do not align with the global feature $\mathbf{u}$. Definition~\ref{def:longtail} identifies this ``intrinsic tail'' by selecting samples residing in the extreme tail of the noise distribution (i.e., $\zeta' \geq L$). Given the sub-Gaussian nature of the noise, the probability of such samples occurring is exponentially small, which formally characterizes their rare property. Crucially, these samples rely fundamentally on the memorization channel to be recognized, explaining their high sensitivity to privacy-preserving training.
\end{remark}
\section{Key lemmas}\label{sec:lem}
\begin{lemma}\label{lem:b1}
Let $\mathbf{Z} \in \mathbb{R}^{m \times n} (m > n)$ be a matrix whose entries are independent and identically distributed Gaussian variables, i.e., $\mathbf{Z}_{i,j} \sim \mathcal{N}(0, 1)$ for all $i \in [m], j \in [n]$. With probability at least $1 - \delta$, all singular values of $\mathbf{Z}$, $\lambda_i(\mathbf{Z})$, for all $i \in [n]$ satisfies
\begin{equation}
\lambda_i(\mathbf{Z}) \geq \sqrt{m} - 2 \sqrt{8 \log(\frac{2}{\delta}) + 8 \log(9)n}.
\end{equation}
\end{lemma}

\noindent Lemma~\ref{lem:b1} follows from the concentration inequality of Gaussian random matrices \citep{vershynin2018high}.

\begin{lemma}\label{lem:b2}
For $n$ random $\sigma_p$ sub-Gaussian variable $x_1, \dots, x_n$, with probability $1 - \delta$, we have
\begin{equation}
\mathbb{P}[|x| \geq t] \leq 2 \exp \left( -\frac{t^2}{2\sigma_p^2} \right).
\end{equation}
\end{lemma}

\begin{proof}
Based on the definition of sub-Gaussian distribution, with probability of $1 - \delta/n$, we have
\begin{equation}
|x_i| \geq \sqrt{2\sigma_p^2 \log(\frac{2n}{\delta})}.
\end{equation}
By Union bound, we finishes the proof.
\end{proof}

\begin{lemma}[Half-normal distribution concentration bound]\label{lem:aa3}
    Suppose $x_1, x_2, \dots, x_n \sim \mathcal{N}(0, \sigma_0^2)$. Then,
\begin{equation}
\mathbb{P} \left[ \frac{1}{n} \sum_{i=1}^n |x_i| - \sqrt{\frac{2}{\pi}} \sigma_0 \geq t \sigma_0 \right], \mathbb{P} \left[ \frac{1}{n} \sum_{i=1}^n |x_i| - \sqrt{\frac{2}{\pi}} \sigma_0 \leq -t \sigma_0 \right] \leq \exp \left( -\frac{t^2}{2} \right), \forall t \geq 0. 
\end{equation}
\begin{proof}
    First, half-normal variables $|x_i|, \forall i \in [n]$ are sub-Gaussian as a half-normal variable has a negative tail bounded by $-\sqrt{\frac{2}{\pi}}$ and a Gaussian delay positive tail. Then, by Hoeffding's inequality, we have
\begin{equation}
\mathbb{P} \left[ \frac{1}{n} \sum_{i=1}^n |x_i| - \sqrt{\frac{2}{\pi}} \sigma_0 \geq t \sigma_0 \right], \mathbb{P} \left[ \frac{1}{n} \sum_{i=1}^n |x_i| - \sqrt{\frac{2}{\pi}} \sigma_0 \leq -t \sigma_0 \right] \leq \exp \left( -\frac{t^2}{2} \right), \forall t \geq 0. 
\end{equation}
\end{proof}
\end{lemma}

\begin{lemma}\label{lem:b3}
Suppose two zero-mean random vectors $\mathbf{x}, \mathbf{y} \in \mathbb{R}^d$ are generated as $\mathbf{x} = \mathbf{A}\boldsymbol{\zeta}_1, \mathbf{y} = \mathbf{A}\boldsymbol{\zeta}_2$, where $\boldsymbol{\zeta}_i$'s each coordinate is independent, symmetric, and $\sigma_p$ sub-Gaussian with $\mathbb{E}[\zeta_{i,j}^2] = 1$, for any $i \in [2], j \in [d]$. Then, $\mathbf{x}$ and $\mathbf{y}$ satisfy
\begin{equation}
\mathbb{P}[|\langle \mathbf{x}, \mathbf{y} \rangle| \geq t] \leq 2 \exp \left( -\Omega \left( \min \left\{ \frac{t^2}{\|\mathbf{A}^\top \mathbf{A}\|_F^2 \sigma_p^4}, \frac{t}{\|\mathbf{A}^\top \mathbf{A}\|_{\text{op}} \sigma_p^2} \right\} \right) \right).
\end{equation}
\end{lemma}

\begin{proof}
We have
\begin{equation}
\langle \mathbf{x}, \mathbf{y} \rangle = \langle \mathbf{A}\boldsymbol{\zeta}_1, \mathbf{A}\boldsymbol{\zeta}_2 \rangle = \boldsymbol{\zeta}_1^\top \mathbf{A}^\top \mathbf{A} \boldsymbol{\zeta}_2 = \boldsymbol{\zeta}_1^\top \mathbf{U}^\top \boldsymbol{\Lambda} \mathbf{U} \boldsymbol{\zeta}_2 = \tilde{\boldsymbol{\zeta}}_1^\top \boldsymbol{\Lambda} \tilde{\boldsymbol{\zeta}}_2 = \operatorname{Tr} \left( \boldsymbol{\Lambda} \tilde{\boldsymbol{\zeta}}_2 \tilde{\boldsymbol{\zeta}}_1^\top \right).
\end{equation}
As $\tilde{\boldsymbol{\zeta}}_1$ and $\tilde{\boldsymbol{\zeta}}_2$ are isotropic, by Bernstein's inequality, we have
\begin{equation}
\mathbb{P} \left[ \operatorname{Tr} \left( \boldsymbol{\Lambda} \tilde{\boldsymbol{\zeta}}_2 \tilde{\boldsymbol{\zeta}}_1^\top \right) \geq t \right] \leq 2 \exp \left( -\Omega \left( \min \left\{ \frac{t^2}{\|\mathbf{A}^\top \mathbf{A}\|_F^2 \sigma_p^4}, \frac{t}{\|\mathbf{A}^\top \mathbf{A}\|_{\text{op}} \sigma_p^2} \right\} \right) \right).
\end{equation}
\end{proof}

\begin{lemma}\label{lem:b4}
Suppose a zero-mean random vector $\mathbf{x} \in \mathbb{R}^d$ is generated as $\mathbf{x} = \mathbf{A}\boldsymbol{\zeta}$, where $\boldsymbol{\zeta}$'s each coordinate is independent, symmetric, and $\sigma_p$ sub-Gaussian with $\mathbb{E}[\zeta_j^2] = 1$, for any $j \in [d]$. Then, $\mathbf{x}$ satisfies
\begin{equation}
\mathbb{P} \left[ \|\mathbf{x}\|_2^2 - \operatorname{Tr}(\mathbf{A}^\top \mathbf{A}) \geq t \right] \leq 2 \exp \left( -\Omega \left( \min \left\{ \frac{t^2}{\|\mathbf{A}^\top \mathbf{A}\|_F^2 \sigma_p^4}, \frac{t}{\|\mathbf{A}^\top \mathbf{A}\|_{\text{op}} \sigma_p^2} \right\} \right) \right).
\end{equation}
\end{lemma}

\begin{proof}
We have
\begin{equation}
\|\mathbf{x}\|_2^2 = \langle \mathbf{A}\boldsymbol{\zeta}_1, \mathbf{A}\boldsymbol{\zeta}_1 \rangle = \boldsymbol{\zeta}_1^\top \mathbf{A}^\top \mathbf{A} \boldsymbol{\zeta}_1 = \operatorname{Tr} \left( \mathbf{A}^\top \mathbf{A} \boldsymbol{\zeta}_1 \boldsymbol{\zeta}_1^\top \right).
\end{equation}
Then, expectation of $\|\mathbf{x}\|_2$ satisfies
\begin{equation}
\mathbb{E} \left[ \|\mathbf{x}\|_2^2 \right] = \operatorname{Tr}(\mathbf{A}^\top \mathbf{A}).
\end{equation}
By Bernstein's inequality, we have
\begin{equation}
\mathbb{P} \left[ \|\mathbf{x}\|_2^2 - \operatorname{Tr}(\mathbf{A}^\top \mathbf{A}) \geq t \right] \leq 2 \exp \left( -\Omega \left( \min \left\{ \frac{t^2}{\|\mathbf{A}^\top \mathbf{A}\|_F^2 \sigma_p^4}, \frac{t}{\|\mathbf{A}^\top \mathbf{A}\|_{\text{op}} \sigma_p^2} \right\} \right) \right).
\end{equation}
\end{proof}

\begin{lemma}\label{lem:b5}
Suppose two zero-mean random vectors $\mathbf{x}_1, \mathbf{x}_2 \in \mathbb{R}^d$ are generated as $\mathbf{x}_1 = \mathbf{A}\boldsymbol{\zeta}_1, \mathbf{x}_2 = \mathbf{B}\boldsymbol{\zeta}_2$, where $\boldsymbol{\zeta}_i$'s each coordinate is independent, symmetric, and $\sigma_p$ sub-Gaussian with $\mathbb{E}[\zeta_{i,j}^2] = 1$, for any $i \in [2], j \in [d]$. Then, $\mathbf{x}_1$ and $\mathbf{x}_2$ satisfy
\begin{equation}
\mathbb{P} \left[ |\langle \mathbf{x}_1, \mathbf{x}_2 \rangle| \geq t \right] \leq 2 \exp \left( -\Omega \left( \min \left\{ \frac{t^2}{\|\mathbf{A}^\top \mathbf{B}\|_F^2}, \frac{t}{\|\mathbf{A}^\top \mathbf{B}\|_{\text{op}}} \right\} \right) \right).
\end{equation}
\end{lemma}

\begin{proof}
We have
\begin{equation}
\langle \mathbf{x}, \mathbf{y} \rangle = \langle \mathbf{A}\boldsymbol{\zeta}_1, \mathbf{B}\boldsymbol{\zeta}_2 \rangle = \boldsymbol{\zeta}_1^\top \mathbf{A}^\top \mathbf{B} \boldsymbol{\zeta}_2 = \operatorname{Tr} \left( \mathbf{A}^\top \mathbf{B} \boldsymbol{\zeta}_2 \boldsymbol{\zeta}_1^\top \right).
\end{equation}
Then, by Bernstein's inequality, we have
\begin{equation}
\begin{aligned}
\mathbb{P} \left[ |\langle \mathbf{x}_1, \mathbf{x}_2 \rangle| \geq t \right] &= \mathbb{P} \left[ \operatorname{Tr} \left( \mathbf{A}^\top \mathbf{B} \boldsymbol{\zeta}_2 \boldsymbol{\zeta}_1^\top \right) \geq t \right] \\
&\leq 2 \exp \left( -\Omega \left( \min \left\{ \frac{t^2}{\|\mathbf{A}^\top \mathbf{B}\|_F^2 \sigma_p^4}, \frac{t}{\|\mathbf{A}^\top \mathbf{B}\|_{\text{op}} \sigma_p^2} \right\} \right) \right).
\end{aligned}
\end{equation}
\end{proof}
\begin{lemma}\label{lem:nmu}
Let $\mathbf{n}^{(t)} \sim \mathcal{N}(0, \sigma_n^2 \mathbf{I}_d)$ and $\mathbf{u} \in \mathbb{R}^d$ be a fixed vector. For any $\delta \in (0, 1)$, with probability at least $1 - \delta$, we have:
\begin{equation}
    |\langle \mathbf{n}^{(t)}, \mathbf{u} \rangle| \le \sigma_n \|\mathbf{u}\|_2 \sqrt{2 \log (2/\delta)}.
\end{equation}
\end{lemma}

\begin{proof}
Let $Z = \langle \mathbf{n}^{(t)}, \mathbf{u} \rangle$. Since $\mathbf{n}^{(t)}$ is a Gaussian vector, $Z$ is a zero-mean Gaussian random variable with variance $\mathbb{E}[Z^2] = \mathbf{u}^\top \mathbb{E}[\mathbf{n}^{(t)} (\mathbf{n}^{(t)})^\top] \mathbf{u} = \sigma_n^2 \|\mathbf{u}\|_2^2$. 
Specifically, $Z$ is $\sigma_n \|\mathbf{u}\|_2$-sub-Gaussian. Applying the standard sub-Gaussian tail bound :
\begin{equation}
    \mathbb{P}(|Z| \ge \epsilon) \le 2 \exp \left( - \frac{\epsilon^2}{2 \sigma_n^2 \|\mathbf{u}\|_2^2} \right).
\end{equation}
Setting $\epsilon = \sigma_n \|\mathbf{u}\|_2 \sqrt{2 \log \frac{2}{\delta}}$ completes the proof.
\end{proof}

\begin{lemma}
\label{lem:nxi}
Let $\mathbf{n}^{(t)} \sim \mathcal{N}(0, \sigma_n^2 \mathbf{I}_d)$ and $\boldsymbol{\zeta} \sim \mathcal{N}(0, \sigma_p^2 \mathbf{I}_d)$ be two independent isotropic Gaussian vectors. Let $\boldsymbol{\xi} = \mathbf{A} \boldsymbol{\zeta}$ for a fixed matrix $\mathbf{A} \in \mathbb{R}^{d \times d}$. Then for any $\delta \in (0, 1)$, there exists a universal constant $C > 0$ such that the following bound holds with probability at least $1 - \delta$:
\begin{equation}
    |\langle \mathbf{n}^{(t)}, \boldsymbol{\xi} \rangle| \le C \sigma_n \sigma_p \left( \|\mathbf{A}\|_F \sqrt{\log(2/\delta)} + \|\mathbf{A}\|_2 \log(2/\delta) \right)
\end{equation}
where $\|\mathbf{A}\|_F$ and $\|\mathbf{A}\|_2$ denote the Frobenius norm and the spectral norm of $\mathbf{A}$, respectively.
\end{lemma}

\begin{proof}
Consider the random variable $Z = \langle \mathbf{n}^{(t)}, \mathbf{A} \boldsymbol{\zeta} \rangle$. We can write $Z$ as a quadratic form:
\begin{equation}
    Z = \mathbf{n}^{(t)\top} \mathbf{A} \boldsymbol{\zeta} = \frac{1}{2} \mathbf{z}^\top \mathbf{Q} \mathbf{z}, \quad \text{where } \mathbf{z} = \begin{bmatrix} \mathbf{n}^{(t)} \\ \boldsymbol{\zeta} \end{bmatrix} \in \mathbb{R}^{2d}, \quad \mathbf{Q} = \begin{bmatrix} \mathbf{0} & \mathbf{A} \\ \mathbf{A}^\top & \mathbf{0} \end{bmatrix}
\end{equation}
The vector $\mathbf{z}$ is a concatenated Gaussian vector with covariance $\boldsymbol{\Sigma} = \text{diag}(\sigma_n^2 \mathbf{I}_d, \sigma_p^2 \mathbf{I}_d)$. Let $\tilde{\mathbf{z}} = \boldsymbol{\Sigma}^{-1/2} \mathbf{z} \sim \mathcal{N}(0, \mathbf{I}_{2d})$, then $Z = \tilde{\mathbf{z}}^\top \tilde{\mathbf{Q}} \tilde{\mathbf{z}}$ where:
\begin{equation}
    \tilde{\mathbf{Q}} = \frac{1}{2} \boldsymbol{\Sigma}^{1/2} \mathbf{Q} \boldsymbol{\Sigma}^{1/2} = \frac{1}{2} \begin{bmatrix} \mathbf{0} & \sigma_n \sigma_p \mathbf{A} \\ \sigma_n \sigma_p \mathbf{A}^\top & \mathbf{0} \end{bmatrix}
\end{equation}
The random variable $Z$ is a sum of dependent sub-exponential variables. By the Hanson-Wright inequality, for any $t > 0$:
\begin{equation}
    \mathbb{P}(|Z| > t) \le 2 \exp \left( - c \min \left( \frac{t^2}{\|\tilde{\mathbf{Q}}\|_F^2}, \frac{t}{\|\tilde{\mathbf{Q}}\|_2} \right) \right)
\end{equation}
We compute the norms of $\tilde{\mathbf{Q}}$:
\begin{itemize}
    \item $\|\tilde{\mathbf{Q}}\|_F^2 = \frac{1}{4} ( \sigma_n^2 \sigma_p^2 \|\mathbf{A}\|_F^2 + \sigma_n^2 \sigma_p^2 \|\mathbf{A}^\top\|_F^2 ) = \frac{1}{2} \sigma_n^2 \sigma_p^2 \|\mathbf{A}\|_F^2$.
    \item $\|\tilde{\mathbf{Q}}\|_2 = \frac{1}{2} \sigma_n \sigma_p \|\mathbf{A}\|_2$ .
\end{itemize}
Setting the tail probability to $\delta$, we have $t \lesssim \|\tilde{\mathbf{Q}}\|_F \sqrt{\log(2/\delta)}$  and $t \lesssim \|\tilde{\mathbf{Q}}\|_2 \log(2/\delta)$ . Combining these yields:
\begin{equation}
    |Z| \le C \left( \sigma_n \sigma_p \|\mathbf{A}\|_F \sqrt{\log(2/\delta)} + \sigma_n \sigma_p \|\mathbf{A}\|_2 \log(2/\delta) \right)
\end{equation}
This completes the proof.
\end{proof}
\begin{lemma}\label{lem:b6}
Let $x_1, \dots, x_m$ be $m$ independent zero-mean Gaussian variables. Denote $z_i$ as indicators for signs of $x_i$, i.e., for all $i \in [m]$,
\begin{equation}
z_i = \begin{cases} 1, & x_i > 0, \\ 0, & x_i \leq 0. \end{cases}
\end{equation}
Then, we have
\begin{equation}
\operatorname{Pr} \left[ \sum_{i=1}^m z_i \geq 0.4m \right] \geq 1 - \exp \left( -\frac{8}{25}m \right).
\end{equation}
\end{lemma}

\begin{proof}
Because $z_i, i \in [m]$ are bounded in $[0, 1]$, $z_i, i \in [m]$ are sub-Gaussian variables. By Hoeffding's inequality, we have
\begin{equation}
\operatorname{Pr} \left[ m \left( \frac{1}{m} \sum_{i=1}^m z_i \right) \leq m \left( \frac{1}{2} - \epsilon \right) \right] \leq \exp \left( -\frac{2m^2\epsilon^2}{m(1/16)} \right).
\end{equation}
Let $\epsilon = 0.1$, we have
\begin{equation}
\operatorname{Pr} \left[ \sum_{i=1}^m z_i \leq 0.4m \right] \leq \exp \left( -\frac{8}{25}m \right).
\end{equation}
Therefore, we have
\begin{equation}
\operatorname{Pr} \left[ \sum_{i=1}^m z_i \geq 0.4m \right] \geq 1 - \exp \left( -\frac{8}{25}m \right).
\end{equation}
This completes the proof.
\end{proof}

\begin{lemma}\label{lem:b7}
For any constant $t \in (0, 1]$ and $x \in [-a, b], a, b > 0$, we have
\begin{equation}
\log(1 + t(\exp(x) - 1)) \leq \Gamma(x)x,
\end{equation}
where $\Gamma(x) = \mathbb{I}(x \geq 0) + [\log(1 + t(\exp(-a) - 1)) / -a] \mathbb{I}(x < 0)$.
\end{lemma}

\begin{proof}
First, considering $x \geq 0$, we have
\begin{equation}
\frac{\partial \log(1 + t(\exp(x) - 1))}{\partial t} = \frac{\exp(x) - 1}{1 + t(\exp(x) - 1)} \geq 0.
\end{equation}
Thus, $\log(1 + t(\exp(x) - 1)) \leq x, \forall x > 0$. Second, considering $x < 0$, we have
\begin{equation}
\frac{\partial^2 \log(1 + t(\exp(x) - 1))}{\partial x^2} = \frac{(1 - t)t \exp(x)}{[1 + t(\exp(x) - 1)]^2} \geq 0.
\end{equation}
So $\log(1 + t(\exp(x) - 1))$ is a convex function of $x$. We can conclude that
\begin{equation}
\log(1 + t(\exp(x) - 1)) \leq \frac{\log(1 + t(\exp(-a) - 1))}{-a} x, \forall x < 0.
\end{equation}
This completes the proof.
\end{proof}

\begin{lemma}\label{lem:convex}
With input $\mathbf{x} \in \mathbb{R}^K$, the function $f(\mathbf{x}) = -\log(\exp(x_i) / \sum_{j \in [K]} \exp(x_j))$ with any $i \in [K]$ is convex.
\end{lemma}

\begin{proof}
For any $\mathbf{x}_1, \mathbf{x}_2 \in \mathbb{R}^K$ and $\alpha \in [0, 1]$, we have
\begin{equation}
\begin{aligned}
\alpha f(\mathbf{x}_1) + (1 - \alpha) f(\mathbf{x}_2) &= - \alpha \log \left( \frac{\exp(x_{1,i})}{\sum_{j \in [K]} \exp(x_{1,j})} \right) - (1 - \alpha) \log \left( \frac{\exp(x_{2,i})}{\sum_{j \in [K]} \exp(x_{2,j})} \right) \\
&= - \log \left( \left( \frac{\exp(x_{1,i})}{\sum_{j \in [K]} \exp(x_{1,j})} \right)^\alpha \left( \frac{\exp(x_{2,i})}{\sum_{j \in [K]} \exp(x_{2,j})} \right)^{1-\alpha} \right) \\
&= - \log \left( \frac{\exp(\alpha x_{1,i})}{(\sum_{j \in [K]} \exp(x_{1,j}))^\alpha} \frac{\exp((1 - \alpha)x_{2,i})}{(\sum_{j \in [K]} \exp(x_{2,j}))^{1-\alpha}} \right) \\
&= - \log \left( \frac{\exp(\alpha x_{1,i} + (1 - \alpha)x_{2,i})}{(\sum_{j \in [K]} \exp(x_{1,j}))^\alpha (\sum_{j \in [K]} \exp(x_{2,j}))^{1-\alpha}} \right) \\
&\geq - \log \left( \frac{\exp(\alpha x_{1,i} + (1 - \alpha)x_{2,i})}{\sum_{j \in [K]} \exp(\alpha x_{1,j} + (1 - \alpha)x_{2,j})} \right) \\
&= f(\alpha \mathbf{x}_1 + (1 - \alpha) \mathbf{x}_2).
\end{aligned}
\end{equation}
This finishes the proof.
\end{proof}
\begin{lemma}\label{lem:b9}
A vector $\mathbf{z}$ uniformly sampled from $\mathbb{S}^{d-1}$ satisfies
\begin{equation}
\mathbb{P} \left[ \left| \frac{1}{d} \sum_{i=1}^d \mathbb{I}(z_i) - \frac{1}{2} \right| \geq t \right] \leq \exp \left( -2 d t^2 \right).
\end{equation}
\end{lemma}

\noindent This lemma directly follows from Hoeffding's inequality.

\begin{lemma}
For a constant $0 < t < 1$, a $\sigma_p$ sub-Gaussian variable $x$ with variance $1$ satisfies
\begin{equation}
\mathbb{P}[|x| > t] \geq \Omega((1 - t^2)^2).
\end{equation}
\end{lemma}

\begin{proof}
Since $x$ is a sub-Gaussian variable with variance $1$, we have $\mathbb{E}[x] = 1$. Applying Paley–Zygmund inequality to $x^2$, we have
\begin{equation}
\mathbb{P}[x^2 \geq t] \geq (1 - t)^2 \frac{1}{\mathbb{E}[x^4]}.
\end{equation}
As the fourth moment of sub-Gaussian variable $\mathbb{E}[x^4]$ is bounded by $\mathcal{O}(\sigma_p^4)$ \citep{vershynin2018high}, we have
\begin{equation}
\mathbb{P}[x^2 \geq t] \geq \frac{C_4}{\sigma_p^4}(1 - t)^2,
\end{equation}
where $C_4$ is a constant. This completes the proof.
\end{proof}
\begin{lemma}\label{lem:b11}
Let $\mathbf{A} \in \mathbb{R}^{m \times n}$ be a matrix with rank $M$. Suppose $\delta > 0$ and $m \geq \Omega(\log(n/\delta)(\lambda_{\max}^+(\mathbf{A}))^2 / ((\lambda_{\min}^+(\mathbf{A}))^2))$. With probability $1 - \delta$, for any orthant $\mathcal{T} \in \mathbb{R}^m$ and vector $\mathbf{x} \in \mathbb{R}^n$ is a random sub-Gaussian vector with each coordinate follows $\mathcal{D}_{\xi}$, we have
\begin{equation}
\mathbb{P} [\mathbf{Ax} \in \mathcal{T}] \leq (0.6)^M.
\end{equation}
\end{lemma}

\begin{proof}
Using singular value decomposition for $\mathbf{A}$, for any orthant $\mathcal{T} \in \mathbb{R}^m$, we have
\begin{equation}
\mathbb{P} [\mathbf{Ax} \in \mathcal{T}] = \mathbb{P} [\mathbf{U \Sigma x} \in \mathcal{T}].
\end{equation}
Without loss of generality, we assume the orthant $\mathcal{T}$ is that with all positive entries. Then, we have
\begin{equation}
\mathbb{P} [\mathbf{Ax} \in \mathcal{T}] \leq \mathbb{P} [\mathbf{\tilde{U} \tilde{\Sigma} x} \in \mathcal{T}],
\end{equation}
where
\begin{equation}
\mathbf{\tilde{U}} = \left[ \mathbf{1} \cdot 1/\sqrt{m} \quad \mathbf{\tilde{U}}_2 \quad \dots \quad \mathbf{\tilde{U}}_m \right],
\end{equation}
and
\begin{equation}
\mathbf{\tilde{\Sigma}} = \begin{bmatrix}
\lambda_{\max}^+(\mathbf{A}) & 0 & \dots & 0 \\
0 & \lambda_{\min}^+(\mathbf{A}) & \dots & 0 \\
\vdots & \vdots & & \vdots \\
0 & 0 & \dots & \lambda_{\min}^+(\mathbf{A}) \\
\dots & \dots & \dots & 0
\end{bmatrix},
\end{equation}
Here, $\mathbf{\tilde{\Sigma}}$ is generated by replacing the non-zero singular values other than the largest one with $\lambda_{\min}^+(\mathbf{A})$. $\mathbf{1}$ is a vector with all one entries. In the following, we abbreviate $\lambda_{\max}^+(\mathbf{A})$ and $\lambda_{\min}^+(\mathbf{A})$ as $\lambda_{\max}^+$ and $\lambda_{\min}^+$ for convenience. Then, we have
\begin{equation}
\mathbf{\tilde{U} \tilde{\Sigma} x} = \lambda_{\max}^+ \cdot \frac{1}{\sqrt{m}} \cdot \mathbf{1} \cdot x_1 + \lambda_{\min}^+ \sum_{i=2}^M x_i \mathbf{\tilde{U}}_i.
\end{equation}
Here, as each coordinate of $\mathbf{x}$ is generate from $\mathcal{D}_{\xi}$, by Lemma~\ref{lem:b2}, we have
\begin{equation}
|x_i| \leq \sqrt{2 \sigma_p^2 \log \left( \frac{2n}{\delta} \right)},
\end{equation}
with probability $1 - \delta$. Then, we have
\begin{equation}
\begin{aligned}
\mathbb{P} [\mathbf{\tilde{U} \tilde{\Sigma} x} \in \mathcal{T}] &= \mathbb{P} \left[ (\lambda_{\max}^+ - \lambda_{\min}^+) \cdot \frac{1}{\sqrt{m}} \cdot \mathbf{1} \cdot x_1 + \lambda_{\min}^+ \sum_{i=1}^M x_i \mathbf{\tilde{U}}_i \in \mathcal{T} \right] \\
&\leq \mathbb{P} \left[ \lambda_{\max}^+ \cdot \frac{1}{\sqrt{m}} \cdot \mathbf{1} \cdot \sqrt{2 \sigma_p^2 \log \left( \frac{2n}{\delta} \right)} + \lambda_{\min}^+ \sum_{i=1}^M x_i \mathbf{\tilde{U}}_i \in \mathcal{T} \right] \\
&= \mathbb{P} \left[ \mathbf{1} \cdot \frac{\sqrt{2 \sigma_p^2 \log(2n/\delta)} \lambda_{\max}^+}{\sqrt{m}} + \lambda_{\min}^+ \mathbf{x} \in \mathcal{T} \right].
\end{aligned}
\end{equation}
As the entries of $\mathbf{x}$ are independent, we can bound each entry independently. Then, when $m \geq \Omega(\log(n/\delta)(\lambda_{\max}^+)^2 / ((\lambda_{\min}^+)^2))$, we have
\begin{equation}
\mathbb{P} \left[ \lambda_{\min}^+ x_i + \frac{\sqrt{2 \sigma_p^2 \log(2n/\delta)} \lambda_{\max}^+}{\sqrt{m}} < 0 \right] \geq 0.4,
\end{equation}
by the property of $\mathcal{D}_{\xi}$, resulting in
\begin{equation}
\mathbb{P} [\mathbf{\tilde{U} \tilde{\Sigma} x} \in \mathcal{T}] \leq (0.6)^M.
\end{equation}
This completes the proof.
\end{proof}
\begin{lemma}\label{lem:b12}
Let $\mathbf{A} \in \mathbb{R}^{m \times n}$ be a matrix with rank $M$. Suppose $m \geq \Omega(\log(n/\delta)(\lambda_{\max}^+(\mathbf{A}))^2 / ((\lambda_{\min}^+(\mathbf{A}))^2))$. With probability $1 - \delta$, we have
\begin{equation}
\mathbb{P} \left[ \sum_{i=1}^m \mathbb{I}((\mathbf{Ax})_i \leq 0) \geq 0.9m \right] \leq \exp(-0.07m).
\end{equation}
\end{lemma}

\begin{proof}
By Lemma~\ref{lem:b9}, the number $N_o$ of orthants that have more than $0.9m$ negative entries satisfies
\begin{equation}
N_o \leq 2^m \cdot \exp(-0.32m) = \exp((\log 2 - 0.32)m).
\end{equation}
Then, by Lemma~\ref{lem:b11}, we have
\begin{equation}
\begin{aligned}
\mathbb{P} \left[ \sum_{i=1}^m \mathbb{I}((\mathbf{Ax})_i \leq 0) \geq 0.9m \right] &\leq 0.6^M \cdot \exp((\log 2 - 0.32)m) \\
&= \exp(\log(0.6)M + (\log 2 - 0.32)m) \\
&\leq \exp(-0.51M + 0.38m).
\end{aligned}
\end{equation}
As long as $M \geq 0.9m$, we have
\begin{equation}
\mathbb{P} \left[ \sum_{i=1}^m \mathbb{I}((\mathbf{Ax})_i \leq 0) \geq 0.9m \right] \leq \exp(-0.07m).
\end{equation}
This completes the proof.
\end{proof}

\begin{lemma}\label{lem:b13}
Suppose that $\delta > 0$ and $\operatorname{Tr}(\mathbf{A}_i^\top \mathbf{A}_i) = \Omega \left( \max \left\{ \left( \|\mathbf{A}_i^\top \mathbf{A}_j\|_F^2 \sigma_p^4 \log(6n/\delta) \right)^{1/2}, \|\mathbf{A}_i^\top \mathbf{A}_j\|_{\text{op}} \sigma_p^2 \log(6n/\delta) \right\} \right)$. For all $i, j \in [K]$, with probability $1 - \delta$, we have
\begin{equation}
\frac{1}{2} \operatorname{Tr}(\mathbf{A}_{y_i}^\top \mathbf{A}_{y_i}) \leq \|\boldsymbol{\xi}_i\|_2^2 \leq \frac{3}{2} \operatorname{Tr}(\mathbf{A}_{y_i}^\top \mathbf{A}_{y_i}),
\end{equation}
\begin{equation}
|\langle \boldsymbol{\xi}_i, \boldsymbol{\xi}_j \rangle| \leq \mathcal{O} \left( \max \left\{ \left( \|\mathbf{A}_{y_i}^\top \mathbf{A}_{y_j}\|_F^2 \log\left(\frac{6n^2}{\delta}\right) \right)^{1/2}, \|\mathbf{A}_{y_i}^\top \mathbf{A}_{y_j}\|_{\text{op}} \log\left(\frac{6n^2}{\delta}\right) \right\} \right).
\end{equation}
\end{lemma}

\begin{proof}
By Lemma~\ref{lem:b4}, with probability at least $1 - \delta/3n$, there exists a constant $C_1 > 0$ such that
\[
\|\boldsymbol{\xi}_i\|_2^2 \leq \operatorname{Tr}(\mathbf{A}_{y_i}^\top \mathbf{A}_{y_i}) + \max \left\{ \left( \|\mathbf{A}_{y_i}^\top \mathbf{A}_{y_i}\|_F^2 \sigma_p^4 \frac{1}{C_1} \log\left(\frac{6n}{\delta}\right) \right)^{1/2}, \|\mathbf{A}_{y_i}^\top \mathbf{A}_{y_i}\|_{\text{op}} \sigma_p^2 \frac{1}{C_1} \log\left(\frac{6n}{\delta}\right) \right\},
\]
and
\[
\|\boldsymbol{\xi}_{y_i}\|_2^2 \geq \operatorname{Tr}(\mathbf{A}_{y_i}^\top \mathbf{A}_{y_i}) - \max \left\{ \left( \|\mathbf{A}_{y_i}^\top \mathbf{A}_{y_i}\|_F^2 \sigma_p^4 \frac{1}{C_1} \log\left(\frac{6n}{\delta}\right) \right)^{1/2}, \|\mathbf{A}_{y_i}^\top \mathbf{A}_{y_i}\|_{\text{op}} \sigma_p^2 \frac{1}{C_1} \log\left(\frac{6n}{\delta}\right) \right\}.
\]
In addition, by Lemmas~\ref{lem:b3} and ~\ref{lem:b5}, with probability at least $1 - \delta/3n^2$, there exists a constant $C_2 > 0$ such that
\[
|\langle \boldsymbol{\xi}_{y_i}, \boldsymbol{\xi}_{y_j} \rangle| \leq \max \left\{ \left( \|\mathbf{A}_{y_i}^\top \mathbf{A}_{y_j}\|_F^2 \sigma_p^4 \frac{1}{C_2} \log\left(\frac{6n^2}{\delta}\right) \right)^{1/2}, \|\mathbf{A}_{y_i}^\top \mathbf{A}_{y_j}\|_{\text{op}} \sigma_p^2 \frac{1}{C_2} \log\left(\frac{6n^2}{\delta}\right) \right\}.
\]
Applying $\sigma_p = \Theta(1)$ finishes the proof.
\end{proof}
\begin{lemma}\label{lem:17}
Suppose that $d = \Omega(\log(mn/\delta))$ and $m = \Omega(\log(1/\delta))$. With probability at least $1-\delta$, for all $r \in [m], j \in [2], i \in [n]$,
\begin{equation}\label{eq:54}
\begin{aligned}
|\langle \mathbf{w}_{j,r}^{(0)}, \mathbf{u}_k \rangle| &\leq \sqrt{2 \log \left( \frac{4Km}{\delta} \right)} \|\mathbf{u}_k\|_2 \sigma_0,   
\end{aligned}
\end{equation}
\begin{equation}\label{eq:55}
\begin{aligned}
|\langle \mathbf{w}_{j,r}^{(0)}, \boldsymbol{\xi}_i \rangle| &\leq \mathcal{O} \left( \log \left( \frac{Km}{\delta} \right) \|\mathbf{A}_{y_i}\|_F \sigma_0 \right). 
\end{aligned}
\end{equation}
\end{lemma}

\begin{proof}
We prove the first bound (~\ref{eq:54}) with Hoeffding's inequality. For all $j \in [K], r \in [m], i \in [n]$, with probability $1 - \delta/(2Km)$,
\begin{equation}
|\langle \mathbf{w}_{j,r}^{(0)}, \mathbf{u}_k \rangle| \leq \sqrt{2 \log \left( \frac{4Km}{\delta} \right)} \|\mathbf{u}_k\|_2 \sigma_0. 
\end{equation}
Then, we prove the second bound (~\ref{eq:55}) leveraging Lemma B.5. For all $j \in [K], r \in [m], i \in [n]$, with probability $1 - \delta/(2Km)$,
\begin{equation}
\begin{aligned}
|\langle \mathbf{w}_{j,r}^{(0)}, \boldsymbol{\xi}_i \rangle| &\leq \mathcal{O} \left( \max \left\{ \sqrt{\log \left( \frac{Km}{\delta} \right)} \|\mathbf{A}_{y_i}\|_F, \log \left( \frac{Km}{\delta} \right) \|\mathbf{A}_{y_i}\|_{\text{op}} \right\} \sigma_0 \right) \\
&\leq \mathcal{O} \left( \log \left( \frac{Km}{\delta} \right) \|\mathbf{A}_{y_i}\|_F \sigma_0 \right).
\end{aligned}
\end{equation}
\end{proof}

\begin{lemma}\label{lem:b15}
Suppose $\delta > 0$ and $m \geq \Omega(\log(n/\delta))$. With probability at least $1 - \delta$, we have
\begin{equation}
|\mathcal{S}_i^{(0)}| \geq 0.4m.
\end{equation}
\end{lemma}

\noindent Lemma~\ref{lem:b15} follows from Lemma~\ref{lem:b6} and union bound.

\begin{lemma}\label{lem:b16}
Suppose $\delta > 0$ and $n \geq \Omega(\log(m/\delta))$. For any neuron $\mathbf{w}_{j,r}^{(0)}, j \in [2], r \in [m]$, with probability at least $1 - \delta$, we have
\begin{equation}
\sum_{i=1}^n \mathbb{I} (\langle \mathbf{w}_{j,r}^{(0)}, \boldsymbol{\xi}_i \rangle) \geq 0.4n. 
\end{equation}
\end{lemma}

\noindent Lemma~\ref{lem:b16} follows from Lemma~\ref{lem:b6} and union bound.
\section{Training Loss analysis}\label{sec:trainingloss}

In this section, we analyze the training loss. These results are based on the high probability conclusions in Appendix~\ref{sec:lem}. Due to the Assumption~\ref{ass:non-perfect}, we consider that the training data $(\mathbf{x}, y) \in \mathcal{S}$ satisfies $1 - \operatorname{logit}_y(\mathbf{W}, \mathbf{x}) \geq1-\exp(-s)= \Theta(1)$ during all training process.

\subsection{Network Gradient}

The mini-batch gradient on the neuron $\mathbf{w}_{k,r}$ at iteration $t$ is
\begin{equation}\label{eq:grad}
\begin{aligned}
\nabla_{\mathbf{w}_{k,r}^{(t)}} \mathcal{L}(\mathbf{W}^{(t)}, \mathbf{x}, y) = 
& - \frac{1}{mB} \sum_{(\mathbf{x}, y) \in \mathcal{S}^{(t)}} \left[ \mathbb{I}(y = k) (1 - \operatorname{logit}_k(\mathbf{W}^{(t)}, \mathbf{x})) \sum_{j=1}^2 \sigma' \left( \langle \mathbf{w}_{k,r}^{(t)}, \mathbf{x}^{(j)} \rangle \right) \mathbf{x}^{(j)} \right] \\
& + \frac{1}{mB} \sum_{(\mathbf{x}, y) \in \mathcal{S}^{(t)}} \left[ \mathbb{I}(y \neq k) \operatorname{logit}_k(\mathbf{W}^{(t)}, \mathbf{x}) \sum_{j=1}^2 \sigma' \left( \langle \mathbf{w}_{k,r}^{(t)}, \mathbf{x}^{(j)} \rangle \right) \mathbf{x}^{(j)} \right].
\end{aligned}
\end{equation}

\subsection{Bound of the Clipping Multiplier $h(C, \mathbf{x}, y)$}
\begin{lemma}\label{lem:bound_grad}
{In each iteration $t$, with probability at least $1 - \exp(-\Omega(d))$, for any $(\mathbf{x}, y) \in \mathcal{D}_K$, we have}
\begin{equation}
\left\| \nabla_{\mathbf{W}^{(t)}} \mathcal{L}(\mathbf{W}^{(t)}, \mathbf{x}, y) \right\|_2 \le \mathcal{O} \left( \frac{1}{\sqrt{m}} \cdot \left( \|\mathbf{u}_{k}\|_2 + \sqrt{Tr(A_k^{T}A_k}) \right) \right). 
\end{equation}
Lemma~\ref{lem:bound_grad} follows from Lemma~\ref{lem:b4}.
\end{lemma}

For convenience, we first define the clipping multiplier of data $(\mathbf{x}, y)$ as
\begin{equation}\label{eq:36}
h(C, \mathbf{x}, y) = \frac{1}{\max \left\{ 1, \frac{\| \nabla \mathcal{L}(\mathbf{W}^{(t)}, \mathbf{x}, y) \|_2}{C} \right\}}.
\end{equation}
Then, we compute the gradient of the neural networks and prove a bound for it.

By definition (\ref{eq:36}), we know that
\begin{equation*}
    h(C, \mathbf{x}, y) \le 1.
\end{equation*}
In addition, from Lemma~\ref{lem:bound_grad}, we know that with probability at least $1 - \exp(\Omega(d))$,
\begin{equation*}
    h(C, \mathbf{x}, y) \ge \Omega \left( \frac{C \sqrt{m}}{\|\mathbf{u}_{k}\|_2 + \sqrt{Tr(A_k^{T}A_k)}} \right).
\end{equation*}

\subsection{Training Loss Bounds}
First, we characterize the training loss.

\begin{equation*}
\begin{aligned}
&\mathcal{L}(\mathbf{W}^{(t+1)}, \mathbf{x}, y) - \mathcal{L}(\mathbf{W}^{(t)}, \mathbf{x}, y) \\
&= -\log \left( \operatorname{prob}_y \left( \mathbf{W}^{(t+1)}, \mathbf{x} \right) \right) + \log \left( \operatorname{prob}_y \left( \mathbf{W}^{(t)}, \mathbf{x} \right) \right) \\
&= \log \left( \frac{\operatorname{prob}_y (\mathbf{W}^{(t)}, \mathbf{x})}{\operatorname{prob}_y (\mathbf{W}^{(t+1)}, \mathbf{x})} \right) \\
&= \log \left( \frac{\exp \left( F_y^{(t)}(\mathbf{x}) \right) / \left( \exp \left( F_y^{(t)}(\mathbf{x}) \right) + \sum_{j \neq y} \exp \left( F_j^{(t)}(\mathbf{x}) \right) \right)}{\exp \left( F_y^{(t+1)}(\mathbf{x}) \right) / \left( \exp \left( F_y^{(t+1)}(\mathbf{x}) \right) + \sum_{j \neq y} \exp \left( F_j^{(t+1)}(\mathbf{x}) \right) \right)} \right) \\
&= \log \left( \frac{1 + \sum_{j \neq y} \exp \left( F_j^{(t)}(\mathbf{x}) - F_y^{(t)}(\mathbf{x}) \right) \exp \left( \Delta_j^{(t)}(\mathbf{x}) - \Delta_y^{(t)}(\mathbf{x}) \right)}{1 + \sum_{j \neq y} \exp \left( F_j^{(t)}(\mathbf{x}) - F_y^{(t)}(\mathbf{x}) \right)} \right) \\
&= \log \left( 1 + \frac{\sum_{j \neq y} \exp \left( F_j^{(t)}(\mathbf{x}) - F_y^{(t)}(\mathbf{x}) \right) \left( \exp \left( \Delta_j^{(t)}(\mathbf{x}) - \Delta_y^{(t)}(\mathbf{x}) \right) - 1 \right)}{1 + \sum_{j \neq y} \exp \left( F_j^{(t)}(\mathbf{x}) - F_y^{(t)}(\mathbf{x}) \right)} \right) \\
&\leq \log \left( 1 + \left( 1 - \operatorname{prob}_y \left( \mathbf{W}^{(t)}, \mathbf{x} \right) \right) \max_{j \neq y} \left( \exp \left( \Delta_j^{(t)}(\mathbf{x}) - \Delta_y^{(t)}(\mathbf{x}) \right) - 1 \right) \right),
\end{aligned}
\end{equation*}

\noindent where $\Delta_y^{(t)}(\mathbf{x}) = F_y^{(t+1)}(\mathbf{x}) - F_y^{(t)}(\mathbf{x})$, $\Delta_j^{(t)}(\mathbf{x}) = F_j^{(t+1)}(\mathbf{x}) - F_j^{(t)}(\mathbf{x})$.

\noindent By Lemma~\ref{lem:b7}, we have
\begin{equation}\label{eq:loss_up1}
\mathcal{L} \left( \mathbf{W}^{(t+1)}, \mathbf{x}, y \right) - \mathcal{L} \left( \mathbf{W}^{(t)}, \mathbf{x}, y \right) \leq \Theta(1) \cdot \max_{j \neq y} \left( \Delta_j^{(t)}(\mathbf{x}) - \Delta_y^{(t)}(\mathbf{x}) \right).
\end{equation}

\noindent To measure how training loss changes over iterations, we need to characterize the change of $\Delta_j^{(t)}(\mathbf{x}) - \Delta_y^{(t)}(\mathbf{x})$.

We first rearrange $\Delta_j^{(t)}(\mathbf{x}) - \Delta_y^{(t)}(\mathbf{x})$ as follows.
\begin{equation}\label{eq:delta_de}
\begin{aligned}
    &\Delta_j^{(t)}(\mathbf{x}) - \Delta_y^{(t)}(\mathbf{x}) \\
    = &\frac{1}{m} \sum_{r=1}^m \sum_{i=1}^2 \left[ \sigma \left( \langle \mathbf{w}_{j,r}^{(t+1)}, \mathbf{x}^{(i)} \rangle \right) - \sigma \left( \langle \mathbf{w}_{j,r}^{(t)}, \mathbf{x}^{(j)} \rangle \right) \right] \\
    &- \frac{1}{m} \sum_{r=1}^m \sum_{i=1}^2 \left[ \sigma \left( \langle \mathbf{w}_{y,r}^{(t+1)}, \mathbf{x}^{(i)} \rangle \right) - \sigma \left( \langle \mathbf{w}_{y,r}^{(t)}, \mathbf{x}^{(i)} \rangle \right) \right] \\
    = &\underbrace{\frac{1}{m} \sum_{r=1}^m \left[ \sigma \left( \langle \mathbf{w}_{j,r}^{(t+1)}, \mathbf{u}_y \rangle \right) - \sigma \left( \langle \mathbf{w}_{j,r}^{(t)}, \mathbf{u}_y \rangle \right) \right]}_{A} + \underbrace{\frac{1}{m} \sum_{r=1}^m \left[ \sigma \left( \langle \mathbf{w}_{j,r}^{(t+1)}, \boldsymbol{\xi} \rangle \right) - \sigma \left( \langle \mathbf{w}_{j,r}^{(t)}, \boldsymbol{\xi} \rangle \right) \right]}_{B} \\
    &- \underbrace{\frac{1}{m} \sum_{r=1}^m \left[ \sigma \left( \langle \mathbf{w}_{y,r}^{(t+1)}, \mathbf{u}_y \rangle \right) - \sigma \left( \langle \mathbf{w}_{y,r}^{(t)}, \mathbf{u}_y \rangle \right) \right]}_{C} - \underbrace{\frac{1}{m} \sum_{r=1}^m \left[ \sigma \left( \langle \mathbf{w}_{y,r}^{(t+1)}, \boldsymbol{\xi} \rangle \right) - \sigma \left( \langle \mathbf{w}_{y,r}^{(t)}, \boldsymbol{\xi} \rangle \right) \right]}_{D},
\end{aligned}
\end{equation}
We then bound $A, B, C, D$ 

\noindent \textbf{Bound of $A$}
\[
\begin{aligned}
A &= \frac{1}{m} \sum_{r=1}^m \left[ \sigma \left( \left\langle \mathbf{w}_{j,r}^{(t)} - \frac{\eta}{mB} \sum_{(\mathbf{x}_i, y_i) \in \mathcal{S}_y} \sigma'(\langle \mathbf{w}_{j,r}^{(t)}, \mathbf{u}_y \rangle) h(C,\mathbf{x}_i,y_i)\operatorname{logit}_j(\mathbf{W}^{(t)}, \mathbf{x}_i) \mathbf{u}_y+\eta\cdot \mathbf{n}_{j,r}^{(t)}, \mathbf{u}_y \right\rangle \right) \right] \\
&\quad - \frac{1}{m} \sum_{r=1}^m \left[ \sigma \left( \langle \mathbf{w}_{j,r}^{(t)}, \mathbf{u}_y \rangle \right) \right] \\
&\le \frac{1}{m} \sum_{r=1}^m \left[ \sigma \left( \langle \mathbf{w}_{j,r}^{(t)}+\eta\cdot \mathbf{n}_{j,r}^{(t)}, \mathbf{u}_y \rangle \right) \right]-\frac{1}{m} \sum_{r=1}^m \left[ \sigma \left( \langle \mathbf{w}_{j,r}^{(t)}, \mathbf{u}_y \rangle \right) \right]\\
&\leq \frac{1}{m} \sum_{r=1}^m \left[  \left| \langle \eta\cdot \mathbf{n}_{j,r}^{(t)}, \mathbf{u}_y \rangle \right| \right]\\
&\leq \mathcal{O}(\eta\sigma_n\sqrt{d}\|\mathbf{u}_y\|_2)
\end{aligned}
\]
where the first inequality is obtained by the monotonicity of ReLU activation function; the second inequality is because ReLU function is 1-Lipschitz continuous; the last inequality is due to Lemma~\ref{lem:aa3}.

\noindent \textbf{Bound of $B$}
\[
\begin{aligned}
B &= \frac{1}{m} \sum_{r=1}^m \left[ \sigma \left( \left\langle \mathbf{w}_{j,r}^{(t)} - \frac{\eta}{mB} \sum_{(\mathbf{x}_i, y_i) \in \mathcal{S}^{(t)} \setminus \mathcal{S}_j^{(t)}} \sigma'(\langle \mathbf{w}_{j,r}^{(t)}, \boldsymbol{\xi}_i \rangle) h(C,\mathbf{x}_i,y_i)\operatorname{logit}_j(\mathbf{W}^{(t)}, \mathbf{x}_i) \boldsymbol{\xi}_i \right. \right. \right. \\
&\quad \left. \left. \left. + \frac{\eta}{mB} \sum_{(\mathbf{x}_i, y_i) \in \mathcal{S}_j^{(t)}} \sigma'(\langle \mathbf{w}_{j,r}^{(t)}, \boldsymbol{\xi}_i \rangle) h(C,\mathbf{x}_i,y_i)(1 - \operatorname{logit}_j(\mathbf{W}^{(t)}, \mathbf{x}_i)) \boldsymbol{\xi}_i+\eta\cdot\mathbf{n}_{j,r}^{(t)}, \boldsymbol{\xi} \right\rangle \right) \right] - \frac{1}{m} \sum_{r=1}^m \left[ \sigma \left( \langle \mathbf{w}_{j,r}^{(t)}, \boldsymbol{\xi} \rangle \right) \right] \\
&\le  \frac{1}{m} \sum_{r=1}^m \left[ \left| \left\langle \frac{\eta}{mB} \cdot \sum_{(\mathbf{x}_i, y_i) \in \mathcal{S}^{(t)}} \boldsymbol{\xi}_i, \boldsymbol{\xi} \right\rangle \right| + \left| \left\langle \eta \mathbf{n}_{j, r}^{(t)}, \boldsymbol{\xi} \right\rangle \right| \right] \\
&\leq  \mathcal{O} \left( \frac{\eta}{m\sqrt{B}} \operatorname{Tr}(\mathbf{A}_y^{\top}\mathbf{A}_y) + \eta \sqrt{d} \sigma_n \|\mathbf{A}_y\|_F \right)
\end{aligned}
\]
where the first inequality is because $\sigma'(\cdot) \geq 0, \operatorname{logit}_j\in [0, 1]$ and ReLU function is 1-Lipschitz continuous; the second inequality is because of Cauchy--Schwarz inequality, the property of 1-norm and 2-norm, and Lemma~\ref{lem:aa3} and Lemma~\ref{lem:b13}.

\begin{equation}
\begin{aligned}
C = & \frac{1}{m} \sum_{r=1}^m \sigma \left( \left\langle \mathbf{w}_{y,r}^{(t)} + \frac{\eta}{mB} \cdot \sum_{(\mathbf{x}_i, y_i) \in \mathcal{S}_{y}^{(t)}} \sigma' \left( \langle \mathbf{w}_{y,r}^{(t)}, \mathbf{u}_{y} \rangle \right) \cdot h(C, \mathbf{x}_i, y_i) \cdot \left( 1 - \operatorname{logit}_y \left( \mathbf{W}^{(t)}, \mathbf{x}_i \right) \right) \cdot \mathbf{u}_{y} \right. \right. \\
& \left. \left. + \eta \cdot \mathbf{n}_{y,r}^{(t)}, \mathbf{u}_{y} \right\rangle \right) - \frac{1}{m} \sum_{r=1}^m \sigma \left( \langle \mathbf{w}_{y,r}^{(t)}, \mathbf{u}_{y} \rangle \right)
\end{aligned}
\end{equation}

Based on Lemma~\ref{lem:b6}, we can conclude that with probability at least $1 - \exp(-2m)$, the number of activated neurons at iteration $t$ are at least $\frac{m}{4}$. Then, with probability at least $1 - \exp(-\tilde{\Omega}(d))$, we have

\begin{equation}
\begin{aligned}
C \geq & \frac{1}{m} \sum_{r=1}^m \sigma \left( \left\langle \mathbf{w}_{y,r}^{(t)} + \frac{\eta}{mB} \sum_{(\mathbf{x}_i, y_i) \in \mathcal{S}_{y}^{(t)}} \sigma' \left( \langle \mathbf{w}_{y,r}^{(t)}, \mathbf{u}_{y} \rangle \right) \cdot h(C, \mathbf{x}_i, y_i) \cdot \left( 1 - \operatorname{logit}_y \left( \mathbf{W}^{(t)}, \mathbf{x}_i \right) \right) \right. \right. \\
& \left. \left. \mathbf{u}_{y}, \mathbf{u}_{y} \right\rangle \right) - \frac{1}{m} \sum_{r=1}^m \left| \langle \eta \cdot \mathbf{n}_{y,r}^{(t)}, \mathbf{u}_{y} \rangle \right| - \frac{1}{m} \sum_{r=1}^m \sigma \left( \langle \mathbf{w}_{y,r}^{(t)}, \mathbf{u}_{y} \rangle \right) \\
\geq & \Omega \left( \frac{\eta C}{B \sqrt{m} (\|\mathbf{u}_{y}\|_2 + \sqrt{\operatorname{Tr}(\mathbf{A}_y^{\top}\mathbf{A}_y)})} \right) \sum_{(\mathbf{x}_i, y_i) \in \mathcal{S}_{y}^{(t)}} \left( 1 - \operatorname{logit}_y \left( \mathbf{W}^{(t)}, \mathbf{x}_i \right) \right) \|\mathbf{u}_{y}\|_2^2 - \mathcal{O} \left( \eta \sigma_n \sqrt{d} \|\mathbf{u}_{y}\|_2 \right)
\end{aligned}
\end{equation}

The second inequality is by using the bound of the clipping multiplier and Lemma~\ref{lem:aa3}. 

In the following, we prove the bound of $D$. Similar to the proof of bound of $B$, we have that with probability at least $1 - \exp(-\tilde{\Omega}(d))$, we have

\textbf{Bound of $D$}
\begin{equation}
\begin{aligned}
D &= \frac{1}{m} \sum_{r=1}^{m} \left[ \sigma \left( \left\langle \mathbf{w}_{y,r}^{(t)} - \frac{\eta}{mB} \sum_{(\mathbf{x}_i, y_i) \in S^{(t)} \setminus S_y^{(t)}} \sigma'(\langle \mathbf{w}_{y,r}^{(t)}, \boldsymbol{\xi}_i \rangle) h(C, \mathbf{x}_i, y_i)\text{logit}_y(\mathbf{W}^{(t)}, \mathbf{x}_i) \boldsymbol{\xi}_i \right. \right. \right. \\
&\quad + \left. \left. \left. \frac{\eta}{mB} \sum_{(\mathbf{x}_i, y_i) \in S_y^{(t)}} \sigma'(\langle \mathbf{w}_{y,r}^{(t)}, \boldsymbol{\xi}_i \rangle)h(C, \mathbf{x}_i, y_i) (1 - \text{logit}_y(\mathbf{W}^{(t)}, \mathbf{x}_i)) \boldsymbol{\xi}_i+\eta\cdot\mathbf{n}_{y,r}^{(t)}, \boldsymbol{\xi} \right\rangle \right) \right] \\
&\quad - \frac{1}{m} \sum_{r=1}^{m} \left[ \sigma (\langle \mathbf{w}_{y,r}^{(t)}, \boldsymbol{\xi} \rangle) \right] \\
&\ge \frac{1}{m} \sum_{r=1}^{m} \left[ \sigma \left( \langle \mathbf{w}_{y,r}^{(t)}, \boldsymbol{\xi} \rangle - \frac{\eta}{mB} \sum_{(\mathbf{x}_i, y_i) \in S^{(t)} \setminus S_y^{(t)}} \sigma'(\langle \mathbf{w}_{y,r}^{(t)}, \boldsymbol{\xi}_i \rangle) h(C, \mathbf{x}_i, y_i)\text{logit}_y(\mathbf{W}^{(t)}, \mathbf{x}_i) |\langle \boldsymbol{\xi}_i, \boldsymbol{\xi} \rangle| \right. \right. \\
&\quad + \frac{\eta}{mB} \sigma'(\langle \mathbf{w}_{y,r}^{(t)}, \boldsymbol{\xi} \rangle)h(C, \mathbf{x}, y) (1 - \text{logit}_y(\mathbf{W}^{(t)}, \mathbf{x})) \|\boldsymbol{\xi}\|_2^2 \\
&\quad - \left. \left. \frac{\eta}{mB} \sum_{(\mathbf{x}_i, y_i) \in S_y^{(t)} \setminus (\mathbf{x}, y)} \sigma'(\langle \mathbf{w}_{y,r}^{(t)}, \boldsymbol{\xi}_i \rangle) (1 - h(C, \mathbf{x}_i, y_i)\text{logit}_y(\mathbf{W}^{(t)}, \mathbf{x}_i)) |\langle \boldsymbol{\xi}_i, \boldsymbol{\xi} \rangle| \right) \right] \\
&\quad -\frac{1}{m} \sum_{r=1}^{m} |\langle\eta\cdot\mathbf{n}_{y,r}^{(t)},\boldsymbol{\xi}\rangle|- \frac{1}{m} \sum_{r=1}^{m} \left[ \sigma (\langle \mathbf{w}_{y,r}^{(t)}, \boldsymbol{\xi} \rangle) \right] \\
&\ge - \frac{\eta}{mB} \sum_{(\mathbf{x}_i, y_i) \in S^{(t)} \setminus S_y^{(t)}} h(C, \mathbf{x}_i, y_i)\text{logit}_y(\mathbf{W}^{(t)}, \mathbf{x}_i) |\langle \boldsymbol{\xi}_i, \boldsymbol{\xi} \rangle| + \frac{2\eta}{5mB}h(C, \mathbf{x}, y) (1 - \text{logit}_y(\mathbf{W}^{(t)}, \mathbf{x})) \|\boldsymbol{\xi}\|_2^2 \\
&\quad - \frac{\eta}{mB} \sum_{(\mathbf{x}_i, y_i) \in S_y^{(t)} \setminus (\mathbf{x}, y)}h(C, \mathbf{x}_i, y_i) (1 - \text{logit}_y(\mathbf{W}^{(t)}, \mathbf{x}_i)) |\langle \boldsymbol{\xi}_i, \boldsymbol{\xi} \rangle|-\frac{1}{m} \sum_{r=1}^{m} |\langle\eta\cdot\mathbf{n}_{y,r}^{(t)},\boldsymbol{\xi}\rangle| \\
&\ge \Omega\left( \frac{\eta C}{n\sqrt{m} (\|\mathbf{u}_{y}\|_2 + \sqrt{\operatorname{Tr}(\mathbf{A}_y^{\top}\mathbf{A}_y)})} \right)(1 - \text{logit}_y(\mathbf{W}^{(t)}, \mathbf{x})) \|\boldsymbol{\xi}\|_2^2-\mathcal{O} \left( \eta \sigma_n \sqrt{d} \|\mathbf{A}_{y}\|_F \right),
\end{aligned}
\end{equation}
where the last inequality is by using the bound of the clipping multiplier and by the fact that the incremental term is larger than zero if a neuron is activated and by Lemma~\ref{lem:aa3}.

Substituting bounds of $A, B, C, D$ to (\ref{eq:delta_de}), we obtain the upper bound of $\Delta_{j}^{(t)}(\mathbf{x}) - \Delta_y^{(t)}(\mathbf{x})$. With probability at least $1 - \exp(-\tilde{\Omega}(d))$,

\begin{equation}\label{eq:delta_up}
    \begin{aligned}
        \Delta_{j}^{(t)}(\mathbf{x}) - \Delta_y^{(t)}(\mathbf{x})&\leq
        \mathcal{O}(\eta\sigma_n\sqrt{d}\|\mathbf{u}_y\|_2)+\mathcal{O} \left( \frac{\eta}{m\sqrt{B}} \operatorname{Tr}(\mathbf{A}_y^{\top}\mathbf{A}_y) + \eta \sqrt{d} \sigma_n \|\mathbf{A}_y\|_F \right)\\
        &-\Omega \left( \frac{\eta C}{B \sqrt{m} (\|\mathbf{u}_{y}\|_2 + \sqrt{\operatorname{Tr}(\mathbf{A}_y^{\top}\mathbf{A}_y)})} \right) \sum_{(\mathbf{x}_i, y_i) \in \mathcal{S}_{y}^{(t)}} \left( 1 - \operatorname{logit}_y \left( \mathbf{W}^{(t)}, \mathbf{x}_i \right) \right) \|\mathbf{u}_{y}\|_2^2 + \mathcal{O} \left( \eta \sigma_n \sqrt{d} \|\mathbf{u}_{y}\|_2 \right)\\
        &-\Omega\left( \frac{\eta C}{B\sqrt{m} (\|\mathbf{u}_{y}\|_2 + \sqrt{\operatorname{Tr}(\mathbf{A}_y^{\top}\mathbf{A}_y)})} \right)(1 - \text{logit}_y(\mathbf{W}^{(t)}, \mathbf{x})) \|\boldsymbol{\xi}\|_2^2+\mathcal{O} \left( \eta \sigma_n \sqrt{d} \|\mathbf{A}_{y}\|_F \right)\\
        \leq&-\Omega \left( \frac{\eta\Lambda_y}{B \sqrt{m} } \right) \sum_{(\mathbf{x}_i, y_i) \in \mathcal{S}_{y}^{(t)}} \left( 1 - \operatorname{logit}_y \left( \mathbf{W}^{(t)}, \mathbf{x}_i \right) \right) \|\mathbf{u}_{y}\|_2^2 + \mathcal{O} \left( \eta \sigma_n \sqrt{d} (\|\mathbf{u}_{y}\|_2+\|\mathbf{A}_y\|_F \right)
    \end{aligned}
\end{equation}

Combining (\ref{eq:loss_up1}) and (\ref{eq:delta_up}), for any $(\mathbf{x}, y) \in \mathcal{S}$ we have
\begin{equation}\label{eq:lossupdate}
\begin{aligned}
\mathcal{L}(\mathbf{W}^{(t+1)}, \mathbf{x}, y) &\le \mathcal{L}(\mathbf{W}^{(t)}, \mathbf{x}, y) - \frac{\eta\Lambda_y}{\sqrt{m}B} \sum_{(\mathbf{x}_i, y_i) \in \mathcal{S}_y^{(t)}} \left( 1 - \operatorname{logit}_y (\mathbf{W}^{(t)}, \mathbf{x}_i) \right) \|\mathbf{u}_y\|_2^2 + \mathcal{O} \left( \eta \sigma_n \sqrt{d} (\|\mathbf{u}_{y}\|_2+\|\mathbf{A}_y\|_F \right)\\
&\stackrel{(a)}{=} \left( 1 - \Omega \left( \frac{\eta\Lambda_y}{n\sqrt{m}} |\mathcal{S}_y| \|\mathbf{u}_y\|_2^2 \right) \right) \mathcal{L}(\mathbf{W}^{(t)}, \mathbf{x}, y)+ \mathcal{O} \left( \eta \sigma_n \sqrt{d} (\|\mathbf{u}_{y}\|_2+\|\mathbf{A}_y\|_F \right),
\end{aligned}
\end{equation}
This finishes the training loss analysis.

\subsection{Proof of Theorem~\ref{thm:mem}}

By (\ref{eq:updaterule}) and (\ref{eq:grad}), for any $(\mathbf x,y)\in S$, we have:
\begin{equation*}
\begin{aligned}
    \langle \mathbf w_{y,r}^{(t+1)},\boldsymbol\xi\rangle&-\langle \mathbf w_{y,r}^{(t)},\boldsymbol\xi\rangle=\left\langle- \frac{\eta}{mB} \sum_{(\mathbf{x}_i, y_i) \in S^{(t)} \setminus S_y^{(t)}} \sigma'(\langle \mathbf{w}_{y,r}^{(t)}, \boldsymbol{\xi}_i \rangle) h(C, \mathbf{x}_i, y_i)\text{logit}_y(\mathbf{W}^{(t)}, \mathbf{x}_i) \boldsymbol{\xi}_i \right.   \\
&\quad+  \left. \frac{\eta}{mB} \sum_{(\mathbf{x}_i, y_i) \in S_y^{(t)}} \sigma'(\langle \mathbf{w}_{y,r}^{(t)}, \boldsymbol{\xi}_i \rangle)h(C, \mathbf{x}_i, y_i) (1 - \text{logit}_y(\mathbf{W}^{(t)}, \mathbf{x}_i)) \boldsymbol{\xi}_i+\eta\cdot\mathbf{n}_{y,r}^{(t)}, \boldsymbol{\xi} \right\rangle
\\  &\ge  - \frac{\eta}{mB} \sum_{(\mathbf{x}_i, y_i) \in S^{(t)} \setminus S_y^{(t)}} \sigma'(\langle \mathbf{w}_{y,r}^{(t)}, \boldsymbol{\xi}_i \rangle) h(C, \mathbf{x}_i, y_i)\text{logit}_y(\mathbf{W}^{(t)}, \mathbf{x}_i) |\langle \boldsymbol{\xi}_i, \boldsymbol{\xi} \rangle|  \\
&\quad + \frac{\eta}{mB} \sigma'(\langle \mathbf{w}_{y,r}^{(t)}, \boldsymbol{\xi} \rangle)h(C, \mathbf{x}, y) (1 - \text{logit}_y(\mathbf{W}^{(t)}, \mathbf{x})) \|\boldsymbol{\xi}\|_2^2 \\
&\quad -  \frac{\eta}{mB} \sum_{(\mathbf{x}_i, y_i) \in S_y^{(t)} \setminus (\mathbf{x}, y)} \sigma'(\langle \mathbf{w}_{y,r}^{(t)}, \boldsymbol{\xi}_i \rangle) h(C, \mathbf{x}_i, y_i)(1 - \text{logit}_y(\mathbf{W}^{(t)}, \mathbf{x}_i)) |\langle \boldsymbol{\xi}_i, \boldsymbol{\xi} \rangle|  -\frac{1}{m} \sum_{r=1}^{m} |\langle\eta\cdot\mathbf{n}_{y,r}^{(t)},\boldsymbol{\xi}\rangle| \\
&\ge - \frac{\eta}{mB} \sum_{(\mathbf{x}_i, y_i) \in S^{(t)} \setminus S_y^{(t)}} h(C, \mathbf{x}_i, y_i)\text{logit}_y(\mathbf{W}^{(t)}, \mathbf{x}_i) |\langle \boldsymbol{\xi}_i, \boldsymbol{\xi} \rangle| + \frac{2\eta}{5mB}h(C, \mathbf{x}, y) (1 - \text{logit}_y(\mathbf{W}^{(t)}, \mathbf{x})) \|\boldsymbol{\xi}\|_2^2 \\
&\quad - \frac{\eta}{mB} \sum_{(\mathbf{x}_i, y_i) \in S_y^{(t)} \setminus (\mathbf{x}, y)}h(C, \mathbf{x}_i, y_i) (1 - \text{logit}_y(\mathbf{W}^{(t)}, \mathbf{x}_i)) |\langle \boldsymbol{\xi}_i, \boldsymbol{\xi} \rangle|-\frac{1}{m} \sum_{r=1}^{m} |\langle\eta\cdot\mathbf{n}_{y,r}^{(t)},\boldsymbol{\xi}\rangle| \\
&\ge \Omega\left( \frac{\eta C}{n\sqrt{m} (\|\mathbf{u}_{y}\|_2 + \sqrt{\operatorname{Tr}(\mathbf{A}_y^{\top}\mathbf{A}_y)})} \right)(1 - \text{logit}_y(\mathbf{W}^{(t)}, \mathbf{x})) \|\boldsymbol{\xi}\|_2^2-\mathcal{O} \left( \eta \sigma_n \sqrt{d} \|\mathbf{A}_{y}\|_F \right),
\\ &= \Omega\left(\frac{\eta\Lambda_y}{n\sqrt{m}} \left( 1 - \text{logit}_y(\mathbf{W}^{(t)}, \mathbf{x}) \right) \|\boldsymbol{\xi}\|_2^2\right)-\mathcal{O} \left( \eta \sigma_n \sqrt{d} \|\mathbf{A}_{y}\|_F \right)
\end{aligned}
\end{equation*}

where the last inequality is by Lemma~\ref{lem:aa3} and the bound of the clipping multiplier .

Then, by taking the summation, for any $(\mathbf x,y)\in S$, we have:
\begin{equation*}
\begin{aligned}
    \langle \mathbf w_{y,r}^{(T)},\boldsymbol\xi\rangle&\geq \Omega\left(\sum_{t=0}^{T-1}\frac{\eta\Lambda_y}{n\sqrt{m}} \left( 1 - \text{logit}_y(\mathbf{W}^{(t)}, \mathbf{x}) \right) \|\boldsymbol{\xi}\|_2^2\right)-\mathcal{O}\left(\eta T\sigma_n\|\mathbf{A_y}\|_F\right)+\langle \mathbf w_{y,r}^{(0)},\boldsymbol\xi\rangle\\
    &\geq \Omega\left(\sum_{t=0}^{T-1}\frac{\eta\Lambda_y}{n\sqrt{m}} \left( 1 - \text{logit}_y(\mathbf{W}^{(t)}, \mathbf{x}) \right) \|\boldsymbol{\xi}\|_2^2\right)-\mathcal{O}\left(\eta T\sigma_n\|\mathbf{A_y}\|_F+ \log \left( \frac{Km}{\delta} \right) \|\mathbf{A}_{y}\|_F \sigma_0 \right)  
\end{aligned}
\end{equation*}

where the last inequality is by Lemma~\ref{lem:17}. When  $T \geq \Omega \left( (\eta \Lambda_y \|\mathbf{A}_y\|_F)^{-1} n \sqrt{m} \sigma_0 \right)$, the last term related to initialization can be ignored.
\subsection{Proof of Theorem~\ref{thm:trainloss}}
\begin{proof}
    As (\ref{eq:lossupdate}), we have
\begin{equation}
\begin{aligned}
\mathcal{L}(\mathbf{W}^{(t+1)}, \mathbf{x}, y) &\le  \left( 1 - \Theta \left( \frac{\eta\Lambda_y}{n\sqrt{m}} |\mathcal{S}_y| \|\mathbf{u}_y\|_2^2 \right) \right) \mathcal{L}(\mathbf{W}^{(t)}, \mathbf{x}, y)+ \mathcal{O} \left( \eta \sigma_n \sqrt{d} (\|\mathbf{u}_{y}\|_2+\|\mathbf{A}_y\|_F \right),
\end{aligned}
\end{equation}

Combining all $T$ iterations, we have

\begin{equation}
    \begin{aligned}
      \mathcal{L}(\mathbf{W}^{(T)}, \mathbf{x}, y)  \leq& \left( 1 - \Omega \left( \frac{\eta\Lambda_y}{n\sqrt{m}} |\mathcal{S}_y| \|\mathbf{u}_y\|_2^2 \right) \right)^T \mathcal{L}(\mathbf{W}^{(0)}, \mathbf{x}, y)+\mathcal{O} \left( \eta \sigma_n \sqrt{d} (\|\mathbf{u}_{y}\|_2+\|\mathbf{A}_y\|_F \right)\cdot\mathcal{O} \left( \frac{n\sqrt{m}}{\eta\Lambda_y|\mathcal{S}_y| \|\mathbf{u}_y\|_2^2} \right)\\
      \leq&\underbrace{\exp \left( -\Omega \left( \frac{\eta T\Lambda_{y} |\mathcal{S}_y| }{n\sqrt{m}}\|\mathbf{u}_{y}\|_2^2  \right) \right) \mathcal{L}(\mathbf{W}^{(0)},\mathbf{x},y)}_{\text{Vanishing error}} + \underbrace{\mathcal{O} \left( \frac{n\sqrt{m}\cdot\sigma_n\sqrt{d}( \|\mathbf{u}_{y}\|_2+\|\mathbf{A}_y\|_F)}{\Lambda_{y} |\mathcal{S}_y|\|\mathbf{u}_{y}\|_2^2} \right)}_{\text{Privacy protection error}}
    \end{aligned}
\end{equation}

where the first inequality is obtained from the property of Geometric sequences.
\end{proof}

\section{Test Error Analysis}
In this section, we analyze the test error. Similar to the proof in Appendix~\ref{sec:trainingloss}, the results are based on the high probability conclusions in Appendix~\ref{sec:lem}. In order to characterize the test loss, we first prove the following key lemmas.

\subsection{Key Lemmas}
\begin{lemma}\label{lem:d1}
    
\textit{Define}
\begin{equation}
    \mathcal{S}_i^{(t)} = \{r \in [m] : \langle \mathbf{w}_{y_i, r}^{(t)}, \xi_i \rangle > 0\},
\end{equation}
\textit{for all } $(\mathbf{x}_i, y_i) \in \mathcal{S}$. \textit{For any } $(\mathbf{x}_i, y_i), (\mathbf{x}_j, y_j) \in \mathcal{S}, t \in [T]$, \textit{we have}
\begin{equation}
    \frac{1 - \text{logit}_{y_i}(\mathbf{W}^{(t)}, \mathbf{x}_i)}{1 - \text{logit}_{y_j}(\mathbf{W}^{(t)}, \mathbf{x}_j)} \leq \kappa, 
\end{equation}
\textit{for a constant } $\kappa > 1$ \textit{and}
\begin{equation}
    \mathcal{S}_i^{(t)} \subseteq \mathcal{S}_i^{(t+1)}. 
\end{equation}

\end{lemma}

\begin{proof}
    
 We prove the first statement by induction. First, we show the conclusions hold at iteration 0. At iteration 0, with probability at least $1 - \delta$, for all $(\mathbf{x}, y) \in \mathcal{S}$, by Condition~\ref{condition} and Lemma~\ref{lem:17}, we have
\begin{equation}
    0 \leq F_k(\mathbf{W}^{(0)}, \mathbf{x}) \leq C, 
\end{equation}
where $C > 0$ is a constant. Therefore, there exists a constant $\kappa > 0$ such that
\begin{equation}\label{eq:88}
    \frac{1 - \text{logit}_{y_i}(\mathbf{W}^{(0)}, \mathbf{x}_i)}{1 - \text{logit}_{y_j}(\mathbf{W}^{(0)}, \mathbf{x}_j)} \leq \kappa. 
\end{equation}
Suppose there exists $\bar{t}$ such that the conditions hold for any $0 \leq t \leq \bar{t}$. We aim to prove that the conclusions also hold for $t = \bar{t} + 1$. We consider the following two cases.

\noindent {Case 1:} $\frac{1 - \operatorname{logit}_{y_i}(\mathbf{W}^{(t)}, \mathbf{x}_i)}{1 - \operatorname{logit}_{y_j}(\mathbf{W}^{(t)}, \mathbf{x}_j)} < 0.9\kappa$. First, we have
\[
\frac{1 - \operatorname{logit}_{y_i}\left(\mathbf{W}^{(t+1)}, \mathbf{x}_i\right)}{1 - \operatorname{logit}_{y_i}\left(\mathbf{W}^{(t)}, \mathbf{x}_i\right)} = \frac{\sum_{k \neq y_i} \exp\left(F_k^{(t)}(\mathbf{x}_i)\right) \exp\left(\Delta_k^{(t)}(\mathbf{x}_i)\right)}{\sum_{k \neq y_i} \exp\left(F_k^{(t)}(\mathbf{x}_i)\right)} \frac{\sum_{k \in [K]} \exp\left(F_k^{(t)}(\mathbf{x}_i)\right)}{\sum_{k \in [K]} \exp\left(F_k^{(t)}(\mathbf{x}_i)\right) \exp\left(\Delta_k^{(t)}(\mathbf{x}_i)\right)},
\]
indicating that
\[
\frac{1 - \operatorname{logit}_{y_i}\left(\mathbf{W}^{(t+1)}, \mathbf{x}_i\right)}{1 - \operatorname{logit}_{y_i}\left(\mathbf{W}^{(t)}, \mathbf{x}_i\right)} \leq \frac{\max_{k \in [K]} \exp\left(\Delta_k^{(t)}(\mathbf{x}_i)\right)}{\min_{k \in [K]} \exp\left(\Delta_k^{(t)}(\mathbf{x}_i)\right)}.
\]
Moreover, we have
\[
|\Delta_k^{(t)}(\mathbf{x}_i)| \leq \eta \max_{k \in [K]} \left\{ \|\mathbf{u}_k\|_2^2 + \|\boldsymbol{\xi}\|_2^2 \right\} \leq \eta \max_{k \in [K]} \left\{ \|\mathbf{u}_k\|_2^2 + \frac{3}{2} \operatorname{Tr}\left(\mathbf{A}_k^\top \mathbf{A}_k\right) \right\} \leq 0.02,
\]
where the inequalities are by Lemma~\ref{lem:b13} and Condition~\ref{condition}. Then, we have
\begin{align*}
&\frac{1 - \operatorname{logit}_{y_i}(\mathbf{W}^{(t+1)}, \mathbf{x}_i)}{1 - \operatorname{logit}_{y_j}(\mathbf{W}^{(t+1)}, \mathbf{x}_j)} \\
&= \frac{1 - \operatorname{logit}_{y_i}(\mathbf{W}^{(t+1)}, \mathbf{x}_i)}{1 - \operatorname{logit}_{y_i}(\mathbf{W}^{(t)}, \mathbf{x}_i)} \frac{1 - \operatorname{logit}_{y_j}(\mathbf{W}^{(t)}, \mathbf{x}_j)}{1 - \operatorname{logit}_{y_j}(\mathbf{W}^{(t+1)}, \mathbf{x}_j)} \frac{1 - \operatorname{logit}_{y_i}(\mathbf{W}^{(t)}, \mathbf{x}_i)}{1 - \operatorname{logit}_{y_j}(\mathbf{W}^{(t)}, \mathbf{x}_j)} \\
&\leq \frac{1 - \operatorname{logit}_{y_i}(\mathbf{W}^{(t)}, \mathbf{x}_i)}{1 - \operatorname{logit}_{y_j}(\mathbf{W}^{(t)}, \mathbf{x}_j)} \cdot \exp(0.09) \\
&\leq \kappa.
\end{align*}

\noindent Case 2: $\frac{1 - \operatorname{logit}_{y_i}(\mathbf{W}^{(t)}, \mathbf{x}_i)}{1 - \operatorname{logit}_{y_j}(\mathbf{W}^{(t)}, \mathbf{x}_j)} > 0.9\kappa > 1$. We have
\begin{equation}
\begin{aligned}
& \frac{1 - \operatorname{logit}_{y_i}\left(\mathbf{W}^{(t+1)}, \mathbf{x}_i\right)}{1 - \operatorname{logit}_{y_j}\left(\mathbf{W}^{(t+1)}, \mathbf{x}_j\right)} \\
\leq & \frac{\sum_{k \neq y_j} \exp \left(F_k^{(t)}(\mathbf{x}_j)\right) + \exp \left(F_{y_j}^{(t)}(\mathbf{x}_j)\right) \exp \left(\Delta_{y_j}^{(t)}(\mathbf{x}_j) - \min_{k \neq y_j} \Delta_k^{(t)}(\mathbf{x}_j)\right)}{\sum_{k \neq y_i} \exp \left(F_k^{(t)}(\mathbf{x}_i)\right) + \exp \left(F_{y_i}^{(t)}(\mathbf{x}_i)\right) \exp \left(\Delta_{y_i}^{(t)}(\mathbf{x}_i) - \max_{k \neq y_i} \Delta_k^{(t)}(\mathbf{x}_i)\right)} \frac{\sum_{k \neq y_i} \exp \left(F_k^{(t)}(\mathbf{x}_i)\right)}{\sum_{k \neq y_j} \exp \left(F_k^{(t)}(\mathbf{x}_j)\right)} \\
= & \frac{1 + \left( \frac{1}{1 - \operatorname{logit}_{y_j}(\mathbf{W}^{(t)}, \mathbf{x}_j)} - 1 \right) \exp \left(\Delta_{y_j}^{(t)}(\mathbf{x}_j) - \min_{k \neq y_j} \Delta_k^{(t)}(\mathbf{x}_j)\right)}{1 + \left( \frac{1}{1 - \operatorname{logit}_{y_i}(\mathbf{W}^{(t)}, \mathbf{x}_i)} - 1 \right) \exp \left(\Delta_{y_i}^{(t)}(\mathbf{x}_i) - \max_{k \neq y_i} \Delta_k^{(t)}(\mathbf{x}_i)\right)} \\
\leq & \max \left\{ 1, \kappa \exp \left(\Delta_{y_j}^{(t)}(\mathbf{x}_j) - \min_{k \neq y_j} \Delta_k^{(t)}(\mathbf{x}_j) - \Delta_{y_i}^{(t)}(\mathbf{x}_i) + \max_{k \neq y_i} \Delta_k^{(t)}(\mathbf{x}_i)\right) \right\}
\end{aligned} 
\end{equation}

\noindent Denote $l_1 = \arg\min_{k \neq y_j} \Delta_k^{(t)}(\mathbf{x}_j)$ and $l_2 = \arg\max_{k \neq y_i} \Delta_k^{(t)}(\mathbf{x}_i)$. We have
\begin{equation}
\begin{aligned}
& \Delta_{l_2}^{(t)}(\mathbf{x}_i) - \Delta_{y_i}^{(t)}(\mathbf{x}_i) - \Delta_{l_1}^{(t)}(\mathbf{x}_j) + \Delta_{y_j}^{(t)}(\mathbf{x}_j) \\
\leq & \frac{\eta}{n} \sum_{k \neq i} \left(1 - \operatorname{logit}_{y_k}(\mathbf{W}^{(t)}, \mathbf{x}_k)\right) | \langle \xi_i, \xi_k \rangle | - \frac{2\eta}{5n} \left(1 - \operatorname{logit}_{y_i}(\mathbf{W}^{(t)}, \mathbf{x}_i)\right) \|\xi_i\|_2^2 \\
& + \frac{\eta}{n} \sum_{k \neq j} \left(1 - \operatorname{logit}_{y_k}(\mathbf{W}^{(t)}, \mathbf{x}_k)\right) | \langle \xi_j, \xi_k \rangle | + \frac{\eta}{n} \left(1 - \operatorname{logit}_{y_j}(\mathbf{W}^{(t)}, \mathbf{x}_j)\right) \|\xi_j\|_2^2 \\
& - \frac{2\eta}{5n} \left(1 - \operatorname{logit}_{y_i}(\mathbf{W}^{(t)}, \mathbf{x}_i)\right) \|\mathbf{u}_i\|_2^2 + \frac{\eta}{n} \left(1 - \operatorname{logit}_{y_j}(\mathbf{W}^{(t)}, \mathbf{x}_j)\right) \|\mathbf{u}_j\|_2^2 \\
\leq & 0,
\end{aligned} 
\end{equation}

by letting $\kappa > \Theta(\max\{\operatorname{Tr}(\mathbf{A}_{y_j}^\top \mathbf{A}_{y_j}) / \operatorname{Tr}(\mathbf{A}_{y_i}^\top \mathbf{A}_{y_i}), \|\mathbf{u}_j\|_2^2 / \|\mathbf{u}_i\|_2^2\})$. Then the last inequality is by $\frac{1 - \operatorname{logit}_{y_i}(\mathbf{W}^{(t)}, \mathbf{x}_i)}{1 - \operatorname{logit}_{y_j}(\mathbf{W}^{(t)}, \mathbf{x}_j)} > 0.9\kappa$ and Lemma~\ref{lem:b13}. Therefore, we have
\begin{equation}
    \frac{1 - \operatorname{logit}_{y_i}(\mathbf{W}^{(t+1)}, \mathbf{x}_i)}{1 - \operatorname{logit}_{y_j}(\mathbf{W}^{(t+1)}, \mathbf{x}_j)} \leq \kappa. 
\end{equation}

This completes the proof of the induction.

Next, we prove the second statement. For a data sample $(\mathbf{x}_i, y_i) \in \mathcal{S}$ and a neuron in $\mathcal{S}_i^{(t)}$, we have
\begin{equation}
\begin{aligned}
    \langle \mathbf{w}_{y_i, r}^{(t+1)}, \xi_i \rangle = & \langle \mathbf{w}_{y_i, r}^{(t)}, \xi_i \rangle + \frac{\eta}{mn}h(C,\mathbf{x}_i,y_i) (1 - \operatorname{logit}_{y_i}(\mathbf{x}_i)) \sigma' (\langle \mathbf{w}_{y_i, r}^{(t)}, \xi_i \rangle) \|\xi_i\|_2^2 \\
    & - \frac{\eta}{mn} \sum_{i' \neq i} h(C,\mathbf{x}_i,y_i)\operatorname{logit}_{y_i}(\mathbf{x}_i) \sigma' (\langle \mathbf{w}_{y_i, r}^{(t)}, \xi_{i'} \rangle) \langle \xi_{i'}, \xi_i \rangle+\eta\langle\mathbf{n}_{y_i,r}^{(t)},\xi_{i}\rangle \\
    \geq & \langle \mathbf{w}_{y_i, r}^{(t)}, \xi_i \rangle,
\end{aligned}
\end{equation}
where the inequality is by Lemma~\ref{lem:b13}, Condition~\ref{condition} and (\ref{eq:88}). Thus, we have $\mathcal{S}_i^{(t)} \subseteq \mathcal{S}_i^{(t+1)}$. This completes the proof. 

\end{proof}

\begin{lemma}\label{lem:d2}
    \textit{Under Condition~\ref{condition}, for any $j \in [K], l \in [K] \setminus \{j\}, (\mathbf{x}_q, y_q) \in \mathcal{S}_j, (\mathbf{x}_a, y_a) \in \mathcal{S}_l$, and $r \in \mathcal{S}_q^{(0)}$,}
\begin{equation*}
\begin{aligned}
&\sum_{t'=0}^{t-1} h(C,\mathbf{x}_q,y_q)(1 - \operatorname{logit}_{y_q}(\mathbf{W}^{(t')}, \mathbf{x}_q)) \sigma'(\langle \mathbf{w}_{j,r}^{(t')}, \boldsymbol{\xi}_q \rangle) \\
= & \Omega \left( \frac{1}{n} \sum_{t'=0}^{t-1} (h(C,\mathbf{x}_a,y_a)\operatorname{logit}_j(\mathbf{W}^{(t')}, \mathbf{x}_a)) \sigma'(\langle \mathbf{w}_{j,r}^{(t')}, \boldsymbol{\xi}_a \rangle)  \frac{\operatorname{Tr}(\mathbf{A}_l^\top \mathbf{A}_l)}{\max_{l_1, l_2 \in [K]} \|\mathbf{A}_{l_1}^\top \mathbf{A}_{l_2}\|_F^{1/2} \log(3n^2/\delta)^{1/2}} \right).
\end{aligned}
\end{equation*}

\noindent \textit{Proof.} By the update rule of gradient descent, we have
\begin{equation*}
\begin{aligned}
\langle \mathbf{w}_{j,r}^{(t+1)} - \mathbf{w}_{j,r}^{(t)}, \boldsymbol{\xi}_a \rangle \leq & \frac{\eta}{mn} \sum_{(\mathbf{x}_i, y_i) \in \mathcal{S} \setminus (\mathbf{x}_a, y_a)}h(C,\mathbf{x}_i,y_i) (1 - \operatorname{logit}_j(\mathbf{W}^{(t)}, \boldsymbol{\xi}_i)) \sigma'(\langle \mathbf{w}_{j,r}^{(t)}, \boldsymbol{\xi}_i \rangle) |\langle \boldsymbol{\xi}_i, \boldsymbol{\xi}_a \rangle| \\
& + \frac{\eta}{mn}h(C,\mathbf{x}_a,y_a) (-\operatorname{logit}_y(\mathbf{W}^{(t)}, \mathbf{x}_a)) \sigma'(\langle \mathbf{w}_{j,r}^{(t)}, \boldsymbol{\xi}_a \rangle) \|\boldsymbol{\xi}_a\|_2^2+\eta\langle\mathbf{n}_{j,r}^{(t)},\xi_{a}\rangle.
\end{aligned}
\end{equation*}
By Lemma~\ref{lem:b13}, we have
\begin{equation}\label{eq:99}
\begin{aligned}
&\langle \mathbf{w}_{j,r}^{(t+1)} - \mathbf{w}_{j,r}^{(0)}, \boldsymbol{\xi}_a \rangle + \frac{\eta}{mn} \sum_{t'=0}^{t-1} h(C,\mathbf{x}_a,y_a)(\operatorname{logit}_j(\mathbf{W}^{(t')}, \mathbf{x}_a)) \sigma'(\langle \mathbf{w}_{j,r}^{(t')}, \boldsymbol{\xi}_a \rangle) \|\boldsymbol{\xi}_a\|_2^2 \\
\leq & \frac{\eta}{mn} \sum_{(\mathbf{x}, y) \in \mathcal{S} \setminus (\mathbf{x}_a, y_a)} \sum_{t'=0}^t h(C,\mathbf{x},y)(1 - \operatorname{logit}_y(\mathbf{W}^{(t')}, \mathbf{x})) \underbrace{\max_{l,y \in [K]} \|\mathbf{A}_l^\top \mathbf{A}_y\|_F^{1/2} \log \left( \frac{3n^2}{\delta} \right)^{1/2}}_{E_1}.
\end{aligned}
\end{equation}
Additionally, by the nature of ReLU activation function, the magnitude of $\langle \mathbf{w}_{j,r}^{(t+1)} - \mathbf{w}_{j,r}^{(0)}, \boldsymbol{\xi}_a \rangle$ satisfies
\begin{equation}\label{eq:100}
\langle \mathbf{w}_{j,r}^{(t+1)} - \mathbf{w}_{j,r}^{(0)}, \boldsymbol{\xi}_a \rangle \geq - \frac{\eta}{mn} \sum_{(\mathbf{x}, y) \in \mathcal{S} \setminus (\mathbf{x}_a, y_a)} \sum_{t'=0}^t h(C,\mathbf{x},y)(1 - \operatorname{logit}_y(\mathbf{W}^{(t')}, \mathbf{x})) E_1 - \frac{\eta}{mn} h(C,\mathbf{x}_a,y_a)\|\boldsymbol{\xi}_a\|_2^2.
\end{equation}
As the learning rate $\eta$ is small (by Condition~\ref{condition}), combining (\ref{eq:99}) and (\ref{eq:100}), for any $(\mathbf{x}, y) \sim \mathcal{D}$, we have

\begin{equation*}
\begin{aligned}
& \frac{\eta}{mn} \sum_{t'=0}^{t-1} (h(C,\mathbf{x}_a,y_a)\operatorname{logit}_j(\mathbf{W}^{(t')}, \mathbf{x}_a)) \sigma'(\langle \mathbf{w}_{j,r}^{(t')}, \boldsymbol{\xi}_a \rangle) \|\boldsymbol{\xi}_a\|_2^2 \\
= & \mathcal{O} \left( \frac{\eta}{mn} \sum_{(\mathbf{x},y) \in \mathcal{S} \setminus (\mathbf{x}_a, y_a)} \sum_{t'=0}^{t} h(C,\mathbf{x},y)(1 - \operatorname{logit}_y(\mathbf{W}^{(t')}, \mathbf{x})) E_1 \right),
\end{aligned}
\end{equation*}

By Lemma~\ref{lem:d1}, for $r \in \mathcal{S}_q^{(0)}$, we have
\begin{equation}\label{eq:102}
\begin{aligned}
& \sum_{t'=0}^{t-1}h(C,\mathbf{x}_q,y_q) (1 - \operatorname{logit}_{y_q}(\mathbf{W}^{(t')}, \mathbf{x}_q)) \sigma'(\langle \mathbf{w}_{j,r}^{(t')}, \boldsymbol{\xi}_q \rangle) \\
= & \Omega \left( \frac{1}{n} \sum_{t'=0}^{t-1} (h(C,\mathbf{x}_a,y_a)\operatorname{logit}_j(\mathbf{W}^{(t')}, \mathbf{x}_a)) \sigma'(\langle \mathbf{w}_{j,r}^{(t')}, \boldsymbol{\xi}_a \rangle) \frac{\|\boldsymbol{\xi}_a\|_2^2}{E_1} \right).
\end{aligned}
\end{equation}

By Lemma~\ref{lem:b13}, we have
\begin{equation}\label{eq:103}
\frac{\|\boldsymbol{\xi}_a\|_2^2}{E_1} = \Omega \left( \frac{\operatorname{Tr}(\mathbf{A}_l^\top \mathbf{A}_l)}{\max_{l_1, l_2 \in [K]} \|\mathbf{A}_{l_1}^\top \mathbf{A}_{l_2}\|_F^{1/2} \log(3n^2/\delta)^{1/2}} \right).
\end{equation}

Combining (\ref{eq:102}) and (\ref{eq:103}) yields the conclusion. This completes the proof.
\end{lemma}

\begin{lemma}\label{lem:d3}
     \textit{Under Condition~\ref{condition}, for any $j \in [K]$, we have}
\begin{equation}
    \left\| \left( \mathbf{w}_j^{(T)} \right)^\top \mathbf{A}_j \right\|_2 \leq \mathcal{O}(\sqrt{m} \log(T)). 
\end{equation}

\noindent \textit{Proof.} Due to Assumption~\ref{ass:non-perfect}, for any $t \in [0, {T}]$, the loss of data $(\mathbf{x}, y) \in \mathcal{S}$ satisfies
\begin{equation}\label{eq:105}
    \operatorname{logit}_y(\mathbf{W}^{(t)},\mathbf{x}) \leq c_5, 
\end{equation}
where $c_5$ is a constant. Then, the loss satisfies
\begin{equation}
\begin{aligned}
    & (1 - c_5) \exp \left( \sigma \left( \langle \mathbf{w}_{y,r}^{(T)}, \boldsymbol{\xi} \rangle \right) \right) \\
    \leq & (1 - c_5) \exp \left( \sigma \left( \langle \mathbf{w}_{y,r}^{(T)}, \mathbf{u}_y \rangle \right) + \sigma \left( \langle \mathbf{w}_{y,r}^{(T)}, \boldsymbol{\xi} \rangle \right) \right) \\
    \leq & c_5 \sum_{j \neq y} \exp \left( \sigma \left( \langle \mathbf{w}_{j,r}^{(T)}, \mathbf{u}_j \rangle \right) + \sigma \left( \langle \mathbf{w}_{j,r}^{(T)}, \boldsymbol{\xi} \rangle \right) \right) \\
    \leq & c_5 K \exp \left( \sqrt{2 \log \left( \frac{4Km}{\delta} \right)} \|\mathbf{u}_k\|_2 \sigma_0 + \mathcal{O} \left( \log \left( \frac{Km}{\delta} \right) \|\mathbf{A}_y\|_F \sigma_0 \right) + \frac{\eta}{mn} \|\boldsymbol{\xi}\|_2^2 \right) \\
    \leq & 1.1 c_5 K \exp \left( \frac{3\eta}{2mn} \operatorname{Tr} \left( \mathbf{A}_y^\top \mathbf{A}_y \right) \right),
\end{aligned} 
\end{equation}
where the first inequality is by the fact that $\sigma \left( \langle \mathbf{w}_{y,r}^{(T)}, \mathbf{u}_y \rangle \right) \geq 0$, the second inequality is by the definition of softmax function and (\ref{eq:105}), the third inequality is by Lemma~\ref{lem:17} and the fourth inequality is by Lemma~\ref{lem:b13} and Condition~\ref{condition}. Letting $c_5$ be $T / (1.1K + T)$ yields
\begin{equation}\label{eq:107}
    \langle \mathbf{w}_{y,r}^{(T)}, \boldsymbol{\xi} \rangle \leq \log \left( \frac{c_5}{1 - c_5} K \right) + \frac{3\eta}{2mn} \operatorname{Tr} \left( \mathbf{A}_y^\top \mathbf{A}_y \right) \leq \mathcal{O}(\log(T)). 
\end{equation}
With probability of $1 - \delta$, for $n$ randomly sampled data $(\mathbf{x}', y) \sim \mathcal{D}_y$, we have
\begin{equation}\label{eq:109}
\begin{aligned}
    \langle \mathbf{w}_{y,r}^{(T)}, \boldsymbol{\xi}' \rangle & \leq \mathcal{O} \left( \frac{\eta}{mn} \sum_{i \in \mathcal{S}} \sum_{t=0}^{T-1} (1 - \operatorname{logit}(\mathbf{W}^{(t)}, \mathbf{x})) |\langle \boldsymbol{\xi}_i, \boldsymbol{\xi}' \rangle|+\eta\sum_{t=0}^{T-1}\langle\mathbf{n}_{y,r}^{(t)},\boldsymbol{\xi'}\rangle \right) \\
    & \leq \mathcal{O} \left( \frac{\eta}{mn} \sum_{t=0}^{T-1} (1 - \operatorname{logit}(\mathbf{W}^{(t)}, \mathbf{x})) \|\boldsymbol{\xi}\|_2^2 +\eta\sum_{t=0}^{T-1}\langle\mathbf{n}_{y,r}^{(t)},\boldsymbol{\xi}\rangle\right) = \Theta(\langle \mathbf{w}_{y,r}^{(T)}, \boldsymbol{\xi} \rangle) = \left( \mathbf{w}_{y,r}^{(T)} \right)^\top \mathbf{A}_y \boldsymbol{\zeta},
\end{aligned}
\end{equation}
With probability $1 - \delta$, at least one sample $(\mathbf{x}', y)$ satisfies $\langle \mathbf{w}_{y,r}^{(T)}, \boldsymbol{\xi}' \rangle = \Theta \left( \left\| \left( \mathbf{w}_{y,r}^{(T)} \right)^\top \mathbf{A}_y \right\|_2 \right)$ by the property of $\mathcal{D}_{\zeta}$. Therefore, we have
\begin{equation}\label{eq:110}
    \langle \mathbf{w}_{y,r}^{(T)}, \boldsymbol{\xi} \rangle = \Omega \left( \left\| \left( \mathbf{w}_{y,r}^{(T)} \right)^\top \mathbf{A}_y \right\|_2 \right).
\end{equation}
Combining (\ref{eq:107}) and (\ref{eq:110}) completes the proof.
\end{lemma}

\begin{lemma}\label{lem:d4}
    \textit{Under condition~\ref{condition}, for a random vector $\boldsymbol{\xi}$ generated from $\mathbf{A}_j\boldsymbol{\zeta}$, $\boldsymbol{\zeta} \sim \mathcal{D}_{\boldsymbol{\zeta}}$ for any $j \in [K]$, with probability at least $1 - \delta$, we have}
\begin{equation}
    \sum_{r=1}^m \mathbb{I}\left( \langle \mathbf{w}_{j,r}^{(T)}, \boldsymbol{\xi} \rangle \right) \geq 0.1m. 
\end{equation}

\noindent \textit{Proof.} First, we can concatenate the neuron weights for class $j \in [K]$ as
\begin{equation}
    \mathbf{W}_j^{(T)} \mathbf{A}_j = \begin{bmatrix} (\mathbf{w}_{j,1}^{(T)})^\top \\ \vdots \\ (\mathbf{w}_{j,m}^{(T)})^\top \end{bmatrix} \mathbf{A}_j = \underbrace{\begin{bmatrix} (\mathbf{w}_{j,1}^{(0)})^\top \\ \vdots \\ (\mathbf{w}_{j,m}^{(0)})^\top \end{bmatrix} \mathbf{A}_j}_{\mathbf{D}_1} + \underbrace{\begin{bmatrix} \beta_{j,1,1} & \cdots & \beta_{j,1,n} \\ \vdots & \vdots & \vdots \\ \beta_{j,m,1} & \cdots & \beta_{j,m,n} \end{bmatrix} \begin{bmatrix} \boldsymbol{\xi}_1^\top \\ \vdots \\ \boldsymbol{\xi}_n^\top \end{bmatrix} \mathbf{A}_j}_{\mathbf{D}_2}, 
\end{equation}
where $\beta_{j,r,i} := \sum_{t'=0}^{T-1} \sigma'(\langle \mathbf{w}_{j,r}^{(t')}, \boldsymbol{\xi}_i \rangle) \operatorname{logit}(\mathbf{W}^{(t')}, \mathbf{x}_i)$. The rank of matrix $\mathbf{D}_2 \mathbf{A}_j$ satisfies
\begin{equation}
    \operatorname{rank}(\mathbf{D}_2 \mathbf{A}_j) \leq \min\{m, n, d, \operatorname{rank}(\mathbf{A}_j)\} = n. 
\end{equation}
In addition, as the matrix $\mathbf{D}_1$ is a Gaussian random matrix, it has full rank almost surely. We have
\begin{equation}
    \operatorname{rank}(\mathbf{D}_1 \mathbf{A}_j) = \min\{m, d, \operatorname{rank}(\mathbf{A}_j)\}, 
\end{equation}
almost surely. Then by Condition~\ref{condition}, the rank of $\mathbf{W}_j^{(T)} \mathbf{A}_j$ satisfies
\begin{equation}
    \operatorname{rank}(\mathbf{W}_j^{(T)} \mathbf{A}_j) \geq \min\{m, d, \operatorname{rank}(\mathbf{A}_j)\} - n \geq 0.9m. 
\end{equation}
In addition, by Lemma~\ref{lem:b1} and Condition~\ref{condition}, the singular value of $\mathbf{W}_j^{(T)} \mathbf{A}_j$ satisfies
\begin{equation}
    \lambda_{\min\{m, d, \operatorname{rank}(\mathbf{A}_j)\} - n}(\mathbf{W}_j^{(T)} \mathbf{A}_j) \geq 0.1\sqrt{m}\sigma_0. 
\end{equation}
Moreover, by Lemma~\ref{lem:d3}, we have
\begin{equation}
    \lambda_1(\mathbf{W}_j^{(T)} \mathbf{A}_j) \leq \|\mathbf{W}_j^{(T)} \mathbf{A}_j\|_F \leq \mathcal{O}(\sqrt{m} \log(T)). 
\end{equation}

\noindent Therefore, according to Condition~\ref{condition}, we have
\begin{equation*}
    m \geq \Omega \left( \frac{\log(n/\delta) \log(T)^2}{n\sigma_0^2} \right).
\end{equation*}
By Lemma~\ref{lem:b12} and Condition~\ref{condition}, with probability at least $1 - \delta$, we have
\begin{equation*}
    \sum_{r=1}^m \mathbb{I}\left( \langle \mathbf{w}_{j,r}^{(T)}, \boldsymbol{\xi} \rangle \right) \geq 0.1m.
\end{equation*}
This completes the proof. 
\end{lemma}

\subsection{Proof of Statement 1 in Theorem~\ref{thm:testerror}}
In this part, we prove Statement 1 in Theorem~\ref{thm:testerror}. For data samples following $(\mathbf{x}, y) \sim \mathcal{D}$, the test error satisfies
\begin{equation}
\begin{aligned}
L_{\mathcal{D}}(\mathbf{W}^{(t)}) &= \mathbb{P} \left[ \arg \max_k F_k(\mathbf{W}^{(T)}, \mathbf{x}, y) \neq y \right] \\
&\le \sum_{j \neq y} \mathbb{P} \left[ F_y(\mathbf{W}^{(T)}, \mathbf{x}, y) \le F_j(\mathbf{W}^{(T)}, \mathbf{x}, y) \right] \\
&= \sum_{j \neq y} \mathbb{P} \left[ \sum_{r=1}^m \sigma(\langle \mathbf{w}_{y,r}^{(T)}, \mathbf{u}_y \rangle) + \sigma(\langle \mathbf{w}_{y,r}^{(T)}, \boldsymbol{\xi} \rangle) \le \sum_{r=1}^m \sigma(\langle \mathbf{w}_{j,r}^{(T)}, \mathbf{u}_y \rangle) + \sigma(\langle \mathbf{w}_{j,r}^{(T)}, \boldsymbol{\xi} \rangle) \right].
\end{aligned}
\end{equation}

For the features, we bound the loss for any $(\mathbf{x}, y) \sim \mathcal{D}$ through
\begin{equation}\label{eq:121}
\mathcal{L}(\mathbf{W}^{(T)}, \mathbf{x}, y) \le \sum_{j \neq y} \mathbb{P} \left[ \sum_{r=1}^m \sigma(\langle \mathbf{w}_{y,r}^{(T)}, \mathbf{u}_y \rangle) \le \sum_{r=1}^m \sigma(\langle \mathbf{w}_{j,r}^{(T)}, \mathbf{u}_y \rangle) + \sigma(\langle \mathbf{w}_{j,r}^{(T)}, \boldsymbol{\xi} \rangle) \right].
\end{equation}

Then, we can bound the model outputs as
\begin{equation}\label{eq:122}
\begin{aligned}
&\sum_{r=1}^m \sigma \left( \langle \mathbf{w}_{y,r}^{(T)}, \mathbf{u}_y \rangle \right) \\
&= \sum_{r=1}^m \sigma \left( \left\langle \mathbf{w}_{y,r}^{(0)} + \frac{\eta}{B} \sum_{t'=1}^T \sum_{(\mathbf{x}_l, y_l) \in \mathcal{S}_y^{(t)}} (1 - \text{logit}_y(\mathbf{W}^{(t')}, \mathbf{x}_l)) {h(C,x_l,y_l)}\sigma'(\langle \mathbf{w}_{y,r}^{(t')}, \mathbf{u}_y \rangle) \mathbf{u}_y {+\eta\cdot\sum_{t'=1}^{T} n_{y,r}^{(t')}}, \mathbf{u}_y \right\rangle \right) \\
&\ge \frac{3m}{10} \frac{\eta}{n} \sum_{t'=1}^T \sum_{(\mathbf{x}_l, y_l) \in \mathcal{S}_y} (1 - \text{logit}_y(\mathbf{W}^{(t')}, \mathbf{x}_l)) \|\mathbf{u}_y\|_2^2{\cdot \left( \frac{C \sqrt{m}}{\|\mathbf{u}_{y}\|_2 + \sqrt{Tr(A_y^{T}A_y)}} \right)-\eta T\sigma_n \|\mathbf{u}_y\|_2 \sqrt{2 \log (2/\delta)}},
\end{aligned}
\end{equation}
where the last inequality is by Lemma~\ref{lem:nmu} and the bound of the clipping multiplier.
And
\begin{equation}\label{eq:123}
\sum_{r=1}^m \sigma(\langle \mathbf{w}_{j,r}^{(T)}, \boldsymbol{\xi} \rangle) \le \sum_{r=1}^m |\langle \mathbf{w}_{j,r}^{(T)}, \boldsymbol{\xi} \rangle| = \sum_{r=1}^m \|\mathbf{A}_k^\top \mathbf{w}_{j,r}^{(T)}\|_2 |\zeta'|,
\end{equation}
where $\zeta'$ is a sub-Gaussian variable. Suppose that $\delta > 0$ and $\|\mathbf{A}_y^\top \mathbf{A}_j\|_F / \|\mathbf{A}_y^\top \mathbf{A}_j\|_{\text{op}} \geq \Theta(\sqrt{\log(K^2/\delta)})$. For the term $\|\mathbf{A}_y^\top \mathbf{w}_{j,r}^{(T)}\|_2$, with probability at least $1 - \delta$, we have

\begin{align}\label{eq:124}
& \left\| \mathbf{A}_y^\top \mathbf{w}_{j,r}^{(T)} \right\|_2 \nonumber \\
&= \left\| \mathbf{A}_y^\top \left( \mathbf{w}_{j,r}^{(0)} + \frac{\eta}{n} \sum_{t'=0}^{T-1} \sum_{(\mathbf{x}, y) \in \mathcal{S}_j} (1 - \text{logit}_j(\mathbf{W}^{(t')}, \mathbf{x})) {h(C,x,y)}\sigma'(\langle \mathbf{w}_{j,r}^{(t')}, \boldsymbol{\xi} \rangle) \boldsymbol{\xi} \right. \right. \nonumber \\
& \quad \left. \left. - \frac{\eta}{n} \sum_{(\mathbf{x}, y) \in \mathcal{S} \setminus \mathcal{S}_j} \text{logit}_j(\mathbf{W}^{(t')}, \mathbf{x}) {h(C,x,y)}\sigma'(\langle \mathbf{w}_{j,r}^{(t')}, \boldsymbol{\xi} \rangle) \boldsymbol{\xi}{+\eta\cdot\sum_{t'=0}^{T-1}n_{j,r}^{(t)}} \right) \right\|_2 \nonumber \\
&\le \left\| \mathbf{A}_y^\top \mathbf{A}_j \left( \frac{\eta}{n} \sum_{t'=0}^{T-1} \sum_{(\mathbf{x}, y) \in \mathcal{S}_j} (1 - \text{logit}_j(\mathbf{W}^{(t')}, \mathbf{x})) \sigma'(\langle \mathbf{w}_{j,r}^{(t')}, \boldsymbol{\xi} \rangle) \boldsymbol{\zeta} \right) \right\|_2  \\
& \quad + \left\| \sum_{k \neq j} \mathbf{A}_y^\top \mathbf{A}_k \left( \frac{\eta}{n} \sum_{(\mathbf{x}, y) \in \mathcal{S}_k} \text{logit}_j(\mathbf{W}^{(t')}, \mathbf{x}) \sigma'(\langle \mathbf{w}_{j,r}^{(t')}, \boldsymbol{\xi} \rangle) \boldsymbol{\zeta} \right) \right\|_2 {+\eta\cdot T\sigma_n\sqrt{Tr({A_y^{\top}A_y})}}\nonumber \\
&\stackrel{(a)}{=} \Theta \left( \left\| \mathbf{A}_y^\top \left( \frac{\eta}{n} \sum_{t'=0}^{T-1} \sum_{(\mathbf{x}, y) \in \mathcal{S}_j} (1 - \text{logit}_j(\mathbf{W}^{(t')}, \mathbf{x})) \sigma'(\langle \mathbf{w}_{j,r}^{(t')}, \boldsymbol{\xi} \rangle) \boldsymbol{\xi} \right) \right\|_2 \right){+\eta\cdot T\sigma_n\sqrt{Tr({A_y^{\top}A_y})}} \nonumber \\
&\stackrel{(b)}{=} \Theta \left( \frac{\eta}{n} \|\mathbf{A}_y^\top \mathbf{A}_j\|_F \cdot \sqrt{ \sum_{(\mathbf{x}, y) \in \mathcal{S}_j} \left( \sum_{t'=0}^{T-1} (1 - \text{logit}_j(\mathbf{W}^{(t')}, \mathbf{x})) \right)^2 } \right){+\eta\cdot T\sigma_n\sqrt{\text{Tr}({\mathbf{A}_y^{\top}\mathbf{A}_y})}}, \nonumber
\end{align}

where the inequality is by Lemma~\ref{lem:d2} and Condition~\ref{condition}, $(a)$ is based on Lemma~\ref{lem:d2}, and $(b)$ is based on the concentration of random vectors (Theorem 6.2.6 in \citet{vershynin2018high}). Substituting (\ref{eq:122}) and (\ref{eq:123}) into (\ref{eq:121}), for any $(\mathbf{x}, y) \sim \mathcal{D}$ we have
\begin{equation}
\resizebox{1\linewidth}{!}{
$\begin{aligned}
& \mathcal{L}(\mathbf{W}^{(T)}, \mathbf{x}, y) \nonumber \\
&\le \sum_{j \neq y} \mathbb{P} \left[ \frac{m^{3/2}}{5} \frac{\eta}{n} \sum_{t'=0}^{T-1} \sum_{(\mathbf{x}, y) \in \mathcal{S}_y} (1 - \text{logit}_y(\mathbf{W}^{(t')}, \mathbf{x}))\Lambda_y \|\mathbf{u}_y\|_2^2 -\eta T\sigma_n \|\mathbf{u}_y\|_2 \sqrt{2 \log (2/\delta)}\le \sum_{r=1}^m |\langle \mathbf{w}_{j,r}^{(T)}, \boldsymbol{\xi} \rangle| + |\langle \mathbf{w}_{j,r}^{(T)}, \mathbf{u}_y \rangle| \right] \nonumber \\
&= \sum_{j \neq y} \mathbb{P} \left[ \frac{m^{3/2}}{5} \frac{\eta}{n} \sum_{t'=0}^{T-1} \sum_{(\mathbf{x}, y) \in \mathcal{S}_y} (1 - \text{logit}_y(\mathbf{W}^{(t')}, \mathbf{x})) \Lambda_y\|\mathbf{u}_y\|_2^2-\eta T\sigma_n \|\mathbf{u}_y\|_2 \sqrt{2 \log (2/\delta)} \le \sum_{r=1}^m \left\| \mathbf{A}_y^\top \mathbf{w}_{j,r}^{(T)} \right\|_2 |\zeta'| + |\langle \mathbf{w}_{j,r}^{(0)}, \mathbf{u}_y \rangle| \right]  \\
&\le \sum_{j \neq y} \mathbb{P} \left[ \sum_{t'=0}^{T-1} \sum_{(\mathbf{x}, y) \in \mathcal{S}_y} (1 - \text{logit}_y(\mathbf{W}^{(t')}, \mathbf{x}))\Lambda_y \|\mathbf{u}_y\|_2^2-\eta T\sigma_n \|\mathbf{u}_y\|_2 \sqrt{2 \log (2/\delta)} \right. \nonumber \\
& \quad \le c\left( \|\mathbf{A}_y^\top \mathbf{A}_j\|_F \cdot \sqrt{ \sum_{(\mathbf{x}, y) \in \mathcal{S}_j} \left( \sum_{t'=0}^{T-1} (1 - \text{logit}_j(\mathbf{W}^{(t')}, \mathbf{x})) \right)^2 }+\eta\cdot T\sigma_n\sqrt{\text{Tr}({\mathbf{A}_y^{\top}\mathbf{A}_y})}\right) |\zeta'| + \mathcal{O} \left( \frac{n}{\eta} \sqrt{\log \left( \frac{4Km}{\delta} \right)} \|\mathbf{u}_y\|_2 \sigma_0 \right) \left. \vphantom{\sum_{t'=0}^{T-1}} \right] \nonumber \\
&\le \sum_{j \neq y} \exp \left[ -c_1 \cdot \left(\frac{|\mathcal{S}_y|\Lambda_y \|\mathbf{u}_y\|_2^2-\sigma_n \|\mathbf{u}_y\|_2 \sqrt{2 \log (2/\delta)}}{\sqrt{|\mathcal{S}_j|} \|\mathbf{A}_y^\top \mathbf{A}_j\|_F+\sigma_n\sqrt{\text{Tr}({\mathbf{A}_y^{\top}\mathbf{A}_y})}} \right)^2\right], \nonumber
\end{aligned}
$}
\end{equation}
where $c, c_1 > 0$ are some constants and the last inequality is obtained from Hoeffding's inequality.
\subsection{Proof of Statement 2 in Theorem~\ref{thm:testerror}}

\begin{lemma}
    for $(\mathbf{x},y)$ sampled from L-long-tailed data distribution $T_i$ defined in Definition~\ref{def:longtail}, each $i \in [K]$, the following bound holds with probability at least $1-\delta$: 
    \[\|\boldsymbol{\xi}\|_2\leq O\left(\sqrt{L^2+d}\|\mathbf{A}_y\|_2\right)\]
    
\begin{proof}

For $(\mathbf{x},y)\sim T_i$, $\left\langle \sum_{r \in \mathcal{R}(\boldsymbol{\xi})} \mathbf{w}_{y,r}^{(T)}, \boldsymbol{\xi} \right\rangle \geq L \|\mathbf{A}_y^\top \sum_{r \in \mathcal{R}(\boldsymbol{\xi})} \mathbf{w}_{y,r}^{(T)}\|_2$

Let $u = \frac{\mathbf{A}^\top \mathbf{w}_{y,r}^{(T)}}{\|\mathbf{A}^\top \mathbf{w}_{y,r}^{(T)}\|_2}$ be the unit direction vector associated with the screening condition. By the linearity of the inner product, the condition $(\mathbf{x},y) \sim \mathcal{T}_i$ is equivalent to: $u^\top \boldsymbol{\xi} \geq L$

Due to the rotational invariance of $\mathcal{N}(0, I_d)$, we can decompose $\boldsymbol{\xi}$ as $\boldsymbol{\xi} = z u + \boldsymbol{\xi}_\perp$, where $z = u^\top \boldsymbol{\xi} \in \mathbb{R}$ is a scalar and $\boldsymbol{\xi}_\perp = (I - uu^\top)\boldsymbol{\xi} \in \mathbb{R}^d$ is the component orthogonal to $u$. Here, $z \sim \mathcal{N}(0, 1)$ and $\boldsymbol{\xi}_\perp \sim \mathcal{N}(0, I - uu^\top)$ are independent. The squared norm is $\|\boldsymbol{\xi}\|_2^2 = z^2 + \|\boldsymbol{\xi}_\perp\|_2^2$.

The term $\|\boldsymbol{\xi}_\perp\|_2^2$ follows a Chi-squared distribution with $d-1$ degrees of freedom, i.e., $\|\boldsymbol{\xi}_\perp\|_2^2 \sim \chi^2_{d-1}$. By the Laurent-Massart concentration inequality, with probability at least $1-\delta/2$:
\begin{equation}
    \|\boldsymbol{\xi}_\perp\|_2^2 \le d-1 + 2\sqrt{(d-1)\ln(2/\delta)} + 2\ln(2/\delta).
\end{equation}

Given the condition $z > L$, we examine the tail of the truncated Gaussian. For any $s > 0$:
\begin{equation}
    \mathbb{P}(z > L + s \mid z > L) = \frac{1 - \Phi(L+s)}{1 - \Phi(L)} \le \exp\left( -Ls - \frac{s^2}{2} \right).
\end{equation}
By setting this probability to $\delta/2$, we find that with probability at least $1-\delta/2$, $z \le L + \frac{\ln(2/\delta)}{L}$. Squaring this term yields $z^2 \le L^2 + 2\ln(2/\delta) + o(1)$.

Applying a union bound over the events, we obtain:
\begin{equation}
    \|\boldsymbol{\xi}\|_2^2 \le d + L^2 + 2\sqrt{d\ln(2/\delta)} + 4\ln(2/\delta).
\end{equation}
Taking the square root and using the property $\|\mathbf{A}_y\boldsymbol{\xi}\|_2 \le \|\mathbf{A}_y\|_{\text{op}} \|\boldsymbol{\xi}\|_2$ completes the proof.
    \end{proof}
\end{lemma}

\textit{Proof of Statement 2:} Furthermore, for long-tailed data distribution, by Lemma~\ref{lem:d2} and union bound, we have
\begin{equation} \label{eq:126}
\mathcal{L}(\mathbf{W}^{(t)}, \mathbf{x}, y) \leq \sum_{j \neq y} \mathbb{P} \left[ \sum_{r=1}^m \sigma \left( \left\langle \mathbf{w}_{y,r}^{(t)}, \boldsymbol{\xi} \right\rangle \right) \leq \sum_{r=1}^m \sigma \left( \left\langle \mathbf{w}_{j,r}^{(T)}, \mathbf{u}_y \right\rangle \right) + \sigma \left( \left\langle \mathbf{w}_{j,r}^{(t)}, \boldsymbol{\xi} \right\rangle \right) \right].
\end{equation}

First, we consider the model behavior under non-DP training, and select long-tailed data by Definition~\ref{def:longtail}. Then, we bound the gap between DP and non-DP setting.

Suppose $\delta > 0$, by Condition~\ref{condition}, we have $\|\mathbf{A}_y^\top \mathbf{A}_y\|_F / \|\mathbf{A}_y^\top \mathbf{A}_y\|_{\text{op}} \geq \Theta(\sqrt{\log(K^2/\delta)})$, for any $y \in [K]$. With probability at least $1 - \delta$, for all $y \in [K]$, under \textbf{non-DP} setting, we have
\begin{equation} \label{eq:127}
\begin{aligned}
    & \left\| \mathbf{A}_y^\top \mathbf{w}_{y,r,clean}^{(T)} \right\|_2 \\
    &= \left\| \mathbf{A}_y^\top \left( \mathbf{w}_{y,r,clean}^{(0)} + \frac{\eta}{n} \sum_{t'=0}^{T-1} \sum_{(\mathbf{x},y) \in \mathcal{S}_y}(1 - \text{logit}_y(\mathbf{W}^{(t')}, \mathbf{x})) \sigma'(\langle \mathbf{w}_{y,r,clean}^{(t')}, \boldsymbol{\xi} \rangle) \boldsymbol{\xi} \right. \right. \\
    & \quad \left. \left. - \frac{\eta}{n} \sum_{t'=0}^{T-1} \sum_{(\mathbf{x},y) \in \mathcal{S} \setminus \mathcal{S}_y} \text{logit}_y(\mathbf{W}^{(t')}, \mathbf{x}) \sigma'(\langle \mathbf{w}_{y,r,clean}^{(t')}, \boldsymbol{\xi} \rangle) \boldsymbol{\xi} \right) \right\|_2 \\
    &\geq \left\| \mathbf{A}_y^\top \mathbf{A}_y \left( \frac{\eta}{n} \sum_{t'=0}^{T-1} \sum_{(\mathbf{x},y) \in \mathcal{S}_y} (1 - \text{logit}_y(\mathbf{W}^{(t')}, \mathbf{x})) \sigma'(\langle \mathbf{w}_{y,r,clean}^{(t')}, \boldsymbol{\xi} \rangle) \boldsymbol{\zeta} \right) \right\|_2 \\
    & \quad - \left\| \sum_{j \neq y} \mathbf{A}_y^\top \mathbf{A}_j \left( \frac{\eta}{n} \sum_{t'=0}^{T-1} \sum_{(\mathbf{x},y) \in \mathcal{S}_j} \text{logit}_y(\mathbf{W}^{(t')}, \mathbf{x}) \sigma'(\langle \mathbf{w}_{y,r,clean}^{(t')}, \boldsymbol{\xi} \rangle) \boldsymbol{\zeta} \right) \right\|_2 \\
    &\stackrel{(a)}{=} \Omega \left( \frac{\eta}{n} \| \mathbf{A}_y^\top \mathbf{A}_y \|_F \cdot \sqrt{ \sum_{(\mathbf{x},y) \in \mathcal{S}_y} \left( \sum_{t'=0}^{T-1} (1 - \text{logit}_y(\mathbf{W}^{(t')}, \mathbf{x})) \right)^2 } \right),
\end{aligned}
\end{equation}
where $(a)$ is obtained by the condition $\| \mathbf{A}_y^\top \mathbf{A}_y \|_F = \Omega (\| \mathbf{A}_y^\top \mathbf{A}_j \|_F)$ for all $j, y \in [K]$ and $j \neq y, K = \Theta(1)$, and (\ref{eq:85}).

Then, we consider DP setting. Similar to the proof of (\ref{eq:124}), we also have
\begin{equation} \label{eq:129}
\left\| \mathbf{A}_y^\top \mathbf{w}_{j,r}^{(T)} \right\|_2 \leq \Theta \left( \frac{\eta}{n} \| \mathbf{A}_y^\top \mathbf{A}_j \|_F \cdot \sqrt{ \sum_{(\mathbf{x},y) \in \mathcal{S}_j} \left( \sum_{t'=0}^{T-1} (1 - \text{logit}_j(\mathbf{W}^{(t')}, \mathbf{x})) \right)^2 }+\eta\cdot T\sigma_n\sqrt{Tr({A_y^{\top}A_y})} \right),
\end{equation}

Denote $\bar{\mathcal{S}}_j(\boldsymbol{\xi}) = \{r \in [m] : \langle \mathbf{w}_{j,r}^{(T)}, \boldsymbol{\xi} \rangle > 0\}$. We have
\begin{equation} \label{eq:131}
\left\| \sum_{r \in \bar{\mathcal{S}}_j(\boldsymbol{\xi})} \mathbf{A}_y^\top \mathbf{w}_{y,r}^{(T)} \right\|_2 \geq \Theta \left( \frac{\eta m}{n} \| \mathbf{A}_y^\top \mathbf{A}_y \|_F \cdot \sqrt{ \sum_{(\mathbf{x},y) \in \mathcal{S}_y} \left( \sum_{t'=0}^{T-1} (1 - \text{logit}_y(\mathbf{W}^{(t')}, \mathbf{x})) \right)^2 } \right),
\end{equation}

because Lemma~\ref{lem:d4} holds.

Finally, we have:
\begin{equation}
\resizebox{1\linewidth}{!}{
$\begin{aligned}
    & \mathcal{L}(\mathbf{W}^{(T)}, \mathbf{x}, y) \leq \sum_{j \neq y} \mathbb{P} \left[ \frac{1}{m} \sum_{r=1}^m \sigma(\langle \mathbf{w}_{y,r}^{(T)}, \boldsymbol{\xi} \rangle) \leq \frac{1}{m} \sum_{r=1}^m \sigma(\langle \mathbf{w}_{j,r}^{(T)}, \mathbf{u}_y \rangle) + \sigma(\langle \mathbf{w}_{j,r}^{(T)}, \boldsymbol{\xi} \rangle) \right] \\
    &= \sum_{j \neq y} \mathbb{P} \left[ \left| \sum_{r \in \bar{\mathcal{S}}_y(\boldsymbol{\xi})} \langle \mathbf{w}_{y,r}^{(T)}, \boldsymbol{\xi} \rangle \right| \leq \sum_{r=1}^m \left| \langle \mathbf{w}_{j,r}^{(T)}, \boldsymbol{\xi} \rangle \right| + \left| \langle \mathbf{w}_{j,r}^{(T)}, \mathbf{u}_y \rangle \right| \right] \\
    &\leq \sum_{j \neq y} \mathbb{P} \left[ L \cdot \left\| \sum_{r \in \bar{\mathcal{S}}_y(\boldsymbol{\xi})} \mathbf{A}_y^\top \mathbf{w}_{y,r,clean}^{(T)} \right\|_2h(C,\mathbf{x},y)-\eta T\sigma_n\sqrt{L^2+d}\|\mathbf{A}_y\|_{op} - \sum_{r=1}^m \left\| \mathbf{A}_y^\top \mathbf{w}_{j,r}^{(T)} \right\|_2 \right.\\
    &\left. \qquad\qquad\leq \sum_{r=1}^m | \langle \mathbf{w}_{j,r}^{(T)}, \boldsymbol{\xi} \rangle | - \sum_{r=1}^m \mathbb{E} | \langle \mathbf{w}_{j,r}^{(T)}, \boldsymbol{\xi} \rangle | + | \langle \mathbf{w}_{j,r}^{(0)}, \mathbf{u}_y \rangle | \right] \\
    &\leq \sum_{j \neq y} \exp \left( -c' \cdot \frac{ \left( L \cdot \left\| \sum_{r \in \bar{\mathcal{S}}_y(\boldsymbol{\xi})} \mathbf{A}_y^\top \mathbf{w}_{y,r,clean}^{(T)} \right\|_2h(C,\mathbf{x},y)-\eta T\sigma_n\sqrt{L^2+d}\|\mathbf{A}_y\|_{op} - \sum_{r=1}^m \left\| \mathbf{A}_y^\top \mathbf{w}_{j,r}^{(T)} \right\|_2 - \mathcal{O} \left( \sqrt{\log(4Km/\delta)} \|\mathbf{u}_y\|_2 \sigma_0 \right) \right)^2 }{ \sum_{r=1}^m \left\| \mathbf{A}_y^\top \mathbf{w}_{j,r}^{(T)} \right\|_2^2 } \right) \\
    &\leq \sum_{j \neq y} \exp \left( -c_2  \cdot \frac{ (L\cdot\sqrt{|\mathcal{S}_y|}\Lambda_y \|\mathbf{A}_y^\top \mathbf{A}_y\|_F-n\sigma_n\sqrt{L^2+d}\|\mathbf{A}_y\|_{op})^2 }{ |\mathcal{S}_j| \|\mathbf{A}_y^\top \mathbf{A}_j\|_F^2 } \right),
\end{aligned}
$}
\end{equation}
where $c' > 0$ is a constant, the first inequality is by the condition for long-tailed data that $|\langle \mathbf{w}_{y,r}^{(T)}, \boldsymbol{\xi} \rangle| \geq L \left\| \mathbf{A}_y^\top \mathbf{w}_{y,r}^{(T)} \right\|_2$ and the gap between DP and non-DP setting, the second inequality is by Hoeffding's inequality and the last inequality is by (~\ref{eq:127}) and (~\ref{eq:129}).

\end{document}